\documentclass[twoside]{article}

\usepackage[preprint]{aistats2027}
\usepackage[round]{natbib}

\usepackage{hyperref}       
\usepackage{url}            
\usepackage{booktabs}       
\usepackage{amsfonts}       
\usepackage{nicefrac}       
\usepackage{microtype}      
\usepackage{xcolor}         

\usepackage{amsmath}

\usepackage{graphicx}
\usepackage{subcaption}
\usepackage{tabularx}
\usepackage{multirow}
\usepackage{enumitem}
\usepackage{makecell}

\newcommand{\accCI}[3]{%
  #1\,{\scriptsize[#2,\,#3]}%
}
\usepackage{subfiles}

\usepackage{amsthm}
\theoremstyle{plain}
\newtheorem{theorem}{Theorem}[section]
\newtheorem{proposition}[theorem]{Proposition}
\newtheorem{lemma}[theorem]{Lemma}
\newtheorem{corollary}[theorem]{Corollary}
\theoremstyle{definition}

\newtheorem{remark}[theorem]{Remark}

\begin{document}

%

%

\twocolumn[

\aistatstitle{Uncertainty Quantification in Federated Granger Causality Learning}

\aistatsauthor{ Ayush Mohanty\footnotemark[1] \And Nazal Mohamed\footnotemark[1] \And  Nagi Gebraeel }

\aistatsaddress{ Georgia Institute of Technology \And  Georgia Institute of Technology \And Georgia Institute of Technology } ]

\footnotetext[1]{Equal Contributions}

\begin{abstract}
Granger causality identifies predictive dependencies in multivariate time series. In distributed settings where parties cannot share data, federated causal learning enables joint analysis. Most federated causal methods assume that clients observe the same features and infer causal relationships as point estimates, with little formal uncertainty quantification. These assumptions do not hold in many industrial systems, where clients observe different features, and the objective is to estimate cross-client dependencies (edges). These dependencies must be estimated indirectly through repeated client-server iterations. Uncertainty from client data and model parameters propagates through this process, making point estimates alone insufficient for assessing cross-client edges. This paper characterizes this uncertainty propagation and uses edge-specific variances to distinguish genuine cross-client dependencies from spurious estimated edges. We consider aleatoric uncertainty from client data variability and epistemic uncertainty from model parameters. We derive closed-form variance recursions and steady-state variances for the client-server iterations. We prove that the propagated contribution of the initial model-parameter uncertainty vanishes asymptotically. 
These variances enable statistically principled selection of cross-client edges. Synthetic experiments show that our approach improves cross-client edge recovery over competing baselines. On real-world industrial datasets, it achieves high root-cause identification accuracy while yielding interpretable dependency structures.

\end{abstract}

\addtocontents{toc}{\protect\setcounter{tocdepth}{0}}

\section{Introduction}\label{sec:Intro}

Granger causality provides a framework for identifying directed predictive dependencies among variables in multivariate time series. In many distributed systems, however, these variables are partitioned across clients that cannot share their raw time-series data. Each client may therefore observe different features, making the dependencies of interest cross-client. Estimating these cross-client Granger-causal edges requires repeated client--server iterations. As uncertainty from local data and model parameters propagates through this iterative process, point estimates alone cannot determine whether an individual cross-client edge reflects a genuine predictive dependency or estimation uncertainty.

Most federated causal structure learning methods assume that clients observe the same features, allowing causal relationships to be estimated within a shared feature space \cite{FedDAG, FDBNL, notears-admm, FedPC, FedCSL}. This setting does not capture Granger-causal interactions in which one client's past state predicts another client's future state. \cite{mohanty2025federated} addresses this cross-client setting through an iterative client--server procedure, but provides only point estimates of cross-client interactions. It does not characterize how uncertainty from local data and model parameters propagates through the estimation process or quantify the resulting uncertainty in the estimated edges.

We develop a statistical framework for characterizing how uncertainty propagates through the iterative client--server estimation of cross-client Granger-causal interactions. Our analysis follows the state-space procedure of \cite{mohanty2025federated}, in which clients retain their raw measurements and exchange low-dimensional state estimates with the server to estimate cross-client transition terms. We derive closed-form variance recursions for client parameters and server interaction estimates, accounting for data variability, parameter uncertainty, and dependence across iterations. We further characterize the steady-state variances and show that the explicitly propagated contribution of the initial model-parameter uncertainty (epistemic) vanishes asymptotically.
 These variances yield a post-training hypothesis test for determining whether individual cross-client edges represent genuine Granger-causal dependencies or are indistinguishable from zero given their estimation uncertainty.

In this work, we develop a theoretical framework for quantifying uncertainty in federated Granger-causal estimation. We address four fundamental questions:

\textbullet{} \textbf{Sources:} What are the primary sources of uncertainty in this setting? \\
\textbullet{} \textbf{Propagation:} How do these uncertainties propagate through iterative client--server updates? \\
\textbullet{} \textbf{Impact:} What is their asymptotic effect on the uncertainty of inferred cross-client edges? \\
\textbullet{} \textbf{Utility:} How can the resulting uncertainty be used for statistically principled edge selection?


\textbf{Main Contributions:} The key technical contributions are as follows:
\begin{enumerate}[label=\textbf{(\arabic*)}, nosep, leftmargin=1.5em]

\item We formalize uncertainty quantification for federated Granger-causal estimation with coupled dynamical systems, distinguishing aleatoric uncertainty from data variability and epistemic uncertainty from model parameters.

\item We derive closed-form covariance recursions that track uncertainty propagation through the \textit{client-to-server} and \textit{server-to-client} channels, as well as \textit{within-client} and \textit{within-server} update dynamics. This analysis identifies four \textit{cross-covariance terms} that couple data and parameter uncertainties.


\item We establish convergence conditions and derive closed-form steady-state variances for both client and server estimates. We further prove that the propagated contribution of the initial model-parameter covariance vanishes asymptotically, while the steady-state variances are determined by client data statistics and converged training dynamics. 


\item Using the steady-state variances, we construct a post-training hypothesis test for determining whether individual cross-client edges are distinguishable from zero given their estimation uncertainty.

\item We evaluate the framework on synthetic and real-world datasets. On synthetic data, it improves F1-score, AUROC, and AUPRC over existing baselines for cross-client edge recovery. On real-world data, where ground-truth causal graphs are unavailable, it achieves higher top-(k) root-cause identification accuracy while producing interpretable graphs.



\end{enumerate}

\section{Related Work}\label{sec:related_work}
We provide a literature review of relevant centralized machine learning techniques in Appendix \ref{appendix:related_work}. This work focuses on decentralized learning, specifically the federated learning (FL) paradigm. In FL, uncertainty quantification has primarily been approached through Bayesian methods, including federated Bayesian neural networks~\cite{yurochkin2019bayesian}, personalized inference~\cite{kotelevskii2022fedpop, zhang2022personalized}, ensemble methods~\cite{chen2020fedbe}, and Monte Carlo dropout~\cite{park2022federated}. While effective in horizontally partitioned IID settings~\cite{yang2019federated}, these approaches rely on approximate posteriors (e.g., variational inference, sampling, dropout) and often treat aleatoric and epistemic uncertainties independently. Extensions to vertical settings, such as VertiBayes~\cite{van2024vertibayes}, remain static and do not capture temporal dynamics or client--server uncertainty interactions. Importantly, these methods focus on predictive performance at individual clients and do not infer any causal structure. Beyond Bayesian approaches, recursive estimation methods such as federated Kalman filtering~\cite{xing2016distributed, baucas2023federated} and distributed state estimation~\cite{korres2010distributed, primadianto2016review} model uncertainty propagation but assume known system dynamics/causal structure, and therefore do not require causal modeling.

Recent works on federated causal structure learning are closest to our setting. Methods such as NOTEAR-ADMM \cite{notears-admm}, FedPC\cite{FedPC}, FDBNL \cite{FDBNL}, FedDAG \cite{FedDAG}, and FedCSL~\cite{FedCSL} focus on horizontally partitioned IID data, where each client observes the same set of features. Consequently, they only recover causal graphs without modeling temporal dynamics. In contrast, our setting involves vertically partitioned time-series data, where clients observe different subsets of features with interdependent dynamics. To the best of our knowledge, the approach proposed by \cite{mohanty2025federated} is the only one that operates under this setting. However, it remains limited to point estimates. 

\section{Preliminaries}\label{sec:prelims}
\textbf{State Space Model.}
The underlying system is modeled as a linear time‐invariant (LTI) state–space: 
\begin{equation}\label{eq:statespace}
    h^t \;=\; A\,h^{t-1} + w^t,\hspace{0.1cm} \text{and} \hspace{0.1cm}
y^t \;=\; C\,h^t + v^t
\end{equation}
where \(h^t\in\mathbb R^p\) are the \textbf{latent low-dimensional states} and \(y^t\in\mathbb R^d\) are the \textbf{measured high-dimensional data} at time \(t\) with $d >> p$. \(A\) and \(C\) are the constant state‐transition and observation matrices; \(w^t\) and \(v^t\) are zero‐mean i.i.d.\ Gaussian system and measurement noise with covariances \(Q\) and \(R\), respectively. We make the following assumptions about the LTI state space model:
\begin{enumerate}
    \item There are $M$ subsystems such that $h^t = [h_1^t,\cdots, h_M^t], y^t = [y_1^t,\cdots,y_M^t]$. Each of these subsystems is a client for our problem setting. The states are such that, $h_m^t \in \mathbb R^{p_m}$, and $y_m^t \in \mathbb R^{d_m}$ with $d_m >> p_m$, and $\sum_{m=1}^M p_m = p$, and $\sum_{m=1}^M d_m = d$. 
    \item The observation matrix $C$ is block-diagonal i.e., $C = \mathrm{diag}\bigl(C_{11},\cdots,C_{MM}\bigr)$ where each block $C_{mm} \in \mathbb R^{d_m \times p_m}$. The block $C_{mm}$ is known at client $m$. 
    \item The state-transition matrix $A$ is not block-diagonal i.e., $\exists\,n\neq m\hspace{0.1cm}\text{s.t.,}\hspace{0.1cm} A_{mn}\neq0$ with $A_{mn} \in \mathbb R^{p_m \times p_n}$. Each client $m$ locally estimates its diagonal block $A_{mm}$, while the off-diagonal blocks $A_{mn} \forall n \neq m$ are \textbf{unknown}. 
\end{enumerate}
\textbf{Granger Causality.} A time series \(h_n\) is said to \emph{Granger‑cause} another series \(h_m\) if the inclusion of past values of \(h_n\) improves the prediction of \(h_m\).  In the state‑space setting, this notion is captured by the off‑diagonal entries of the state-transition matrix \(A\).  For instance, in a system given by, 
\begin{equation}
\begin{pmatrix}
h_1^t \\[4pt]
h_2^t
\end{pmatrix}
=
\begin{pmatrix}
A_{11} & A_{12} \\[4pt]
A_{21} & A_{22}
\end{pmatrix}
\begin{pmatrix}
h_1^{t-1} \\[4pt]
h_2^{t-1}
\end{pmatrix}
+
\begin{pmatrix}
w_1^t \\[4pt]
w_2^t
\end{pmatrix}    
\end{equation}
the series \(h_2\) Granger‑causes \(h_1\) precisely when $A_{12} \;\neq\; 0.$ More generally, \(A_{mn} \neq0 \) indicates that past values of \(h_n\) influence the future of \(h_m\), revealing a directed causal edge from client $n$ to $m$. Estimating $A_{mn}$ in a decentralized system is the goal of the paper by \cite{mohanty2025federated}, described next. 

\subsection{Federated Learning of Granger Causality} \label{subsec: FGC}
\cite{mohanty2025federated} provided a federated learning approach to estimate Granger causality. It is a server-client framework with $M$ different clients having \textbf{unknown interdependencies} across them. We discuss the details of the client and server models below. 

\textbf{Client Model.} Each client \(m\) models its subsystem as an LTI state–space using only the local (diagonal) blocks \(A_{mm}\in\mathbb R^{p_m\times p_m}\) and \(C_{mm}\in\mathbb R^{d_m\times p_m}\).  A Kalman filter based ``\textbf{\textit{local client model}}'' is used to compress the high‑dimensional data (measurement) \(y_t^{(m)}\in\mathbb R^{d_m}\) into a low‑dimensional state estimate \(\hat h_t^{(m)}\in\mathbb R^{p_m}\) with $d_m >> p_m \forall m \in \{1,\cdots,M\}$ such that: 
  \begin{equation}
\begin{aligned}
  h^t_{m,c}      &= A_{mm}\,\hat h^{t-1}_{m,c}, \\
  \hat h^t_{m,c} &= h^t_{m,c} + K_m\bigl(y^t_m - C_{mm}(h^t_{m,c})_c\bigr).
\end{aligned}
\end{equation}
  where \(K_m\) is the Kalman gain computed from \((A_{mm},C_{mm},Q_{mm},R_{mm})\). The local client model (Kalman Filter) ignores cross‑client dynamics \(A_{mn}\;(n\neq m)\). In order to compensate for these cross-client dynamics, the client states are \textit{augmented} with machine learning (ML) models to obtain an ``\textbf{\textit{augmented client model}}'' as follows: 
  \begin{equation}\label{augStates}
      \hat h^t_{m, a}
    = \hat h^t_{m, c} + \theta_m\,y^t_{m}, \hspace{0.1cm} \text{and}\hspace{0.1cm} 
    h^t_{m, a} = A_{mm} \hat{h}^t_{m, a}
  \end{equation}
  where \(\theta_m\in\mathbb R^{p_m\times d_m}\) is a learnable matrix that captures missing cross‑client effects. The loss function ${(L_m)}_a$ is optimized at client $m$ with \(\theta_m\) being learned using Gradient-descent based update, 
  \begin{equation}\label{eq:clientupdate}
\begin{aligned}
  (L_m)_a          &= \lVert y^t_m - C_{mm} h^t_{m,a} \rVert, \quad \text{with} \\
  \theta_m^{\,t+1} &= \theta_m^{\,t}
                      - \eta_1 \nabla_{\theta_m^{\,t}} (L_m)_a
                      - \eta_2 \nabla_{\theta_m^{\,t}} L_s .
\end{aligned}
\end{equation}

In Eq \ref{eq:clientupdate}, \(\eta_1,\eta_2\) are the learning rates corresponding to client $m$'s loss function ${(L_m)}_a$ and server model's loss function $L_s$ respectively (the server model is discussed next). 
 \begin{remark}
     In the gradient update equation above, while, $\nabla_{\theta_m^{\,t}}\,{(L_m)}_a$ can be computed locally at the client $m$, the gradient $\nabla_{\theta_m^{\,t}}\,L_s$ needs information from the server. However, the server cannot directly compute $\nabla_{\theta_m^{\,t}}\,L_s$ as the gradient is w.r.t. client model parameter $\theta_m$. Therefore, the framework adopts the chain rule to decompose it such that: $\nabla_{\theta_m^{\,t}}\,L_s = \bigl(\nabla_{\hat{h}_{m, a}^{\,t}}\,L_s\bigr)\bigl( \nabla_{\theta_m^{\,t}}\,\hat{h}_{m, a}^t\bigr)$. A key advantage of this decomposition is that the first factor i.e., $\nabla_{\hat{h}_{m, a}^{\,t}}\,L_s$ can be communicated from the server to client $m$, while $\nabla_{\theta_m^{\,t}}\,\hat{h}_{m, a}^t$ can be computed locally at client $m$. 
 \end{remark} 
\textbf{Server Model.} The server collects from each client \(m\) the pair \(\bigl(\hat h^{t-1}_{m,c},\,\hat h^{t-1}_{m,a}\bigr)\) and stacks them into  
\begin{equation}
    \begin{aligned}
        H^{t-1}_c &= \bigl[\hat h^{t-1}_{1,c},\cdots,\hat h^{t-1}_{M,c}\bigr] \\
 H^t_a &= \bigl[A_{11} \hat h^{t-1}_{1,a},\cdots, A_{MM} \hat h^{t-1}_{M,a}\bigr]
 \end{aligned}
\end{equation}
The diagonal matrix $diag(A_{11},\cdots,A_{MM})$ estimated locally by the clients is assumed to be known at the server. However, the off-diagonal blocks (representing the Granger-causality) $A_{mn} \forall n \neq m$ are \textbf{unknown}. The server model's goal is to predict the next‐step state $H_s^t$ as follows:
\begin{equation}
\begin{aligned}
  H^t_s     &= \bigl[h^t_{1,s}, \ldots, h^t_{M,s}\bigr], \quad \text{s.t} \\
  h^t_{m,s} &= A_{mm}\,\hat h^{t-1}_{m,c} + \sum_{n \neq m} \hat A_{mn}\,\hat h^{t-1}_{n,c}.
\end{aligned}
\end{equation}
where the estimated causality $\hat{A}_{mn} \forall n \neq m$ are learned by minimizing the server loss function $L_s$ with gradient-descent based updating used to learn $\hat{A}_{mn} \forall n \neq m$ with server learning rate $\gamma$ s.t., 
\begin{equation}
    L_s \;=\; \bigl\|\,H^t_a - H^t_s\,\bigr\|_2^2; \quad
    \hat{A}_{mn}^{\,t+1}
    = \hat{A}_{mn}^{\,t}, \hspace{0.1cm} \text{with} \hspace{0.1cm}
    \;-\;
    \gamma\,\nabla_{\hat{A}_{mn}^{\,t}}\,L_s
\end{equation}
The server then sends $\nabla_{\hat{h}_{m, a}^{\,t}}\,L_s$ to client $m$ for subsequent client parameter updating. 


\section{Sources of Uncertainty}\label{sec:uncertainty-sources}
\textbf{Sources.} We identify two sources of uncertainty:

\textbf{(S1) \textit{Aleatoric}.} 
Aleatoric uncertainty captures the irreducible noise (both measurement and process noise) \(\epsilon_m^t\) in client $m$'s data, i.e., $y_m^t$. When a client updates its parameter $\theta_m^t$ via gradient descent, this data noise \(\epsilon_m^t\) propagates directly into the gradient \(\nabla_{\theta_m}{(L_m)}_a\) and hence into the variance of \(\theta_m^t\).  Likewise, it also enters the augmented state \(\hat h_{m,a}^t=\hat h_{m,c}^t+\theta_m^t\,y_m^t\). Since $\hat{h}_{m, a}^t$ is communicated to the server from clients, this data noise also influences the estimation of Granger causality $\hat{A}_{mn}^t \forall n \neq m$ (also called the \textit{server model parameter}). Furthermore, this data noise affects the server loss $L_s$ and shows up in the gradients $\nabla_{\hat{h}_{m, a}^t}L_s \hspace{0.1cm} \forall m$ communicated from the server to the clients.

\textbf{(S2) \textit{Epistemic}.}
Epistemic uncertainty reflects our lack of knowledge about the model parameters —both the client parameters \(\theta_m\) and the server parameters \(A_{mn}\).  We assume both of these parameters to be \textbf{random variables} in our problem setting. Sampling from a prior \(\theta_m^0\sim\ \mathcal{D}_1(\mu_{\theta_m}^0,\Sigma_{\theta_m}^0)\) and \(A_{mn}^0\sim\mathcal{D}_2(\mu_{A_{mn}}^0,\Sigma_{A_{mn}}^0)\) where, $\mathcal{D}_1$ and $\mathcal{D}_2$ are any location-scale distributions, we refine $\theta_m^0$, and $\hat{A}_{mn}^0$ using gradient‐descent updates. By accumulating sufficient gradient‐descent iterations, we reduce this epistemic uncertainty and thereby increase our confidence in the estimated causality. 

\textbf{Assumptions.}\; Using these two sources we assume: 

\textbf{(A1) \textit{Stochasticity}.}
For every client \(m\), the \textit{client model parameter}
\(\theta_m^t\), and \textit{client data} \(y_m^t\) are random variables.
The local client state
\(\hat h_{m,c}^{\,t}\) is assumed to be provided and thus deterministic.
Consequently, the only randomness entering the augmented state
\(\hat h_{m,a}^{\,t}=\hat h_{m,c}^{\,t}+\theta_m^{\,t}y_m^{\,t}\)
comes from \(\theta_m^{\,t}\) and \(y_m^{\,t}\).
In the server model, the causal estimate \(\hat A_{mn}^t\),
\(\forall n\neq m\) (aka \textit{server model parameter}), is random.

\textbf{(A2) \textit{Model Parameters}.}\;
At each iteration \(t\), distinct off-diagonal server model parameters
are uncorrelated, i.e.,
\(
\operatorname{Cov}\!\left(
\operatorname{Vec}(\hat A_{mn}^{\,t}),
\operatorname{Vec}(\hat A_{mr}^{\,t})
\right)=0,\quad n\neq r.
\)

\textbf{(A3) \textit{Independent Prior}.}
The initial server and client model parameters are independent, i.e.,
\(
\hat A_{mn}^{\,0}\;\perp\!\!\!\perp\;\theta_m^{\,0}.
\)
The initial client model parameters are also uncorrelated, i.e.,
\(
\operatorname{Cov}(\theta_m^{\,0},\theta_n^{\,0})=0
\hspace{0.2cm}\forall n\neq m.
\)

\textbf{(A4) \textit{Asymptotic Stationarity}.}
The first and second moments of the client data converge to
time-invariant quantities, i.e.,
\(
\mu_{y_m}^t:=\operatorname{E}[y_m^t]\to\mu_{y_m}
\)
and
\(
\Sigma_{y_m}^t:=\operatorname{Var}(y_m^t)\to\Sigma_{y_m}.
\)
The local client state also converges to a fixed value, i.e.,
\(
\hat h_{m,c}^{\,t}\to\hat h_{m,c},\quad \forall m.
\)

\section{Uncertainty Propagation}\label{sec:propagation}
\textit{Notation.}  Table \ref{table:notation} gives an overview of the symbols used throughout this paper. 
\begin{table*}[t]
  \centering
  \caption{Summary of Notation (Client Index $m$, Time Index $t$)}
  \label{table:notation}
  \begin{tabularx}{\textwidth}{@{}l X l@{}}
    \toprule
    \textbf{Symbol} & \textbf{Meaning} & \textbf{Shape / Statistics} \\
    \midrule
    $y_m^t$ & Raw data for client $m$ at time $t$
      & $\in \mathbb{R}^{d_m}$; $\mu_{y_m}^t = \mathrm{E}[y_m^t]$, $\Sigma_{y_m}^t = \operatorname{Var}(y_m^t)$ \\
    $\theta_m^t$ & Model parameter at client $m$
      & $\in \mathbb{R}^{p_m \times d_m}$ \\
    $v_m^t$ & Vectorised $\theta_m^t$, i.e., $v_m^t = \operatorname{Vec}(\theta_m^t)$
      & $\mu_{\theta_m}^t = \mathrm{E}[v_m^t]$, $\Sigma_{\theta_m}^t = \operatorname{Var}(v_m^t)$ \\
    $\Omega_m^t$ & Parameter-data covariance at client $m$
      & $\operatorname{Cov}(v_m^t, y_m^t)$ \\
    $\Lambda_m^t$ & Client parameter-state covariance
      & $\operatorname{Cov}(v_m^t, \hat h_{m,a}^t)$ \\
    $\Psi_{mn}^t$ & Server-client parameter covariance
      & $\operatorname{Cov}(a_{mn}^t, v_m^t)$ \\
    $\Gamma_{mn}^t$ & Cross-covariance between $a_{mn}^t$ and $\hat h_{m,a}^t$
      & $\operatorname{Cov}(a_{mn}^t, \hat h_{m,a}^t)$ \\
    $\hat h_{m,a}^t$ & Augmented state estimate at client $m$
      & $\Sigma_{h_m}^t = \operatorname{Var}(\hat h_{m,a}^t)$ \\
    $\hat A_{mn}^t$ & Server parameter estimate ($n \to m$)
      & $\in \mathbb{R}^{p_m \times p_n}$, $\forall n \neq m$ \\
    $a_{mn}^t$ & Vectorised $\hat A_{mn}^t$, i.e., $a_{mn}^t = \operatorname{Vec}(\hat A_{mn}^t)$
      & $\Sigma_{A_{mn}}^t = \operatorname{Var}(a_{mn}^t)$ \\
    $L_s$ & Loss function at server
      & $\in \mathbb{R}^1$ (scalar) \\
    $(L_m)_a$ & Loss function at client $m$
      & $\in \mathbb{R}^1$ (scalar) \\
    $g_{m,s}^t$ & Gradient of server loss $L_s$ w.r.t.\ $\hat h_{m,a}^t$
      & $\in \mathbb{R}^{p_m}$; $\operatorname{Var}(g_{m,s}^t)$ \\
    $\gamma$ & Learning rate of the server model
      & $\in \mathbb{R}^1$ (scalar) \\
    $\eta_1, \eta_2$ & Learning rates of the client model
      & $\in \mathbb{R}^1$ (scalars) \\
    \bottomrule
  \end{tabularx}
\end{table*}

\subsection{Cross-Covariances}\label{sec:Client-ServerCovariance}
The client-server framework described in Section \ref{sec:prelims} intertwines client data \(y_m^t\){\color{blue}{,}} client state \(\hat h_{m,a}^t\), client model parameter
\(v_m^t\), and server model parameter
\(a_{mn}^t\), creating four essential cross-covariances:
\(\Omega_m^t\), \(\Lambda_m^t\), \(\Gamma_{mn}^t\) and \(\Psi_{mn}^t\), which propagates together with the individual variances. 
Using the notation in Table~\ref{table:notation}, Proposition \ref{prop:theta-y-dependence} shows that $v_m^t$, and $y_m^t$ are dependent with a non-zero cross-covariance $\Omega_m^t$. 

\begin{proposition}[\textbf{Client Model-Client Data Dependence}]\label{prop:theta-y-dependence}
The client parameter
$v_m^t$ depends on the client data history $y_m^{t-1}$ and therefore cannot, in
general, be assumed independent of $y_m^t$. Consequently, the
cross-covariance
\(
\Omega_m^t
:=
\operatorname{Cov}(v_m^t,y_m^t)
\)
must be retained in the uncertainty propagation.
\end{proposition}
Using Eq~\ref{augStates} we know that client states $\hat{h}_{m, a}^t$ are a function of the client data $y_m^t$.Since \(\Omega_m^t\) cannot in general be assumed to be zero, the client model and augmented client state may also be statistically dependent. Proposition \ref{prop:lambda-closed-form} analyzes the evolution of the cross-covariance between the client parameter $v_m^t$, and the states $\hat{h}_{m, a}^t$. 
\begin{proposition}[\textbf{Client Model--Client State Dependence}]
\label{prop:lambda-closed-form}
Let
\(
\bar{\theta}_m^t
=
\operatorname{unvec}_{p_m\times d_m}
(\mu_{\theta_m}^t),
\)
and define
\(
\mathcal T_m^t
=
\mathbb E\!\left[
(v_m^t-\mu_{\theta_m}^t)
\left(
(\theta_m^t-\bar{\theta}_m^t)
(y_m^t-\mu_{y_m}^t)
\right)^\top
\right].
\)
Then we have,
\(
\Lambda_m^t
=
\Sigma_{\theta_m}^t
(\mu_{y_m}^t\otimes I_{p_m})
+
\Omega_m^t\bar{\theta}_m^{t\top}
+
\mathcal T_m^t .
\)
If $(v_m^t,y_m^t)$ are jointly Gaussian, then
$\mathcal T_m^t=0$, and
\(
\Lambda_m^t
=
\Sigma_{\theta_m}^t
(\mu_{y_m}^t\otimes I_{p_m})
+
\Omega_m^t\bar{\theta}_m^{t\top}.
\)
\end{proposition}

Due to the iterative communication between client and server, the client model dynamics are coupled with that of the server model in a feedback loop. Essentially, the client’s noisy state estimates \(\hat{h}_{m, a}^t\) affect the server estimate \(a_{mn}^t\), and the server’s uncertain \(a_{mn}^t\) in turn influences subsequent client state estimates. This effect is captured as the cross-covariance term $\Gamma_{mn}^t$ given in Lemma~\ref{lem:gamma-recursion}. 
\begin{lemma}[\textbf{Server Model--Client State Cross-Covariance}]
\label{lem:gamma-recursion}
The term
\( \Gamma_{mn}^t:=\operatorname{Cov}(a_{mn}^t,\hat h_{m,a}^t) \)
evolves as
\( \Gamma_{mn}^{t+1}
=
D_n^t\Gamma_{mn}^t
-2\gamma\sum_{r\neq m,n}F_{mnr}^t\Gamma_{mr}^t
+2\gamma B_{mn}^t\Sigma_{h_m}^t
+E_{mn}^t \),
where
\( D_n^t=(I_{p_n}-2\gamma\hat h_{n,c}^t\hat h_{n,c}^{t\top})\otimes I_{p_m} \),
\( B_{mn}^t=\hat h_{n,c}^t\otimes A_{mm} \),
\( F_{mnr}^t=(\hat h_{n,c}^t\hat h_{r,c}^{t\top})\otimes I_{p_m} \),
and
\( E_{mn}^t
:=
\operatorname{Cov}(a_{mn}^{t+1},
\hat h_{m,a}^{t+1}-\hat h_{m,a}^{t}) \).
If
\( \|\hat h_{m,a}^{t+1}-\hat h_{m,a}^{t}\|_2\leq\epsilon \)
almost surely, then
\( \|E_{mn}^t\|_F
\leq
\epsilon\sqrt{\operatorname{tr}(\Sigma_{A_{mn}}^{t+1})} \).
\end{lemma}

Because the client parameter and its augmented state are already linked through the cross-covariance in Proposition \ref{prop:lambda-closed-form}, the client state–to–server model coupling of Lemma \ref{lem:gamma-recursion} propagates that link one step further, yielding a direct client model–to-server model dependence captured in Lemma \ref{lem:psi-recursion}.
\begin{lemma}[\textbf{Client Model--Server Model Dependence}]
\label{lem:psi-recursion}
Under Assumption \textbf{(A2)}, treating the current data-dependent
gain matrices \(H_m^t\) and \(P_{mr}^t\) as fixed during covariance
propagation, the cross-covariance term
\(\Psi_{mn}^t := \operatorname{Cov}(a_{mn}^t,v_m^t)\) evolves as,
\(
\Psi_{mn}^{t+1}
={}D_n^t\Psi_{mn}^tH_m^{t\top}
-2\gamma\sum_{r\neq m,n}F_{mnr}^t\Psi_{mr}^tH_m^{t\top}+D_n^t\Sigma_{A_{mn}}^tP_{mn}^{t\top}
-2\gamma\sum_{r\neq m,n}
F_{mnr}^t\Sigma_{A_{mr}}^tP_{mr}^{t\top}+2\gamma B_{mn}^t\Lambda_m^{t\top}H_m^{t\top}
+2\gamma B_{mn}^t
\sum_{r\neq m}\Gamma_{mr}^{t\top}P_{mr}^{t\top},
\)
with
\(
H_m^t
=
I_{p_md_m}
-2(y_m^ty_m^{t\top})\otimes
\left[
\eta_1(C_{mm}A_{mm})^\top(C_{mm}A_{mm})
+\eta_2A_{mm}^\top A_{mm}
\right],
\)
and
$
P_{mr}^t
=
2\eta_2
\left(
y_m^t\hat h_{r,c}^{t\top}\otimes A_{mm}^\top
\right).
$
\end{lemma}
\subsection{During Communication}\label{sec:DuringComm}
We characterize the communication channel as the conduit through which every existing uncertainty, i.e., client-side data noise, client-parameter variance, and server-parameter variance, is redistributed at each iteration. Specifically we analyze the uncertainty propagation in both \textbf{(I)} \textit{client-to-server}, and \textbf{(II)} \textit{server-to-client} communication. 

\textbf{(I) Client to Server.} At iteration $t$, client $m$ sends its augmented states $\hat{h}_{m,a}^t$ to the server. While $\hat{h}_{m,a}^t$ naturally captures $\Sigma_{\theta_m}^t$ and $\Sigma_{y_m}^t$, it may also include cross-covariance $\Omega_{m}^t := \operatorname{Cov}(v_m^t, y_m^t)$. Lemma~\ref{lem:var-h-m} provides a closed-form for the uncertainty in $\hat{h}_{m, a}^t$ using $\Sigma_{\theta_m}^t$, $\Sigma_{y_m}^t$, and $\Omega_m^t$. 
\begin{lemma}[\textbf{Uncertainty in Client-to-Server}]
\label{lem:var-h-m}
Let \( \bar{\theta}_m^t:=\mathbb{E}[\theta_m^t] \) and
\( J_m^t:=\mu_{y_m}^{t\top}\otimes I_{p_m} \).
The first-order
variance of the augmented state is
\( \Sigma_{h_m}^t =
J_m^t\Sigma_{\theta_m}^tJ_m^{t\top}
+\bar{\theta}_m^t\Sigma_{y_m}^t\bar{\theta}_m^{t\top}
+J_m^t\Omega_m^t\bar{\theta}_m^{t\top}
+\bar{\theta}_m^t\Omega_m^{t\top}J_m^{t\top} \).
\end{lemma}


\textbf{(II) Server to Client.} 
At iteration~$t$, the server's uncertainty is encoded in the random matrix $\hat A_{mn}^t$. Instead of sending $\hat A_{mn}^t$, the server computes and transmits the gradient:
\(
g_{m, s}^{\,t+1} := \nabla_{\hat{h}_{m,a}^t} L_s^t \hspace{0.1cm}\text{where} \hspace{0.1cm} L_s^t = \| A_{mm} [ \hat{h}_{m,a}^t - \hat{h}_{m,c}^t ] - \sum_{n \neq m} \hat A_{mn}^t \hat{h}_{n,c}^t \|^2.
\)
This gradient inherits uncertainty from both $\hat A_{mn}^t$ and $\hat{h}_{m,a}^t$, propagating the server's model uncertainty to client~$m$. Lemma~\ref{lem:var-gs-A} shows that $g_{m, s}^t$ captures the uncertainty in the server parameters, client states, and their cross-covariance.
\begin{lemma}[\textbf{Uncertainty in Server-to-Client}]
\label{lem:var-gs-A}
Let \( L_{mn}^t:=\hat h_{n,c}^{t\top}\otimes I_{p_m} \) and define
\( R_m^t:=A_{mm}\Sigma_{h_m}^tA_{mm}^\top
+\sum_{n\neq m}L_{mn}^t\Sigma_{A_{mn}}^tL_{mn}^{t\top}
-\sum_{n\neq m}\bigl(A_{mm}\Gamma_{mn}^{t\top}L_{mn}^{t\top}
+L_{mn}^t\Gamma_{mn}^tA_{mm}^\top\bigr) \).
Under Assumption~\textbf{(A2)}, the variance of the server-to-client gradient is
\( \Sigma_{g_m}^t:= \operatorname{Var}(g_{m, s}^t) = 4A_{mm}^\top R_m^tA_{mm} \).
\end{lemma}

\subsection{Within Server}
In Sections ~\ref{sec:Client-ServerCovariance} and ~\ref{sec:DuringComm} , we quantified how 
 \textbf{(i)} the client–server cross‐covariance \(\Gamma_{mn}^t\) , and \textbf{(ii)} client \(m\)'s state variance \(\Sigma_{h_m}^t\) propagate during the iterative optimization. We now analyze their contribution to the propagation of the server’s parameter uncertainty \(\Sigma_{A_{mn}}^t\). Theorem~\ref{thm:within-server-full} combines these components into a closed‑form recursion for \(\Sigma_{A_{mn}}^t\) within the server.  
\begin{theorem}[\textbf{Uncertainty Propagation within Server}]
\label{thm:within-server-full}
Under Assumption~\textbf{(A2)}, the covariance of the server parameter
\( a_{mn}^t:=\operatorname{Vec}(\hat{A}_{mn}^t) \) evolves as,\\
\( \Sigma_{A_{mn}}^{t+1}
=
D_n^t\Sigma_{A_{mn}}^tD_n^{t\top}
+4\gamma^2B_{mn}^t\Sigma_{h_m}^tB_{mn}^{t\top}
+4\gamma^2\sum_{r\neq m,n}F_{mnr}^t\Sigma_{A_{mr}}^tF_{mnr}^{t\top}
+2\gamma D_n^t\Gamma_{mn}^tB_{mn}^{t\top}
+2\gamma B_{mn}^t\Gamma_{mn}^{t\top}D_n^{t\top}
-4\gamma^2\sum_{r\neq m,n}
\bigl(
B_{mn}^t\Gamma_{mr}^{t\top}F_{mnr}^{t\top}
+
F_{mnr}^t\Gamma_{mr}^tB_{mn}^{t\top}
\bigr) \),
where
\( D_n^t=(I_{p_n}-2\gamma\hat h_{n,c}^t\hat h_{n,c}^{t\top})\otimes I_{p_m} \),
\( B_{mn}^t=\hat h_{n,c}^t\otimes A_{mm} \), and
\( F_{mnr}^t=(\hat h_{n,c}^t\hat h_{r,c}^{t\top})\otimes I_{p_m} \).
\end{theorem}


\subsection{Within Client} 
We now analyze the propagation of uncertainty of the client model parameter $\theta_m$ (or, $v_m$ in vectorized form). Theorem~\ref{thm:within-client} expresses the evolution of client model's variance \(\Sigma_{\theta_m}^t\) in terms of the uncertainty in its states \(\Sigma_{h_m}^{t-1}\), the server model \(\Sigma_{A_{mn}}^{t-1}\), and those cross‑covariances \(\Omega_m^{\,t-1}\), \(\Gamma_{mn}^{\,t-1}\), \(\Psi_{mn}^{\,t-1}\)and \(\Lambda_m^{\,t-1}\).     
\begin{theorem}[\textbf{Uncertainty Propagation within Client}]
\label{thm:within-client}
Under conditions of Lemma~\ref{lem:psi-recursion}, the covariance of the client parameter
\(v_m^t:=\operatorname{Vec}(\theta_m^t)\) evolves as
\(
\Sigma_{\theta_m}^{t+1}
={}
H_m^t\Sigma_{\theta_m}^tH_m^{t\top}
+\sum_{n\neq m}
P_{mn}^t\Sigma_{A_{mn}}^tP_{mn}^{t\top}+\sum_{n\neq m}
\left(
H_m^t\Psi_{mn}^{t\top}P_{mn}^{t\top}
+
P_{mn}^t\Psi_{mn}^tH_m^{t\top}
\right)\)
\end{theorem}

\section{Steady-State Impact}\label{sec:impact}
Let the mean of the gains $H_m^t$ and $P_m^t$ be represented by $\bar H_m^t$ and $\bar P_{mn}^t$, respectively. Then, Proposition~\ref{prop:gains-limit} proves that the gains
\( (D_n^t,B_{mn}^t,F_{mnr}^t,\bar H_m^t,\bar P_{mn}^t) \) converge in the limit
\(t\!\to\!\infty\). 
\begin{proposition}[\textbf{Gain Matrices Convergence}]
\label{prop:gains-limit}
The deterministic gain matrices defined in Section \ref{sec:propagation} converge as
\( (D_n^t,B_{mn}^t,F_{mnr}^t,\bar H_m^t,\bar P_{mn}^t)
\to(D_n,B_{mn},F_{mnr},H_m,P_{mn}) \),
where the steady-state values are,
\( D_n=(I_{p_n}-2\gamma\hat h_{n,c}\hat h_{n,c}^{\top})\otimes I_{p_m} \),
\( B_{mn}=\hat h_{n,c}\otimes A_{mm} \),
\( F_{mnr}=(\hat h_{n,c}\hat h_{r,c}^{\top})\otimes I_{p_m} \),
\( H_m=I_{p_md_m}-2(\Sigma_{y_m}+\mu_{y_m}\mu_{y_m}^{\top})
\otimes[\eta_1(C_{mm}A_{mm})^\top(C_{mm}A_{mm})
+\eta_2A_{mm}^\top A_{mm}] \), and
\( P_{mn}=2\eta_2(\mu_{y_m}\hat h_{n,c}^{\top}\otimes A_{mm}^\top) \).
\end{proposition}
With stable gains,
Proposition~\ref{prop:cross-conv} shows that the cross-covariance terms
\(\Gamma_{mn}^{t}\) and \(\Psi_{mn}^{t}\) also converge, each given in closed
form. Corollary~6.3 then characterizes the steady-state
client-state variance \(\Sigma_{h_m}^{\infty}\).
\begin{proposition}[\textbf{Cross-Covariance Convergence}]
\label{prop:cross-conv}
Suppose the limiting mean gains in Proposition~\ref{prop:gains-limit}
exist, the corresponding coupled linear operators are stable, and
\(E_{mn}^t\to 0\). Under the mean-gain steady-state approximation,
the cross-covariances converge as
\( \Gamma_{mn}^t\to\Gamma_{mn}^{\infty} \) and
\( \Psi_{mn}^t\to\Psi_{mn}^{\infty} \), where
\( \Gamma_{mn}^{\infty}
=
D_n\Gamma_{mn}^{\infty}
-2\gamma\sum_{r\neq m,n}F_{mnr}\Gamma_{mr}^{\infty}
+2\gamma B_{mn}\Sigma_{h_m}^{\infty} \),
and
\( \Psi_{mn}^{\infty}
=
D_n\Psi_{mn}^{\infty}H_m^\top
-2\gamma\sum_{r\neq m,n}F_{mnr}\Psi_{mr}^{\infty}H_m^\top
+D_n\Sigma_{A_{mn}}^{\infty}P_{mn}^\top
-2\gamma\sum_{r\neq m,n}F_{mnr}\Sigma_{A_{mr}}^{\infty}P_{mr}^\top
+2\gamma B_{mn}\Lambda_m^{\infty\top}H_m^\top
+2\gamma B_{mn}\sum_{r\neq m}
\Gamma_{mr}^{\infty\top}P_{mr}^\top \).
\end{proposition}
\begin{corollary}[\textbf{Convergence of Client-State Uncertainty}]
\label{cor:sigmah}
Suppose
\( \Sigma_{\theta_m}^t\to\Sigma_{\theta_m}^{\infty} \),
\( \bar{\theta}_m^t\to\bar{\theta}_m^{\infty} \), and
\( \Omega_m^t\to\Omega_m^{\infty} \).
Under Assumption \textbf{(A4)} and the first-order covariance
approximation of Lemma~\ref{lem:var-h-m}, the client state
covariance converges to
\( \Sigma_{h_m}^{\infty}
\approx
J_m\Sigma_{\theta_m}^{\infty}J_m^\top
+\bar{\theta}_m^{\infty}\Sigma_{y_m}\bar{\theta}_m^{\infty\top}
+J_m\Omega_m^{\infty}\bar{\theta}_m^{\infty\top}
+\bar{\theta}_m^{\infty}\Omega_m^{\infty\top}J_m^\top \),
where \( J_m=\mu_{y_m}^\top\otimes I_{p_m} \).
\end{corollary}
Define the stacked uncertainty vector as \(
\mathbf{s}_m^t
:=
\operatorname{col}\!\Big(
\operatorname{Vec}(\Sigma_{\theta_m}^t),
\{\operatorname{Vec}(\Sigma_{A_{mn}}^t)\}_{n\neq m},
\{\operatorname{Vec}(\Gamma_{mn}^t)\}_{n\neq m},\)\\
\(
\{\operatorname{Vec}(\Psi_{mn}^t)\}_{n\neq m}
\Big)
\). Then Proposition \ref{prop:closed-uncertainty} describes an affine relationship as follows. 
\begin{proposition}[\textbf{Closed Steady-State Uncertainty System}]
\label{prop:closed-uncertainty}
Under the Gaussian closure of Proposition~\ref{prop:lambda-closed-form},
the first-order covariance approximation of
Lemma~\ref{lem:var-h-m}, and the frozen-gain covariance
propagation, the coupled covariance recursions can be written as, 
\(
\mathbf{s}_m^{t+1}
=
\mathcal{K}_m^t\mathbf{s}_m^t+\mathbf{c}_m^t.
\)
If the mean gains and the quantities
\( \bar{\theta}_m^t \), \( \Omega_m^t \), and the corresponding
second-order coefficients converge, then
\( \mathcal{K}_m^t\to\mathcal{K}_m \) and
\( \mathbf{c}_m^t\to\mathbf{c}_m \), yielding the limiting affine
system
\(
\mathbf{s}_m^{t+1}
=
\mathcal{K}_m\mathbf{s}_m^t+\mathbf{c}_m.
\)
\end{proposition}
The key theoretical result on the impact of uncertainty quantification is given next in Theorems~\ref{thm:sigmaA-closed} and \ref{thm:sigmatheta-closed}, where we show that, under stable coupled uncertainty dynamics, the effects of the initial server and client parameter uncertainties (epistemic) vanish asymptotically, while the steady-state uncertainty is driven by stationary client-data statistics (aleatoric) and the converged training dynamics.
\begin{theorem}[\textbf{Convergence of Server Model Uncertainty}]
\label{thm:sigmaA-closed}
Under Proposition~\ref{prop:closed-uncertainty}, suppose there exist
\(T<\infty\) and \(0<q<1\) such that
\(\|\mathcal K_m(t)\|\le q\) for all \(t\ge T\).
Then
\(
\mathbf{s}_m^\infty
=
(I-\mathcal K_m)^{-1}\mathbf c_m
\)
exists uniquely, with
\(
\lim_{t\to\infty}
\mathcal K_m(t-1)\cdots\mathcal K_m(T)\mathbf{s}_m^T=0
\).
Consequently, the explicitly propagated contribution of
\(\Sigma_{A_{mn}}^0\) vanishes asymptotically.
\end{theorem}
\begin{theorem}[\textbf{Convergence of Client Model Uncertainty}]
\label{thm:sigmatheta-closed}
Under the conditions of
Theorem~\ref{thm:sigmaA-closed},
\(\Sigma_{\theta_m}^\infty\) is the unique client-covariance
component of
\(
\mathbf{s}_m^\infty
=
(I-\mathcal K_m)^{-1}\mathbf c_m.
\)
Moreover, the explicitly propagated contribution of
\(\mathbf{s}_m^0\) vanishes asymptotically; hence, for fixed limiting
coefficients \(\mathcal K_m\) and \(\mathbf c_m\),
\(\Sigma_{\theta_m}^\infty\) does not retain an explicit contribution
from \(\Sigma_{\theta_m}^0\).
\end{theorem}
Theorem~\ref{thm:sigmaA-closed} shows that the steady-state uncertainty of the server parameter, represented by \(\Sigma_{A_{mn}}^{\infty}\), is characterized by the client data distribution \((\mu_{y_m},\Sigma_{y_m})\), the steady-state client model variance \(\Sigma_{\theta_m}^{\infty}\), and the converged training dynamics. Under stable coupled uncertainty dynamics, the explicitly propagated contribution of the initial server variance \(\Sigma_{A_{mn}}^{0}\) vanishes asymptotically. Similarly, Theorem~\ref{thm:sigmatheta-closed} establishes that the steady-state variance of the client model, \(\Sigma_{\theta_m}^{\infty}\), is the unique corresponding component of the steady-state covariance system. In particular, for fixed limiting coefficients, the explicitly propagated contribution of the initial client variance \(\Sigma_{\theta_m}^{0}\) vanishes asymptotically. Thus, \textbf{the steady-state values are characterized by client data statistics (aleatoric), without retaining an explicit contribution from the initial model-parameter uncertainty (epistemic)}.

\section{Downstream Utility}\label{sec:utility}

Building on the asymptotic covariance $\Sigma_{A_{mn}}^\infty$, we devise an approach to select statistically supported cross-client edges. Point estimates can declare an edge $\hat A_{mn}\neq 0$. However, finite-sample effects and heterogeneous noise often produce small, nonzero estimates, making it difficult to distinguish relevant interactions from spurious ones. To address this limitation, we construct a post-training hypothesis test:
$
H_0:|A_{mn, k}| \leq \varepsilon,
\hspace{0.1cm} \text{and} \hspace{0.1cm}
H_1:|A_{mn, k}| > \varepsilon 
$ where $A_{mn, k}$ is the element $k$ of block $A_{mn}$. This test is conducted for individual cross-client edges rather than the matrices $\hat{A}_{mn}$. More details of this test and how $\varepsilon$ is chosen is provided in Appendix~\ref{appendix:hypothesis_testing}.



This yields an uncertainty-aware criterion for identifying significant cross-client interactions.

\section{Experiments}\label{sec:experiments}
\begin{table*}[htbp]
\centering
\caption{Causal discovery performance on the synthetic dataset.
Each cell shows the point estimate followed by its 95\% bootstrap
confidence interval over 20 realizations}
\label{tab:synthetic_cd}
\small
\renewcommand{\arraystretch}{1.15}
\begin{tabular}{@{}lcccccc@{}}
\toprule
Method & F1 $\uparrow$ & AUROC $\uparrow$ & AUPRC $\uparrow$
       & SHD $\downarrow$ & FP $\downarrow$ & FN $\downarrow$ \\
\midrule
FedDAG
& \accCI{0.527}{0.522}{0.531} & \accCI{0.679}{0.677}{0.681}
& \accCI{0.625}{0.623}{0.628} & \accCI{39}{38}{39}
& \accCI{0}{0}{0} & \accCI{39}{38}{39} \\
NOTEARS-ADMM
& \accCI{0.325}{0.320}{0.330} & \accCI{0.597}{0.595}{0.599}
& \accCI{0.530}{0.528}{0.532} & \accCI{48}{48}{49}
& \accCI{0}{0}{0} & \accCI{48}{48}{49} \\
FDBNL
& \accCI{0.554}{0.554}{0.554} & \accCI{0.692}{0.692}{0.692}
& \accCI{0.640}{0.640}{0.640} & \accCI{37}{37}{37}
& \accCI{0}{0}{0} & \accCI{37}{37}{37} \\
\textbf{Ours}
& \accCI{0.710}{0.702}{0.723} & \accCI{0.871}{0.867}{0.874}
& \accCI{0.844}{0.842}{0.846} & \accCI{49}{46}{51}
& \accCI{49}{46}{51} & \accCI{0}{0}{1} \\
\bottomrule
\end{tabular}
\end{table*}

\begin{table*}[htbp]
\centering
\caption{Top-k root cause analysis accuracy (AC@k) on real-world datasets. The AC@k cells show the point estimate followed by its 95\% bootstrap interval.
\# CC-edges is the number of cross-client edges recovered.}
\label{tab:realworld_rca}
\small
\setlength{\tabcolsep}{3pt}
\renewcommand{\arraystretch}{1.15}
\begin{tabular}{@{}lcccccccc@{}}
\toprule
& \multicolumn{4}{c}{HAI} & \multicolumn{4}{c}{TEP} \\
\cmidrule(lr){2-5}\cmidrule(lr){6-9}
Method
& \makecell{AC@1} & AC@2 & \makecell{Edge\\Density} & \makecell{\# CC-\\edges}
& \makecell{AC@1} & AC@2 & \makecell{Edge\\Density} & \makecell{\# CC-\\edges} \\
\midrule
FedDAG
& \accCI{.369}{.313}{.424} & \accCI{.822}{.778}{.844} & .347 & 0
& \accCI{.508}{.506}{.510} & \accCI{.609}{.606}{.609} & .182 & 0 \\
NOTEARS-ADMM
& \accCI{.379}{.318}{.495} & \accCI{.778}{.756}{.889} & .053 & 0
& \accCI{.656}{.652}{.661} & \accCI{.795}{.791}{.797} & .034 & 0 \\
FDBNL
& \accCI{.369}{.359}{.480} & \accCI{.844}{.788}{.867} & .368 & 0
& \accCI{.587}{.582}{.593} & \accCI{.747}{.744}{.751} & .207 & 0 \\
\makecell[l]{\textbf{Ours}}
& \accCI{.626}{.551}{.697} & \accCI{.822}{.822}{.978} & .504 & 144
& \accCI{.623}{.623}{.623} & \accCI{.727}{.725}{.728} & .702 & 1317 \\
\bottomrule
\end{tabular}
\end{table*}

\textbf{Synthetic Dataset.} We simulate a multi-client linear time-invariant (LTI) state-space model as described in Section~\ref{sec:prelims} across three experimental settings. The first set validates our theoretical results using a simple two-client setup with $p_m = 2$ and $d_m = 8$ for $m \in \{1,2\}$, where the state matrix satisfies $A_{12} = 0$ and $A_{21} \neq 0$. For baseline comparisons, we use a six-client setup with $p_m = 2$ and $d_m = 8$ for all $m \in \{1,...,6\}$. We further evaluate scalability by varying \textbf{(i)} the no. of clients $M$ and \textbf{(ii)} the data dim. $d_m$ in Appendix~\ref{sec:syntheticData_ScalabilityStudies}. All the experimental details are provided in Appendix~\ref {sec:syntheticData_ExperimentalDetails}.



\textbf{Real-world Datasets.} We evaluate on two industrial datasets: \textbf{(1)} HIL-based Augmented ICS (HAI)~\cite{hai_dataset}, and \textbf{(2)} Tennessee Eastman Process (TEP)~\cite{TEP}. Using domain-specific system structures, we partition HAI and TEP into 3 and 5 clients, respectively.

\textbf{Baselines.} We compare our method against three federated causal structure learning baselines: \textbf{(1)} FedDAG~\cite{FedDAG}, \textbf{(2)} NOTEARS-ADMM~\cite{notears-admm}, and \textbf{(3)} FDBNL~\cite{FDBNL}. Since these algorithms are designed for IID data, we extend them to the time-series setting by incorporating time-lagged variables similar to ~\cite{dynotears}. 

\textbf{Evaluation Metrics.} Following the same metrics as ~\cite{AERCA}, we report F1-score, AUC-ROC, AUC-PR, Structural Hamming Distance (SHD), False Positives (FP), and False Negatives (FN) for the synthetic dataset with known ground-truth causal structure. For the real-world datasets without a ground-truth structure, we perform root cause analysis of anomalies and report AC@k. Details on the top-$k$ ranking for AC@k are provided in Appendix~\ref{appendix:RCA}.

\subsection{Results} 
\begin{figure}
  \centering
  \begin{subfigure}[b]{0.33\columnwidth}
    \centering
    \includegraphics[width=\textwidth]{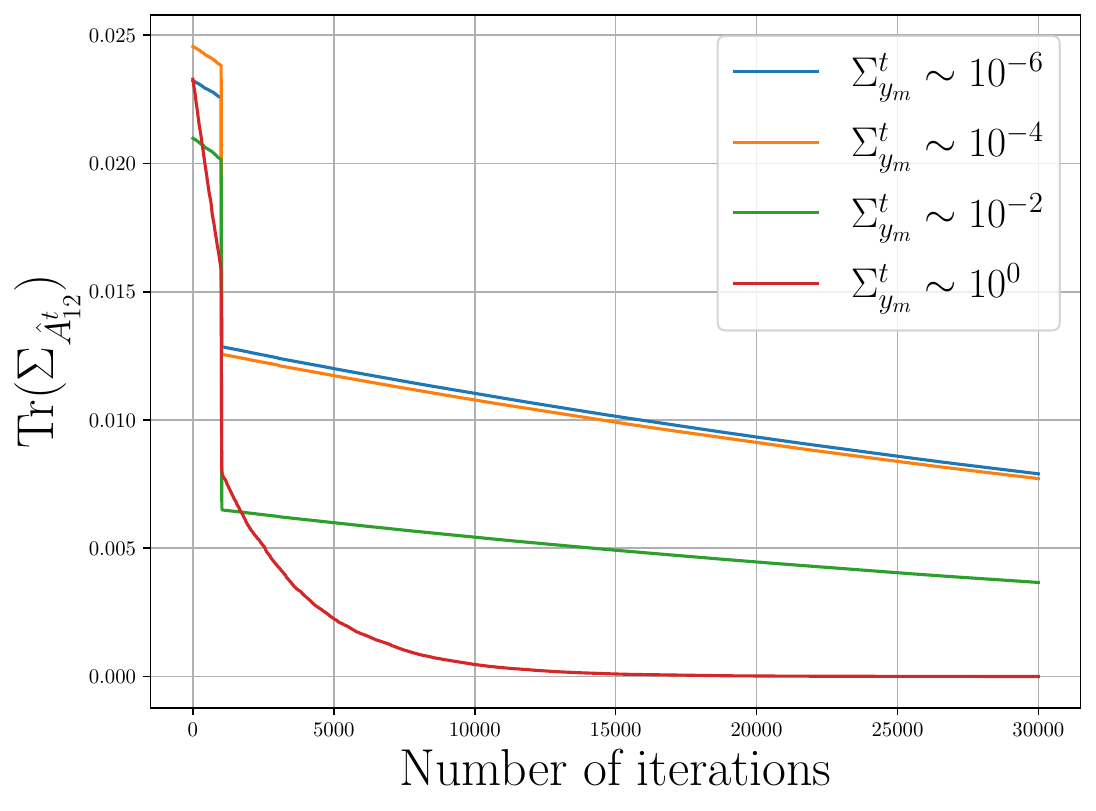}
    \caption{} \label{1a}
  \end{subfigure}%
  \begin{subfigure}[b]{0.33\columnwidth}
    \centering
    \includegraphics[width=\textwidth]{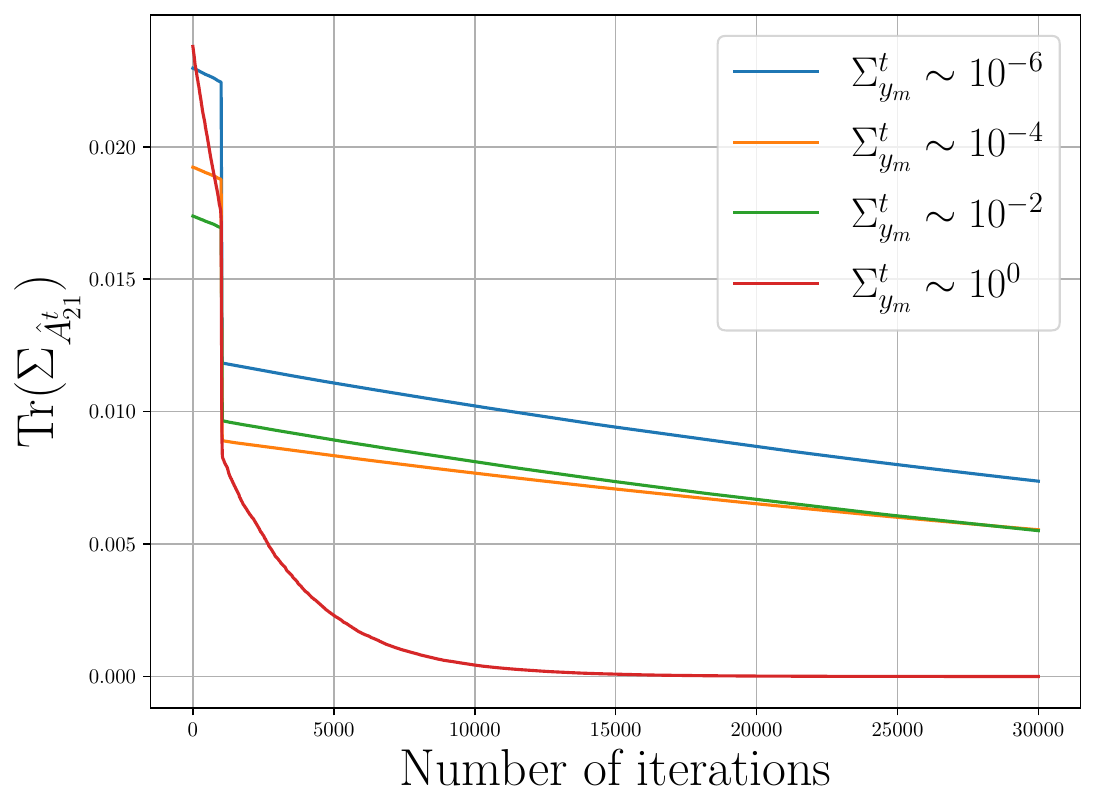}
    \caption{} \label{1b}
  \end{subfigure}%
  \begin{subfigure}[b]{0.33\columnwidth}
    \centering
    \includegraphics[width=\textwidth]{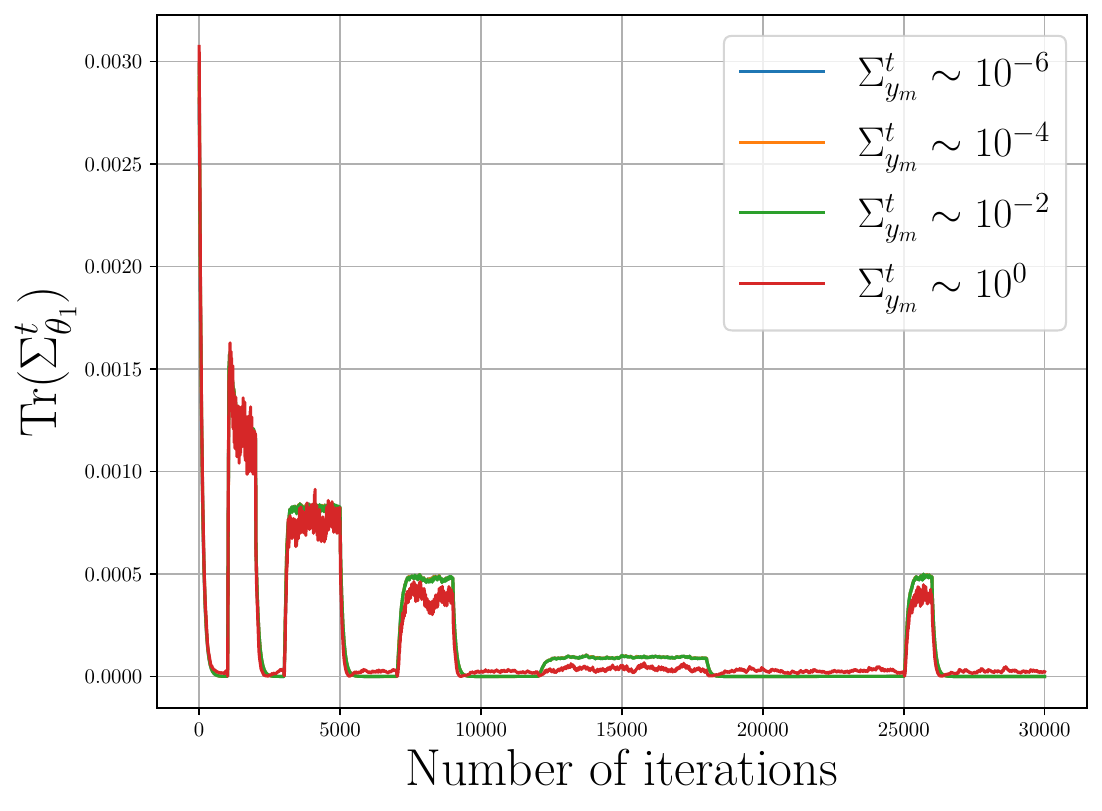}
    \caption{} \label{1c}
  \end{subfigure}\\
  \begin{subfigure}[b]{0.33\columnwidth}
    \centering
    \includegraphics[width=\textwidth]{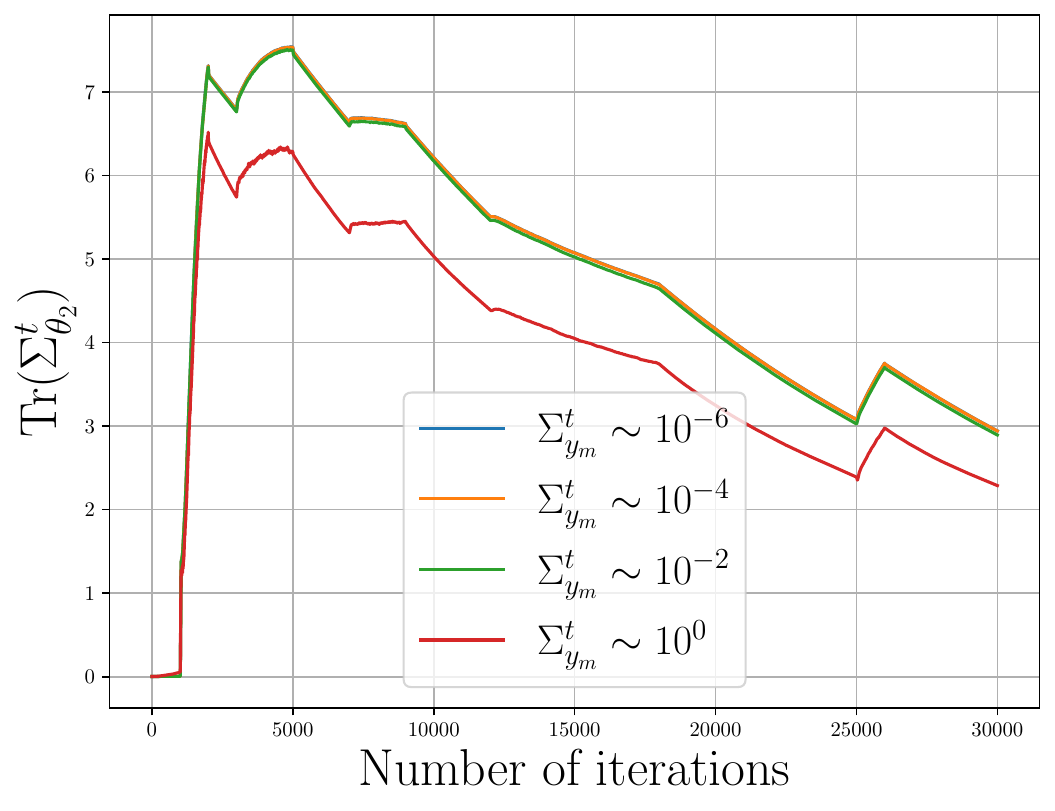}
    \caption{} \label{1d}
  \end{subfigure}%
  \begin{subfigure}[b]{0.33\columnwidth}
    \centering
    \includegraphics[width=\textwidth]{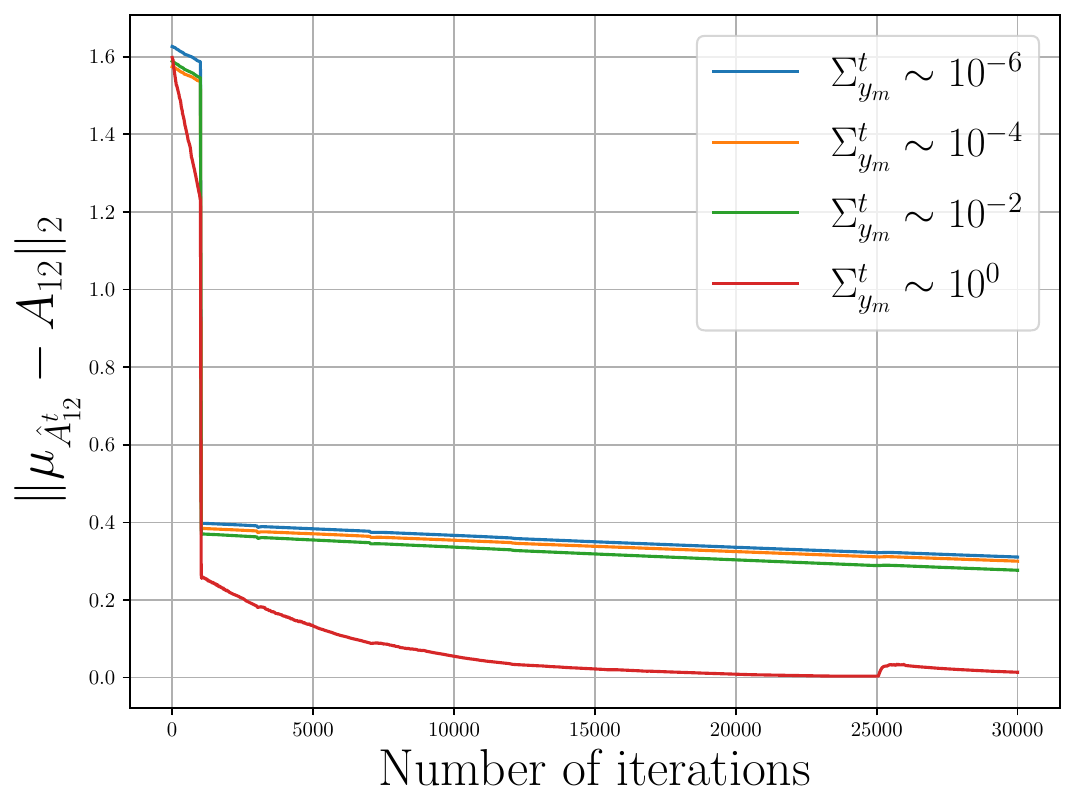}
    \caption{} \label{1e}
  \end{subfigure}%
  \begin{subfigure}[b]{0.33\columnwidth}
    \centering
    \includegraphics[width=\textwidth]{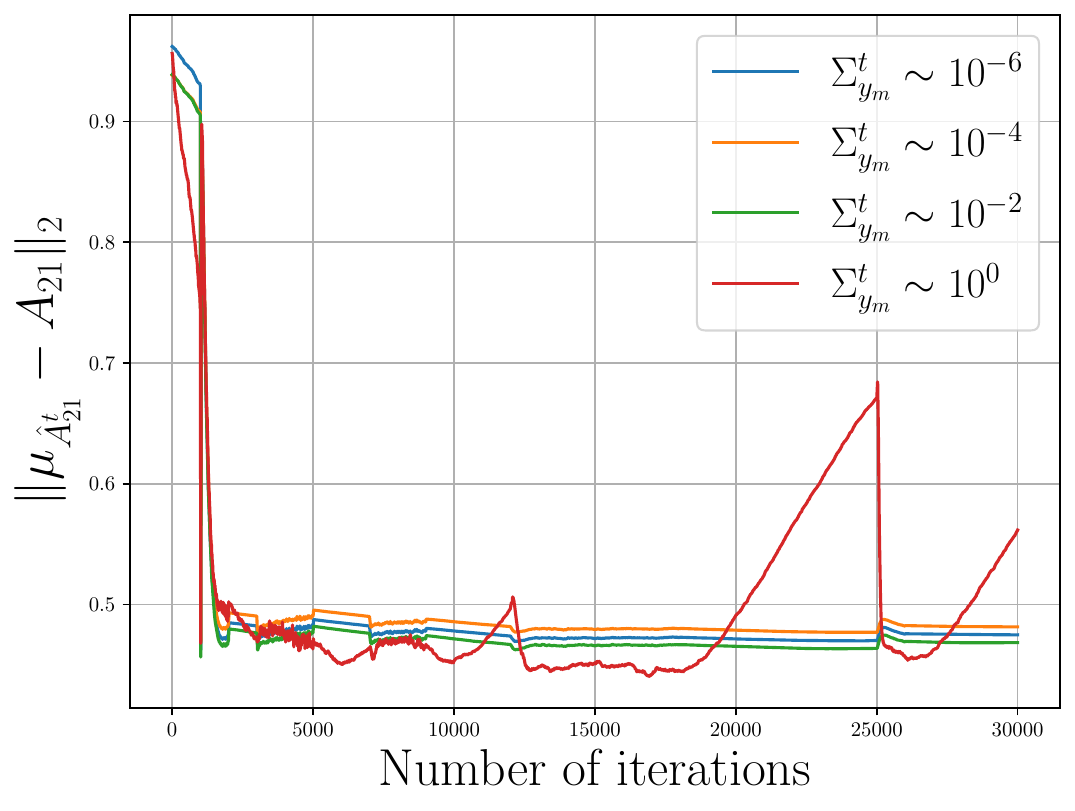}
    \caption{} \label{1f}
  \end{subfigure}
  \caption{Uncertainty prop. during training for different levels of $\Sigma_{y_m}^t$ highlighting
    (a) $\mathrm{Tr}(\Sigma_{\hat{A}_{12}^t})$, 
    (b) $\mathrm{Tr}(\Sigma_{\hat{A}_{21}^t})$, 
    (c) $\mathrm{Tr}(\Sigma_{\theta_1}^t)$, 
    (d) $\mathrm{Tr}(\Sigma_{\theta_2}^t)$, 
    (e) $\|\mu_{\hat{A}_{12}^t}-A_{12}\|_2$, 
    (f) $\|\mu_{\hat{A}_{21}^t}-A_{21}\|_2$ vs iter. $t$
  }
  \label{fig:aleatoric_1}
\end{figure}

\textbf{Validating theory.} We demonstrate theoretical results of the paper on the synthetic dataset and the real-world HAI dataset. To analyze the effect of \textbf{aleatoric uncertainty}, we vary the data variance $\Sigma_{y_m}$ and track the evolution of uncertainty in the server parameters ($\hat{A}_{12}^t$, $\hat{A}_{21}^t$) and client parameters ($\theta_1^t$, $\theta_2^t$). As shown in Fig.~\ref{fig:aleatoric_1}, the covariance traces follow the predicted dynamics and converge, validating the propagation behavior and steady-state results in Section~\ref{sec:propagation} and Theorems~\ref{thm:sigmaA-closed}--\ref{thm:sigmatheta-closed}. Higher $\Sigma_{y_m}$ accelerates variance decay but leads to increased estimation error at large values, especially during mean shifts due to violation of Assumption~\textbf{(A4)}.

The effect of \textbf{epistemic uncertainty} is studied by varying the prior variance of $\hat{A}_{mn}^0$ and $\theta_m^0$ and observing the resulting covariance evolution. As shown in Fig.~\ref{fig:epistemic_sigma_theta}, the dynamics and steady-state covariances remain unaffected by the initial variance in this setting, consistent with Theorems~\ref{thm:sigmaA-closed} and~\ref{thm:sigmatheta-closed}, which show that the propagated contribution of the initial model-parameter uncertainty vanishes asymptotically under stable dynamics.

Appendix~\ref{sec:syntheticData_CrossCovariance}  contains results on the evolution of the \textbf{(a)} cross-covariance, and \textbf{(b)} uncertainty in communicated terms discussed in Section~\ref{sec:propagation}. 

Furthermore, the statistical hypothesis test in Section~\ref{sec:utility} assumes approximate normality of the causal estimates $\hat{A}_{mn}$. Appendix~\ref{appendix:qqplots} provides QQ plots validating this assumption on the synthetic dataset.

\textbf{Causal structure learning.} For the synthetic dataset with a known ground-truth structure, Table~\ref{tab:synthetic_cd} compares our method against all three baselines, all of which exhibit substantially higher FNs. Our approach also significantly outperforms the baselines on F1 score, AUROC, and AUPRC. We observe that the baselines do not learn any cross-client edges. Our method ranks true edges above non-edges well, as reflected by its AUROC and AUPRC. SHD counts every edge in absolute terms and hence remains elevated. Hence, in this case SHD must be analyzed alongside the FNs. 

\textbf{Root cause analysis (RCA).} Since industrial datasets do not provide ground-truth causal structures, Table~\ref{tab:realworld_rca} evaluates the inferred graphs indirectly through downstream RCA performance. On HAI, our method achieves the highest AC@1 among the evaluated methods, while on TEP it achieves comparable AC@1 performance, ranking second to NOTEARS-ADMM. AC@2 is broadly comparable across methods on both datasets; however, AC@2 is less discriminative in this setting because HAI and TEP contain only three and five clients, respectively. We therefore focus primarily on AC@1. 

Good RCA performance without recovering cross-client edges may occur when the strongest deviation from nominal behavior is observed at the root client itself rather than at downstream clients. HAI presents a particularly challenging setting because it contains covert cyberattack scenarios in which measurements are manipulated to remain nominal. This may not hold for TEP, potentially explaining why the methods achieve similar RCA performance despite the baselines' poor recovery of cross-client edges.

\textbf{Ablation Study.} Table~\ref{tab:fedgc-uncertainty-ablation} in Appendix~\ref{appendix: ablation_fedgc} isolates our method from deterministic method proposed in \cite{mohanty2025federated}. When $\varepsilon = 0$, the sample mean of our method would converge to the point estimates in \cite{mohanty2025federated}. On synthetic data, pruning spurious edges improves F1, while keeping AUROC, AUPRC unchanged and reducing SHD. On both HAI and TEP datasets, our method removes a significant portion of cross-client edges, thus producing substantially sparser dependency graphs while preserving AC@1 performance.

\textbf{Privacy.} We provide a differential privacy analysis in Appendix~\ref{appendix:privacy}. Building on that theory, we demonstrate the impact of adding Gaussian noise on asymptotic uncertainty in Appendix~\ref{sec:dp-noise}. 

\textbf{Non-linear Models.} Although the proposed framework is derived under a linearity assumption, Appendix~\ref{appendix:non_linear} discusses extensions to certain classes of nonlinear models. We additionally provide empirical evaluations of these nonlinear extensions in Appendix~\ref{appendix:NonlinearExp}.

\section{Limitations}\label{sec:Limitations}

This work relies on linear dynamics, Gaussian noise, uncorrelated server parameters, and approximations in the covariance propagation. Our steady-state guarantees further require stable and convergent uncertainty dynamics, while finite-horizon inference and real-world root cause analysis introduce additional limitations. We provide a detailed discussion in Appendix~\ref{appendix:limitations}.

\newpage
\bibliography{ref}
\bibliographystyle{plainnat}

\newpage
\appendix
\onecolumn
\aistatstitle{Technical Appendix: Uncertainty Quantification in Federated Granger Causality Learning}
\section{Additional Literature Review}\label{appendix:related_work}
\textbf{Uncertainty Propagation.} In centralized settings, uncertainty propagation under gradient-based optimization has been studied extensively. Prior work analyzes how stochastic gradient noise and perturbations propagate through iterative updates~\cite{wang2023gradient, durasov2024enabling, gawlikowski2023survey}, and formalizes the interaction between aleatoric and epistemic uncertainties~\cite{chan2024estimating, huseljic2021separation, meinert2023unreasonable, hofman2024quantifying, hullermeier2021aleatoric}. However, these analyses assume centralized access to data and do not extend to federated settings with decentralized data and client--server interactions. 
ships between them.

\textbf{Causal Structure Learning.} The field of \textit{causal structure learning} seeks to recover an adjacency graph that governs dependencies among variables in observational data. Classical approaches include constraint-based \cite{pc, spirtes2000causation}, score-based \cite{notears, ges}, and functional model-based methods \cite{anm, lingam}, which infer directed acyclic graphs (DAGs) from fully observed datasets with IID samples. While time-series extensions for some of these methods have been proposed by \cite{dynotears}, \cite{pcmci}, \cite{pcmci_2}, and \cite{var-lingam}, all of these methods assume centralized access to data. 

\textbf{Root Cause Analysis.} Another closely related field is \textit{root cause analysis} (RCA) in time-series data. Representative methods include MicroCause~\cite{MicroCause}, MicroRank~\cite{MicroRank}, CloudRanger~\cite{CloudRanger} CIRCA~\cite{CIRCA}, AERCA~\cite{AERCA}, and RCD~\cite{RCD}, which typically combine anomaly detection with downstream root cause identification. While these methods provide practical metrics for evaluating causal interpretability in nonlinear datasets especially in scenarios where no ground-truth adjacency graph is available, they operate only in fully centralized settings.
 
\section{Root Cause Analysis Metric}\label{appendix:RCA}
For real-world datasets, we evaluate the inferred causal graph using a top-$k$ root cause analysis (RCA) metric. The objective is to determine whether the true root-cause subsystem is ranked among the top anomalous candidates identified using the learned cross-client causal structure.

For each client $m$, let $x_m(t)\in\mathbb{R}^2$ denote the local latent state after preprocessing. We first compute the training mean
\begin{equation}
\mu_m
=
\frac{1}{T_{\mathrm{train}}}
\sum_{t=1}^{T_{\mathrm{train}}}
x_m(t),
\end{equation}
and the corresponding deviation norm
\begin{equation}
d_m(t)
=
\|x_m(t)-\mu_m\|_2.
\end{equation}

The anomaly threshold for each client is then defined as the empirical $99$th percentile of the training deviation norms:
\begin{equation}
\tau_m
=
Q_{0.99}
\left(
\{d_m(t)\}_{t=1}^{T_{\mathrm{train}}}
\right).
\end{equation}

For each anomaly instance $q$, we compute a local anomaly indicator:
\begin{equation}
z_m^{(q)}
=
\mathbf{1}
\left(
d_m^{(q)} > \tau_m
\right).
\end{equation}

Using the pruned causal graph $\hat A$, we then compute a root-cause score for each client:
\begin{equation}
s_j^{(q)}
=
\sum_{i=1}^{N}
z_i^{(q)}
z_j^{(q)}
\,
\|A_{mn}\|_F,
\end{equation}
where $\|A_{mn}\|_F$ denotes the Frobenius norm of the inferred cross-client interaction. Clients are subsequently ranked according to $s_j^{(q)}$.

A prediction is considered correct if the true root-cause client appears within the top-$k$ ranked candidates. The resulting top-$k$ accuracy metric is computed as
\begin{equation}
\mathrm{AC@k}
=
\frac{1}{Q}
\sum_{q=1}^{Q}
\mathbf{1}
\left(
\text{true root cause}
\in
\text{Top-}k
\right),
\end{equation}
where $Q$ denotes the total number of anomaly instances.

\section{Hypothesis Testing} \label{appendix:hypothesis_testing}
Let $\hat{A}_{mn, k}$ is the mean converged estimate of element $k$ of the estimated $\hat{A}_{mn}$ block. The standard error $s_{mn, k}$ of $\hat{A}_{mn, k}$ can be estimated with the asymptotic covariance $\Sigma_{A_{mn}}^\infty$. The hypothesis tested are: $H_0:|A_{mn, k}| \leq \varepsilon, \hspace{0.1cm} \text{and} \hspace{0.1cm} H_1:|A_{mn, k}| > \varepsilon$. Assuming approximate normality of the converged estimator, the statistics are $z_{mn,k}^- = (|\hat{A}_{mn, k}| - \varepsilon)/s_{mn, k}$ and $z_{mn,k}^+ = (|\hat{A}_{mn, k}| + \varepsilon)/s_{mn, k}$. The $p$-value is given by $p = \Bar{\Phi}(z_{mn,k}^-) + \Bar{\Phi}(z_{mn,k}^+)$ where $\Bar{\Phi}(z)$ is the complementary cdf of standard normal distribution. An element is retained if the $p$-value is below the significance level $\alpha = 0.05$

The value of $\varepsilon$ was estimated by independently time-shifting each clients measurements, thereby destroying the cross-client temporal alignment. We train the model to estimate the cross-client edges for this time-shifted data, then choose $\varepsilon$ to be the 95th percentile of the absolute mean of the estimated cross-client edges. 
\section{Additional Results on Synthetic Dataset}\label{sec:synthetic_appendix}
\subsection{Experimental Details}\label{sec:syntheticData_ExperimentalDetails}
Experiments for validating theory on synthetic dataset were conducted by simulating a two client LTI system. Each client has a hidden latent state dimension $p_m = 2$ and measurement dimension $d_m = 8$. A one-way dependency in which client 1 Granger causes client 2 (and not vice-versa) is considered. Therefore $A_{12} = 0$ and $A_{21} \neq 0$. 

All the results reported are in terms of the trace of the covariance matrices. Trace of the covariance matrix was chosen as a scalar quantity to quantify the uncertainty from the covariance matrices for a multitude of reasons such as: 1) it is the sum of variances across all directions, 2) it is rotation invariant, 3) computationally cheap etc. Unless otherwise specified, the quantities plotted are relevant to explaining the uncertainty propagation in learning $\hat{A}_{21}$.

The effect of aleatoric noise in all the experiments is studied by changing the data variance $\Sigma_{y_m}^t$, and the effect of epistemic noise is studied by changing the variance of the initial values $\Sigma_{\hat{A}_{mn}}^0$ and $\Sigma_{\theta_m}^t$. 

A six client system LTI system was simulated to compare the causal discovery against baselines in Table~\ref{tab:synthetic_cd}. Each client has a hidden latent state dimension $p_m = 2$ and measurement dimension $d_m = 8$. There were 9 non-zero cross-client $2\times2$ blocks in the $A$ matrix of this system. The loss functions are regularized and learning rates $\gamma$, $\eta_1$ and $\eta_2$ are adjusted to ensure convergence. 

\subsection{Cross-Covariance and Communicated Terms}\label{sec:syntheticData_CrossCovariance}
\textbf{Cross-Covariances.} The evolution of cross-covariance terms $\Lambda_m^t$, $\Gamma_{mn}^t$ and $\Psi_{mn}^t$ discussed in Section 6 for different regimes of aleatoric and epistemic noise are given in Figures \ref{fig:aleatoric_cross_cov_sigma_y}, \ref{fig:epistemic_cross_cov_Sigma_Amn} and \ref{fig:epistemic_cross_cov_Sigma_theta}. We can observe that the cross-covariance terms converge even for very high noise regimes. 

\textbf{During Communication.} The augmented client state $\hat{h}_{m,a}^t$ sent from the client carries the uncertainty from the client to the server and the gradient $\nabla_{\hat{h}_{m,a}^t} L_s^t$ sent back to the client carries the uncertainty from the server to the client. The evolution of covariances of these communicated terms for different aleatoric and epistemic noises is plotted in Figures \ref{fig:aleatoric_comm_sigma_y}, \ref{fig:aleatoric_comm_sigma_Amn} and \ref{fig:aleatoric_comm_sigma_theta}.

\subsection{Scalability Studies}\label{sec:syntheticData_ScalabilityStudies}
In the scalability experiment, we consider the two-client system with similar causal relationship as in previous experiments. The hidden state dimension $p_m = 2$ is kept constant while the measurement dimension $d_m$ is increased for both clients. Trace of covariance of $\hat{A}_{21}$ at convergence for regimes of $\Sigma_{\hat{A}_{mn}}^0 \sim \{10^{-6}, 10^{-4}, 10^{-2}, 10^{0}\}$ is summarized in Table \ref{tab:scalability_table1}. We could observe that the performance of the framework remains more or less similar with increased dimension of the measurements $d_m$. In the second experiment, we keep $p_m = 2$ and $d_m = 8$ and vary the number of clients $M$. We report the trace of covariance of $\hat{A}_{21}$ at convergence for different regimes of $\Sigma_{\hat{A}_{mn}}^0$ in Table \ref{tab:scalability_M}.

\begin{table}[htbp]
\centering
\caption{Trace($\operatorname{Cov}(\hat{A}_{21})$) vs measurement (i.e., raw data) dim. $d_m$}
\begin{tabular}{|c|cccc|}
\hline
\multirow{2}{*}{\textbf{Order of Variance}} 
  & \multicolumn{4}{|c|}{\textbf{Measurement (Raw Data) Dimension\ ($d_m$)}} \\
  & $d_m=16$ & $d_m=32$ & $d_m=64$ & $d_m=128$ \\
\hline
$\sim10^{-6}$ & $\approx10^{-9}$ & $\approx10^{-5}$ & $\approx10^{-7}$ & $\approx10^{-7}$ \\
$\sim10^{-4}$ & $\approx10^{-8}$ & $\approx10^{-5}$ & $\approx10^{-7}$ & $\approx10^{-7}$ \\
$\sim10^{-2}$ & $1.0\times10^{-4}$ & $4.0\times10^{-4}$ & $1.0\times10^{-4}$ & $\approx10^{-5}$ \\
$\sim10^{0}$  & $1.6722$         & $1.9733$         & $2.4601$         & $1.5882$         \\
\hline
\end{tabular}
\label{tab:scalability_table1}
\end{table}
\begin{table}[htbp]
\centering
\caption{Trace($\operatorname{Cov}(\hat{A}_{21})$) vs.\ no.\ of clients $M$}
\begin{tabular}{|c|cccc|}
\hline
\multirow{2}{*}{\textbf{Order of Variance}} 
  & \multicolumn{4}{|c|}{\textbf{Number of Clients ($M$)}} \\
  & $M=2$ & $M=4$ & $M=8$ & $M=16$ \\
\hline
$\sim10^{-6}$ & $\approx10^{-5}$ & $\approx10^{-5}$ & $\approx10^{-5}$ & $\approx10^{-6}$ \\
$\sim10^{-4}$ & $\approx10^{-5}$ & $\approx10^{-5}$ & $\approx10^{-5}$ & $\approx10^{-6}$\\
$\sim10^{-2}$ & $0.0001$        & $0.0002$         & $0.0002$         & $0$\\
$\sim10^{0}$  & $1.7233$        & $1.6456$         & $2.5835$         & $4.7$\\
\hline
\end{tabular}
\label{tab:scalability_M}
\end{table}

\subsection{Ablation Study} \label{appendix: ablation_fedgc}

\begin{table*}[hpbt]
\centering
\caption{An ablation study comparing our method with a point estimate by \cite{mohanty2025federated}}
\label{tab:fedgc-uncertainty-ablation}
\small
\setlength{\tabcolsep}{4pt}
\renewcommand{\arraystretch}{1.15}
\begin{tabular*}{\textwidth}{@{\extracolsep{\fill}}l*{12}{c}@{}}
\toprule
& \multicolumn{4}{c}{Synthetic (9-edge)}
& \multicolumn{4}{c}{HAI}
& \multicolumn{4}{c}{TEP} \\
\cmidrule(lr){2-5}\cmidrule(lr){6-9}\cmidrule(l){10-13}
Method
& F1 & AUROC & AUPRC & SHD
& AC@1 & AC@2 & Density & \makecell{\# CC-\\edges}
& AC@1 & AC@2 & Density & \makecell{\# CC-\\edges} \\
\midrule
Point Estimate
& .588 & .871 & .844 & 84
& .641 & .934 & 1.000 & 938
& .623 & .776 & .998 & 2116 \\
\textbf{Ours}
& .710 & .871 & .844 & 49
& .626 & .822 & .504 & 144
& .623 & .727 & .702 & 1317 \\
\bottomrule
\end{tabular*}
\end{table*}

\subsection{QQ Plots for Causal Estimates}
\label{appendix:qqplots}

The post-training statistical hypothesis test proposed in Section~\ref{sec:utility} assumes approximate normality of the converged causal estimates $\hat{A}_{mn}$. To validate this assumption empirically, Fig.~\ref{fig:qqplots_amn} presents QQ plots for all entries of four randomly selected off-diagonal blocks of the estimated causal matrix on the synthetic dataset. 
For each selected block $\hat{A}_{mn}$, we generate 30 bootstrap resamples of the dataset and collect the corresponding converged causal estimates. We then compare their empirical quantiles against the theoretical quantiles of a Gaussian distribution. As shown in Fig.~\ref{fig:qqplots_amn}, the empirical quantiles closely follow the theoretical Gaussian quantiles across most of the distribution, indicating that the converged causal estimates are approximately normal.

\begin{figure}[htbp]
    \centering

    \begin{subfigure}[b]{0.48\columnwidth}
        \centering
        \includegraphics[width=\linewidth]{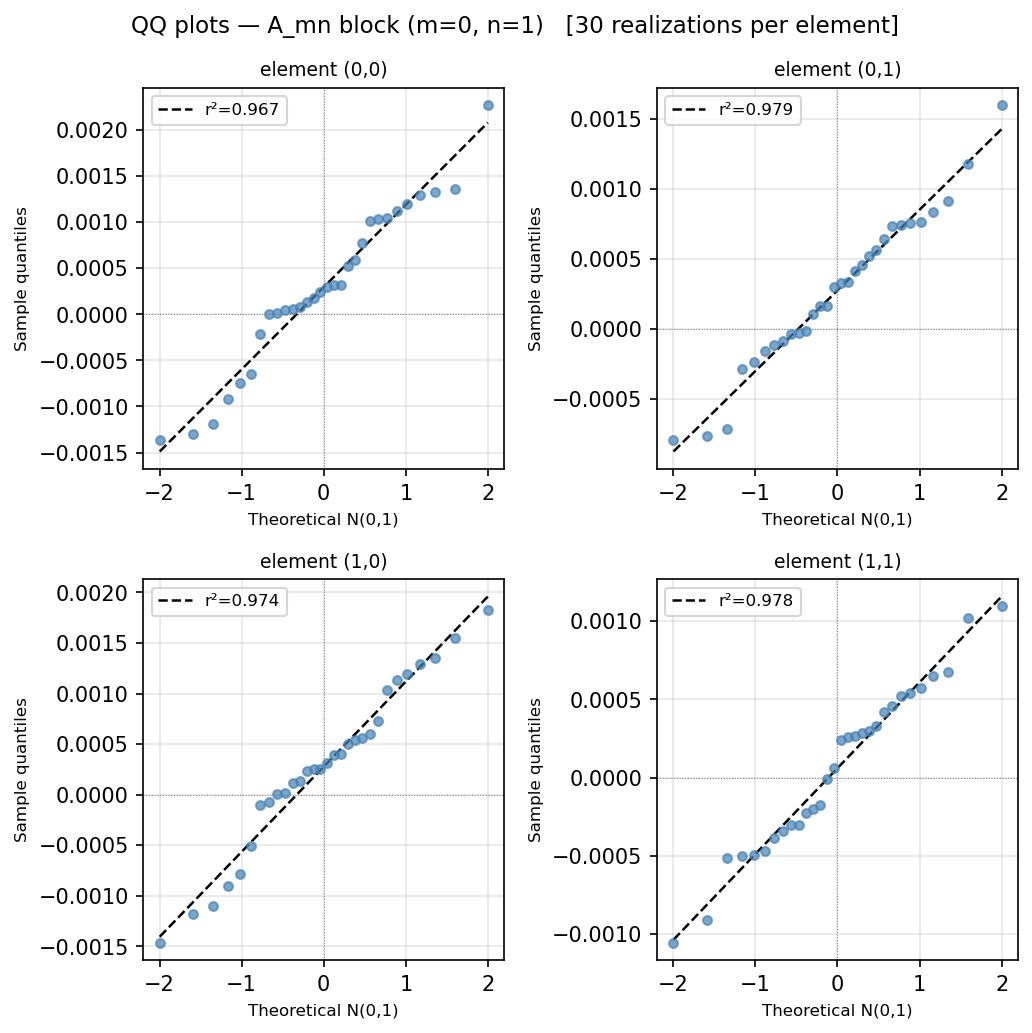}
        \caption{$\hat{A}_{12}$}
    \end{subfigure}
    \hfill
    \begin{subfigure}[b]{0.48\columnwidth}
        \centering
        \includegraphics[width=\linewidth]{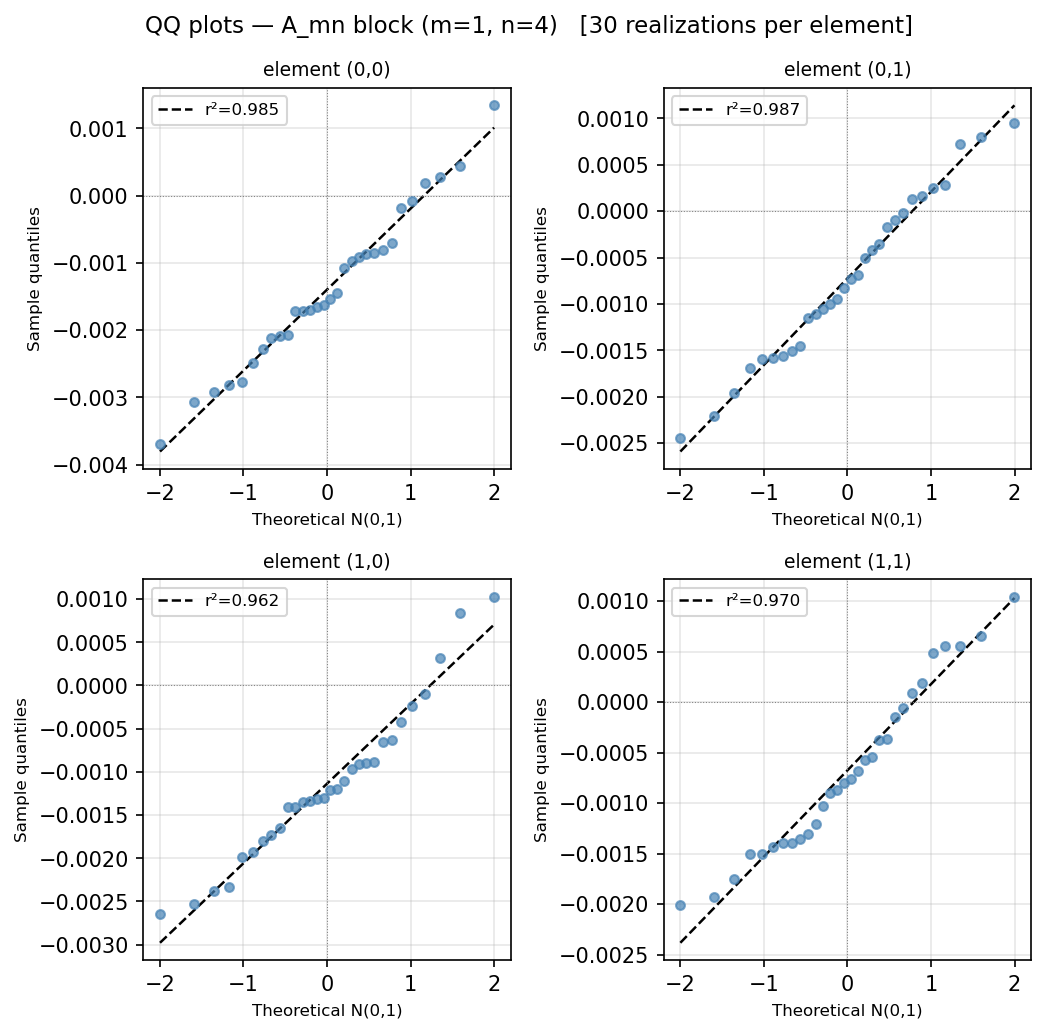}
        \caption{$\hat{A}_{25}$}
    \end{subfigure}

    \vspace{0.2cm}

    \begin{subfigure}[b]{0.48\columnwidth}
        \centering
        \includegraphics[width=\linewidth]{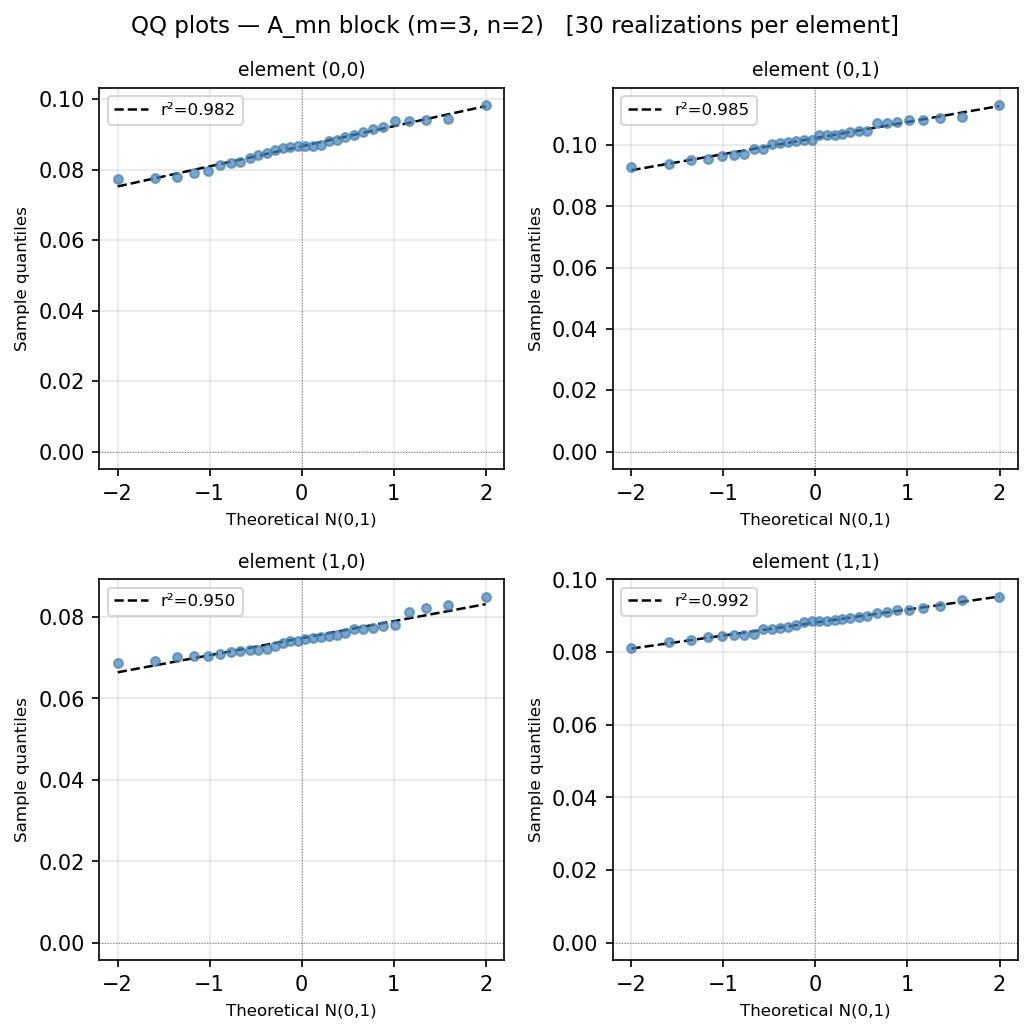}
        \caption{$\hat{A}_{43}$}
    \end{subfigure}
    \hfill
    \begin{subfigure}[b]{0.48\columnwidth}
        \centering
        \includegraphics[width=\linewidth]{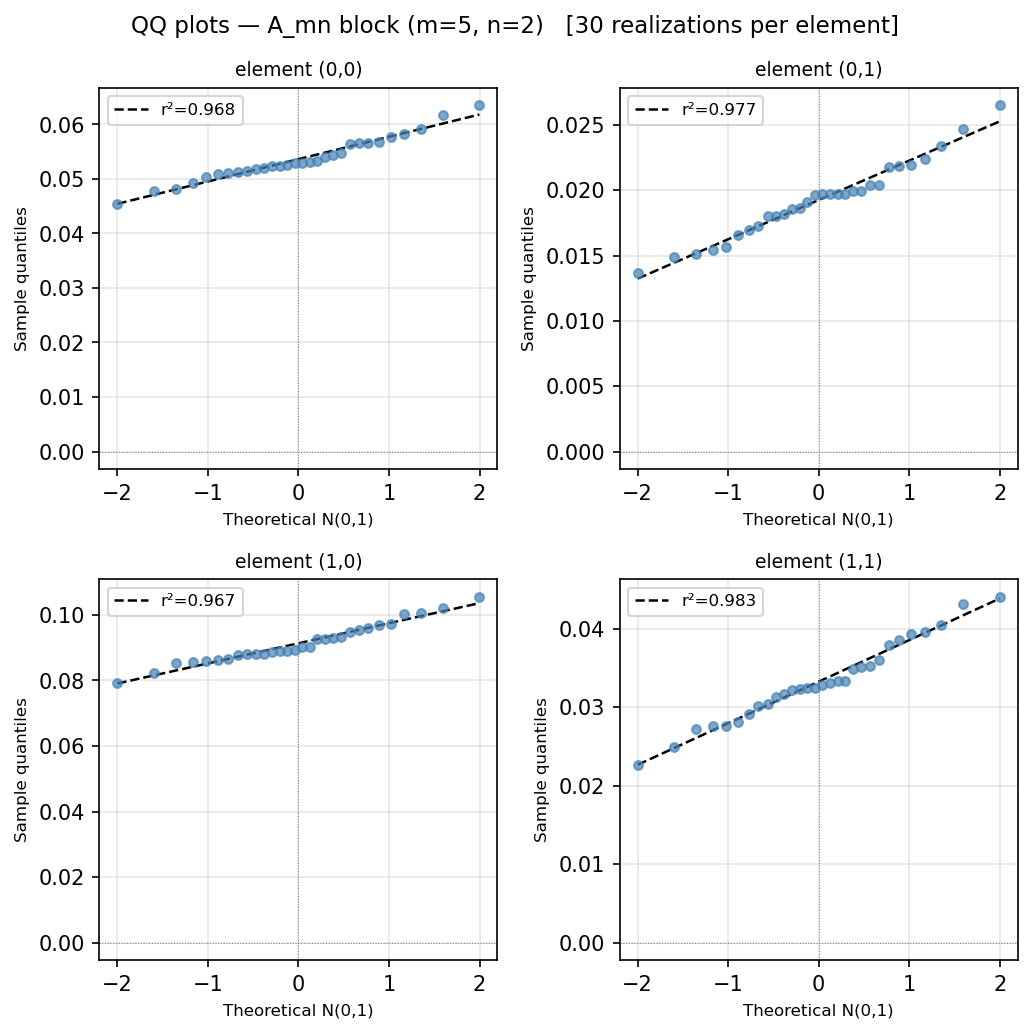}
        \caption{$\hat{A}_{63}$}
    \end{subfigure}

    \caption{QQ plots for the causal estimates $\hat{A}_{mn}$ on the synthetic dataset.}
    \label{fig:qqplots_amn}
\end{figure}

{
\subsection{Effect of DP Noise on Uncertainty and Causal Accuracy}
\label{sec:dp-noise}

In this subsection, we study how differentially private (DP) Gaussian noise injected
into the federated messages affects \textbf{(i)} steady-state parameter uncertainty and
\textbf{(ii)} accuracy of causal link detection. We consider two communication directions:
\emph{client$\to$server}, where noise is added to the states sent by the
clients to the server, and
\emph{server$\to$client}, where noise is added to the gradients sent
back to the clients. In both cases, we use the Gaussian mechanism and sweep the
noise standard deviation over a log-scale grid
$\sigma \in \{10^{-6},10^{-5},\dots,10^{-1}\}$. 

\begin{figure}[htbp]
    \centering
    \begin{subfigure}[b]{0.475\linewidth}
        \centering
        \includegraphics[width=\linewidth]{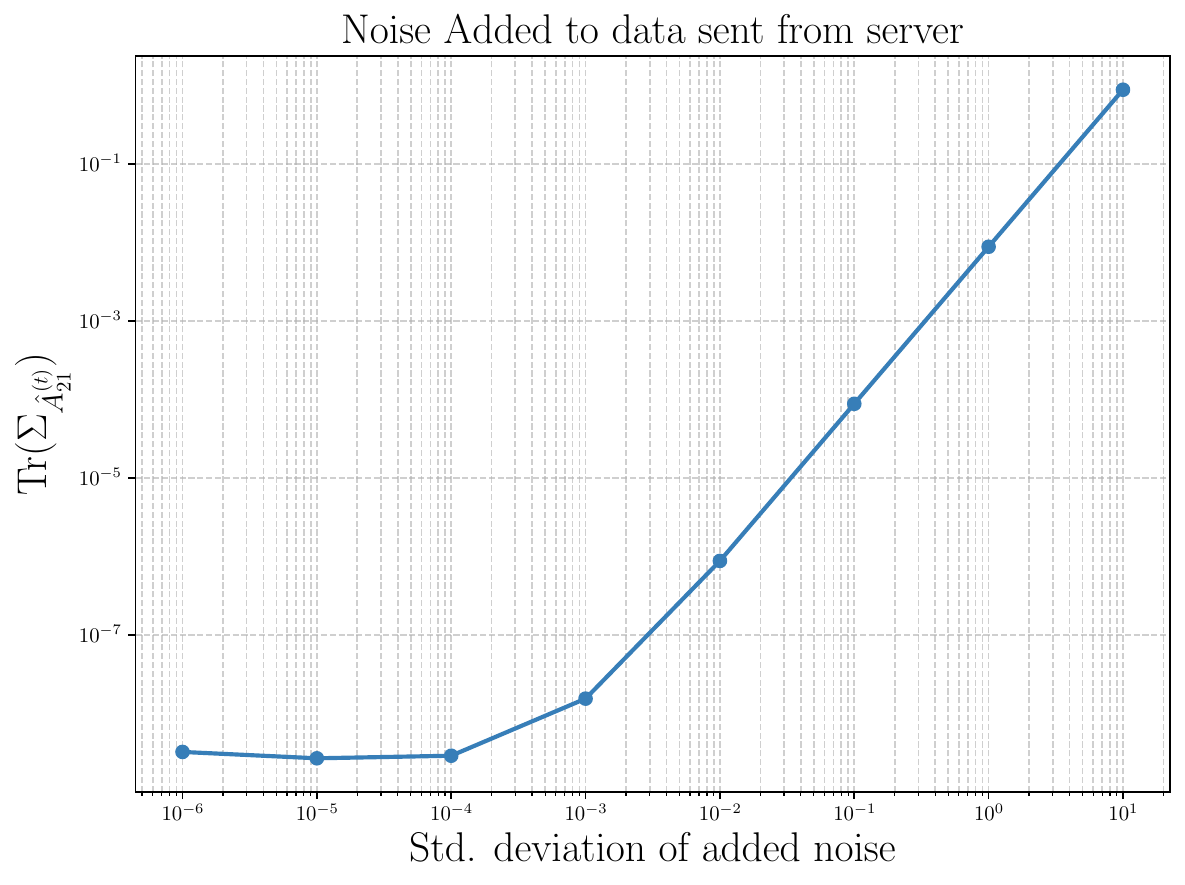}
        \caption{Steady-state variance, server$\to$client noise.}
        \label{fig:dp-ss-server-to-client}
    \end{subfigure}
    \hfill
    \begin{subfigure}[b]{0.475\linewidth}
        \centering
        \includegraphics[width=\linewidth]{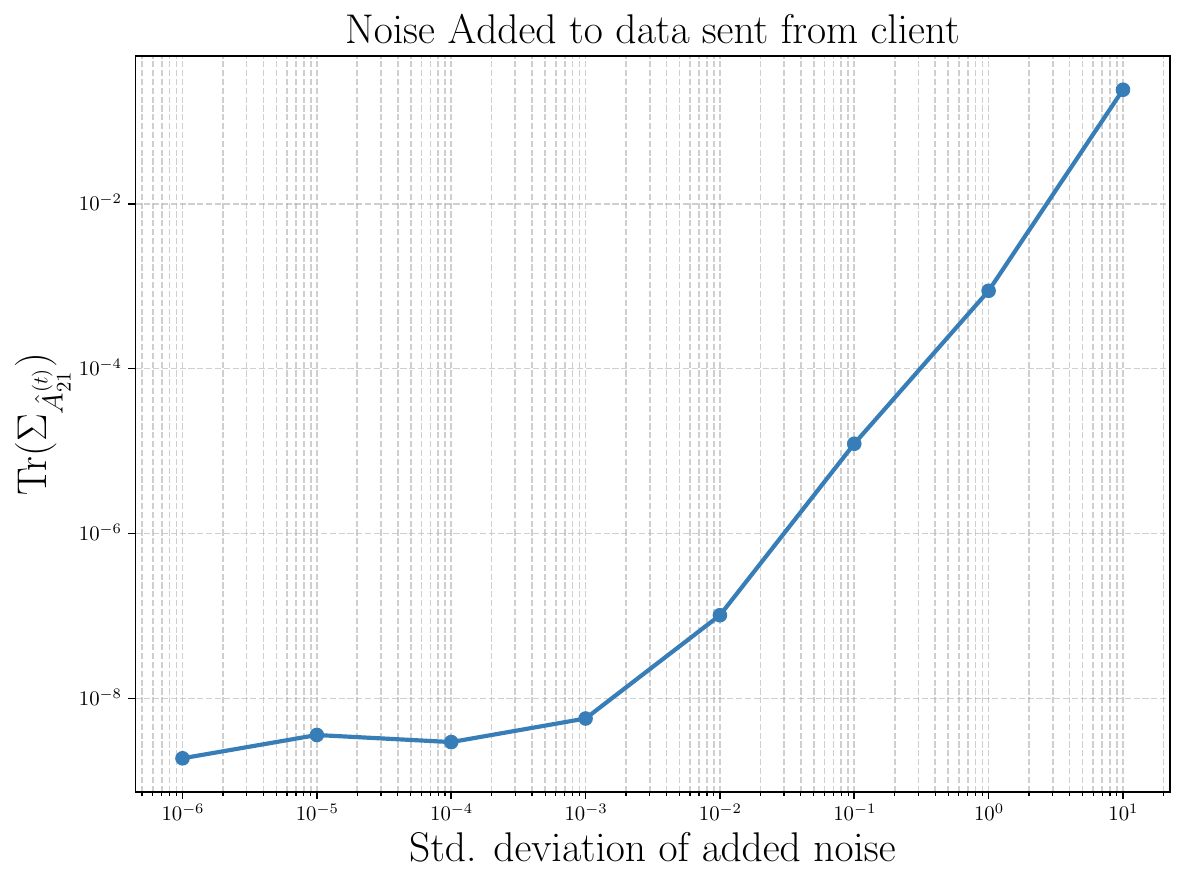}
        \caption{Steady-state variance, client$\to$server noise.}
        \label{fig:dp-ss-client-to-server}
    \end{subfigure}

    \vspace{0.5em}

    \begin{subfigure}[b]{0.475\linewidth}
        \centering
        \includegraphics[width=\linewidth]{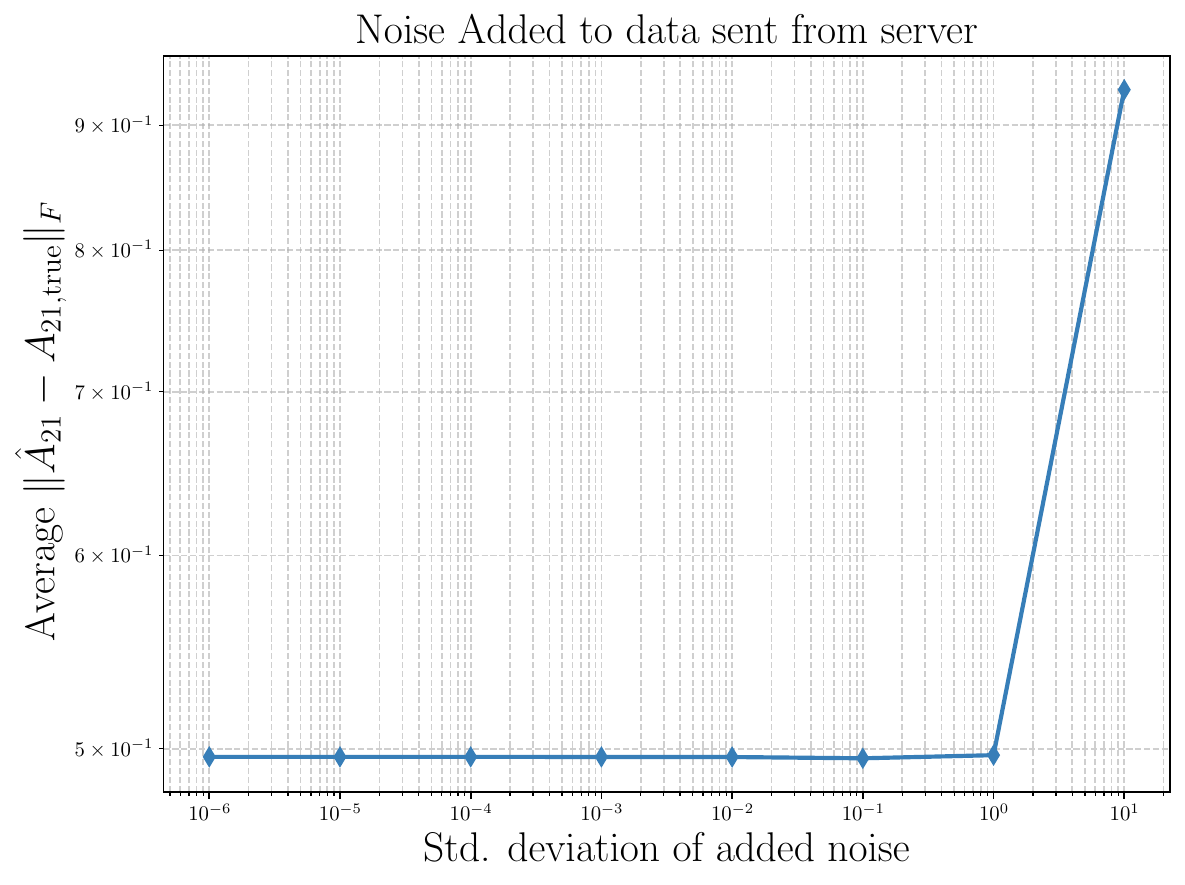}
        \caption{Causal accuracy, server$\to$client noise.}
        \label{fig:dp-causal-server-to-client}
    \end{subfigure}
    \hfill
    \begin{subfigure}[b]{0.475\linewidth}
        \centering
        \includegraphics[width=\linewidth]{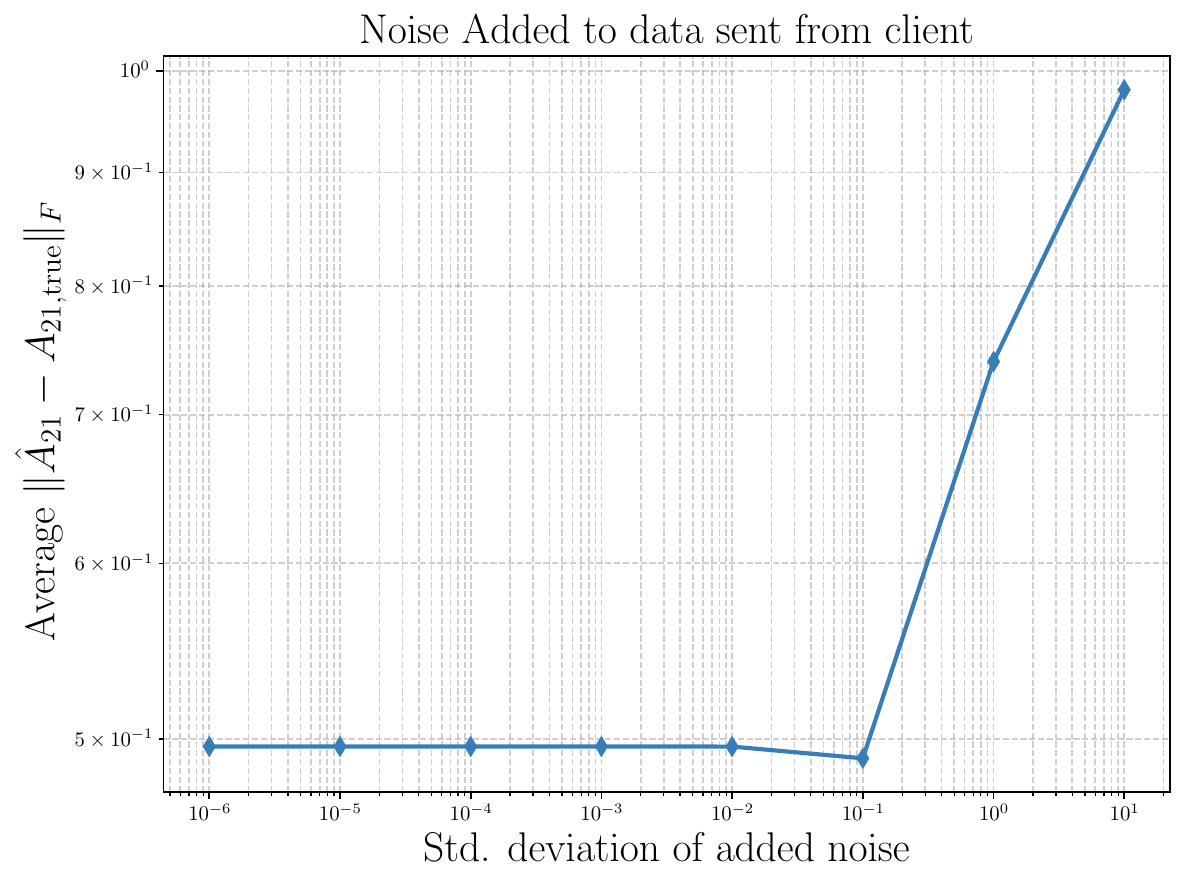}
        \caption{Causal accuracy, client$\to$server noise.}
        \label{fig:dp-causal-client-to-server}
    \end{subfigure}

    \caption{Effect of DP Gaussian noise on steady-state uncertainty and causal
    link detection in the synthetic dataset. Top row: trace of the propagated
    covariance $\mathrm{tr}(\Sigma_{A_{21}})$ as a function of the noise scale
    $\sigma$. Bottom row: Frobenius error $\|\hat{A}_{21} - A_{21,\mathrm{true}}\|_F$
    versus $\sigma$.}
    \label{fig:dp-noise-synthetic}
\end{figure}

Figure~\ref{fig:dp-noise-synthetic} summarizes the results on the synthetic
dataset. The top row shows the steady-state uncertainty of the cross-client
Granger block $A_{21}$, measured via $\mathrm{tr}(\Sigma_{A_{21}})$, as a
function of the injected noise scale $\sigma$. The bottom row reports the
corresponding causal estimation accuracy, measured by the Frobenius norm
$\|\hat{A}_{21} - A_{21, \mathrm{true}}\|_F$ and normalized for visualization.
Panels labelled ``server$\to$client'' and ``client$\to$server'' correspond to the
direction in which DP noise is added. Across all configurations, the uncertainty recursions remain numerically stable
for the entire range of $\sigma$, and the steady-state variance increases
smoothly as the DP noise level grows. For small to moderate noise
($\sigma \sim 10^{-3}$), the inflation in $\mathrm{tr}(\Sigma_{A_{21}})$ is
modest and the causal estimation error remains close to the non-DP baseline in
both communication directions. Only for the largest noise levels
($\sigma \approx 10^{-2}$–$10^{-1}$) do we observe a pronounced increase in
variance accompanied by a degradation in causal link detection, as expected
from strong DP perturbations. These experiments confirm that our uncertainty
propagation framework is robust to reasonably strong DP noise and that the
empirical privacy–utility trade-off behaves consistently with the Gaussian
mechanism in both communication channels.
}

\begin{figure}[htbp]
  \centering
  \begin{subfigure}[b]{0.33\columnwidth}
    \centering
    \includegraphics[width=\textwidth]{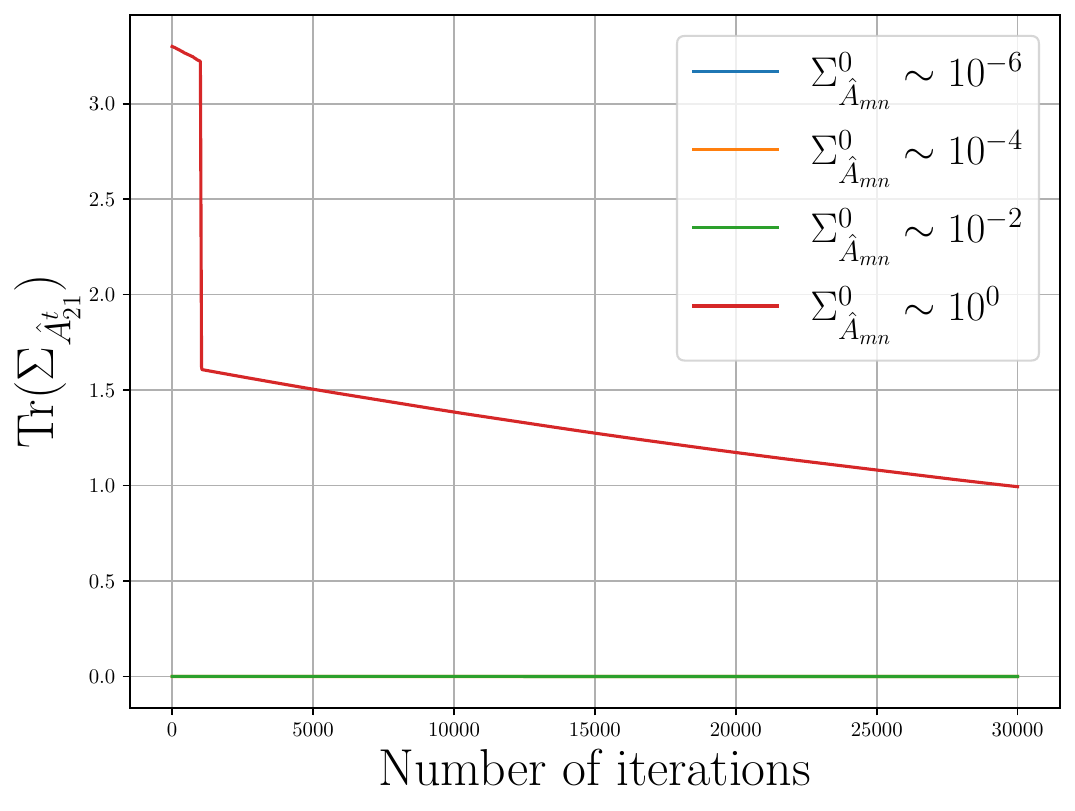}
    \caption{} \label{1a}
  \end{subfigure}%
  \begin{subfigure}[b]{0.33\columnwidth}
    \centering
    \includegraphics[width=\textwidth]{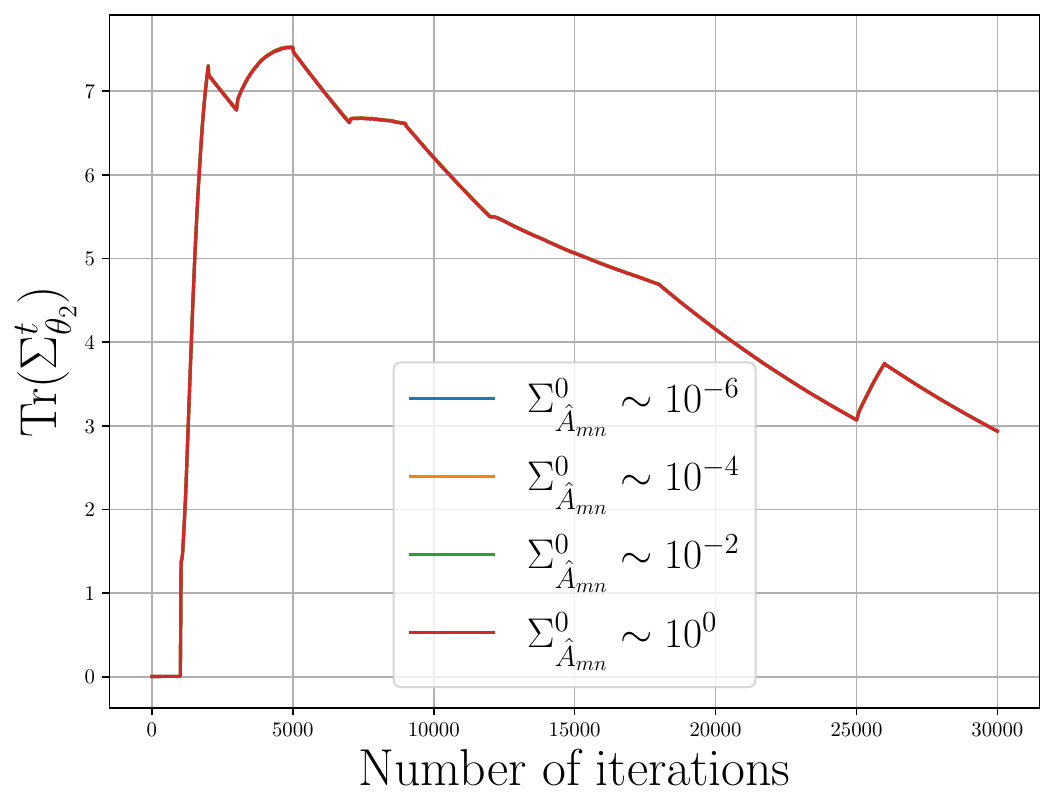}
    \caption{} \label{1b}
  \end{subfigure}%
  \begin{subfigure}[b]{0.33\columnwidth}
    \centering
    \includegraphics[width=\textwidth]{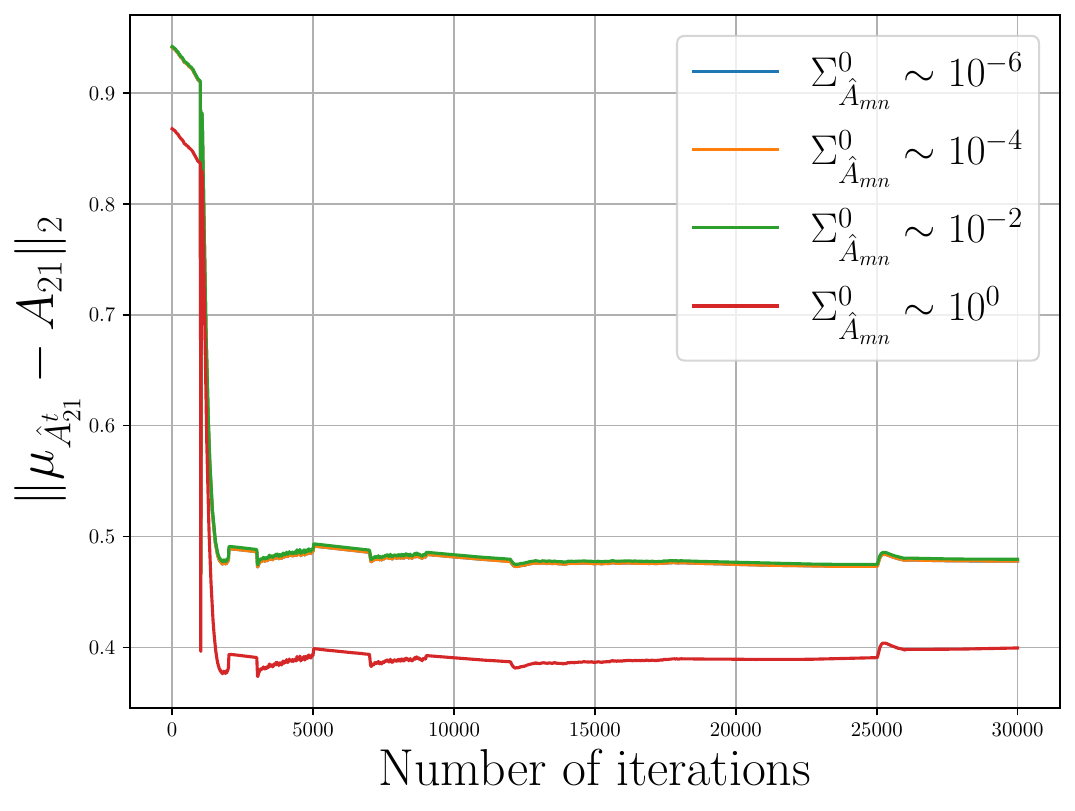}
    \caption{} \label{1c}
  \end{subfigure}
  \begin{subfigure}[b]{0.33\columnwidth}
    \centering
    \includegraphics[width=\textwidth]{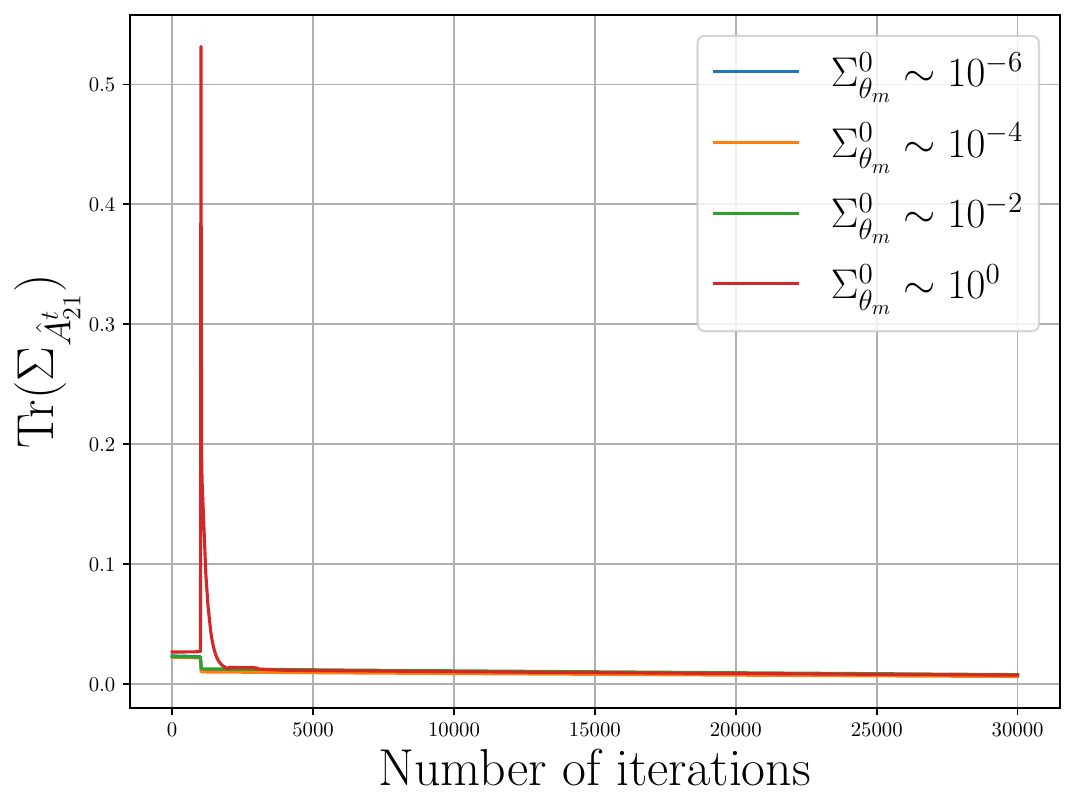}
    \caption{} \label{1a}
  \end{subfigure}%
  \begin{subfigure}[b]{0.33\columnwidth}
    \centering
    \includegraphics[width=\textwidth]{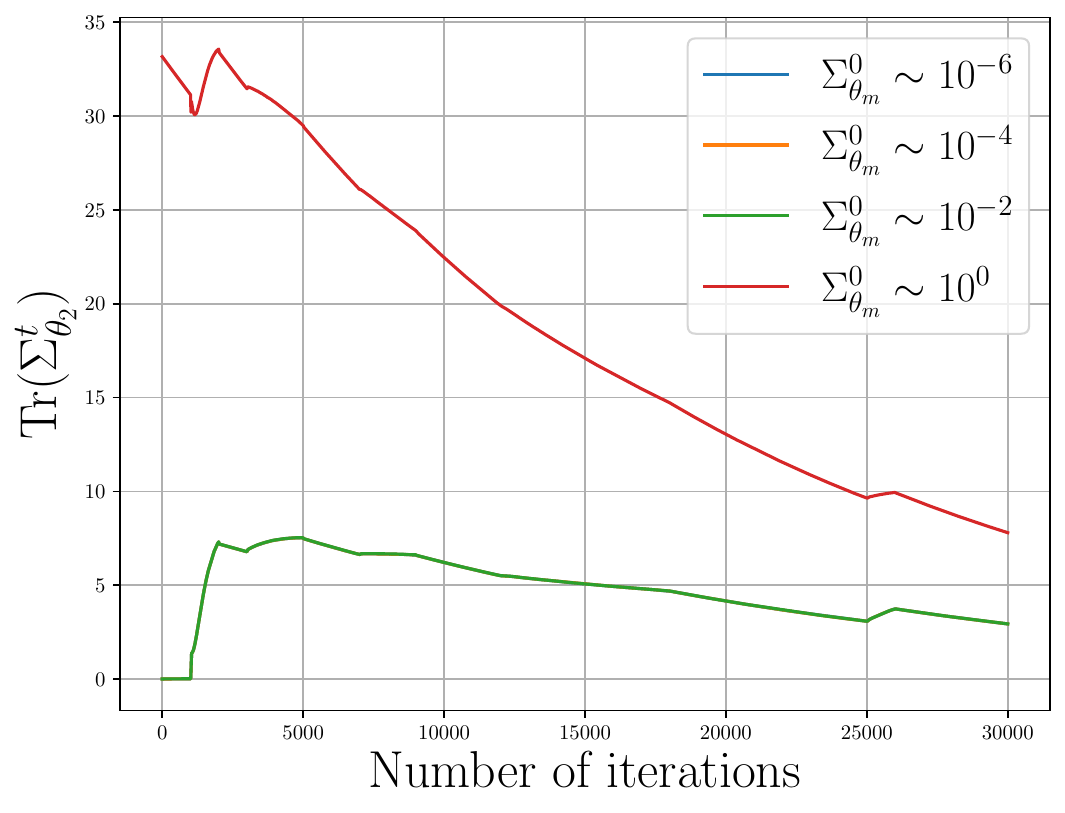}
    \caption{} \label{1b}
  \end{subfigure}%
  \begin{subfigure}[b]{0.33\columnwidth}
    \centering
    \includegraphics[width=\textwidth]{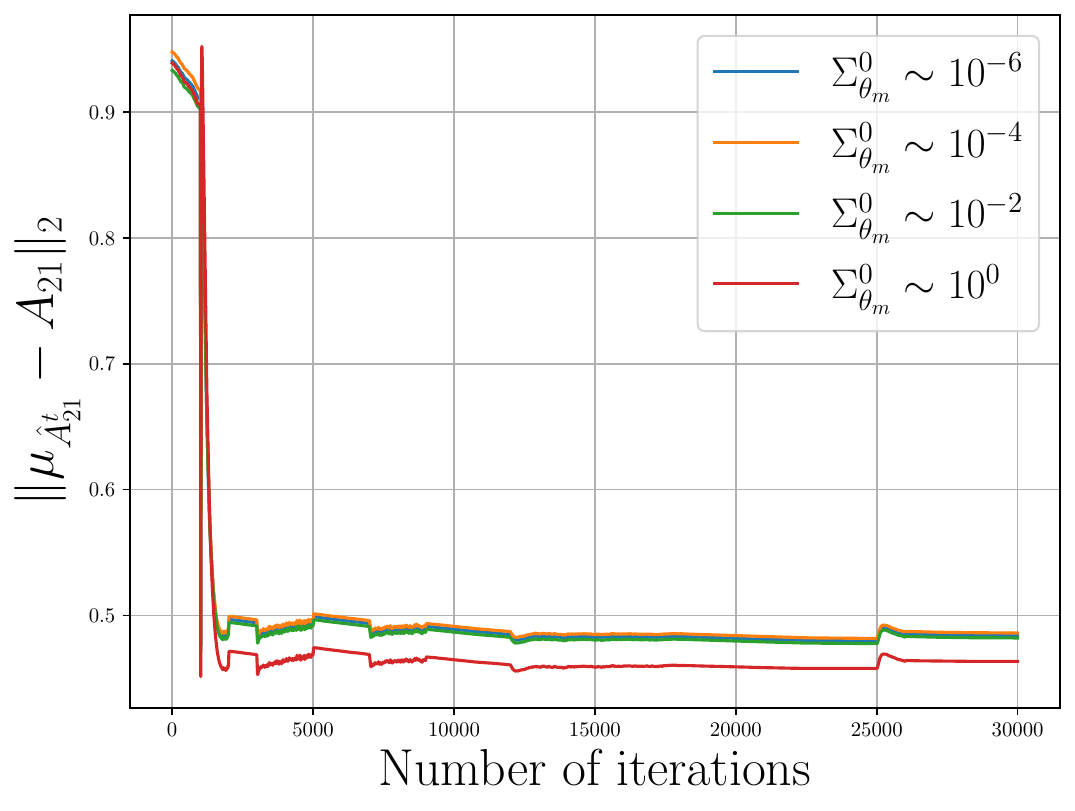}
    \caption{} \label{1c}
  \end{subfigure}
  \caption{%
  Uncertainty prop. highlighting
    (a) $\mathrm{Tr}(\Sigma_{\hat{A}_{21}^t})$, 
    (b) $\mathrm{Tr}(\Sigma_{\theta_2}^t)$, 
    (c) $\|\mu_{\hat{A}_{21}^t}-A_{21}\|_2$ vs iter. $t$ for different levels of $\Sigma_{\hat{A}_{mn}^0}$ (top row), and $\Sigma_{\theta_m}^0$ (bottom row)
  }
  \label{fig:epistemic_sigma_theta}
\end{figure}

\begin{figure}[htbp]
  \centering
  \begin{subfigure}[b]{0.33\columnwidth}
    \centering
    \includegraphics[width=\textwidth]{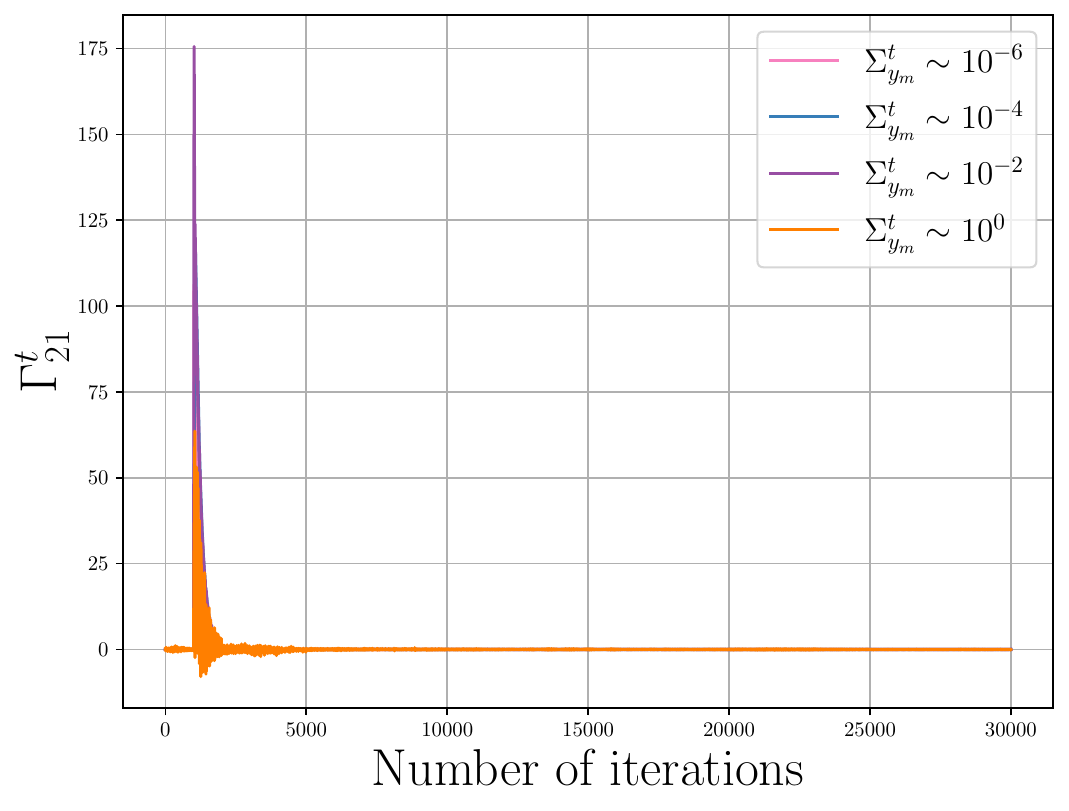}
    \caption{} \label{1a}
  \end{subfigure}%
  \begin{subfigure}[b]{0.33\columnwidth}
    \centering
    \includegraphics[width=\textwidth]{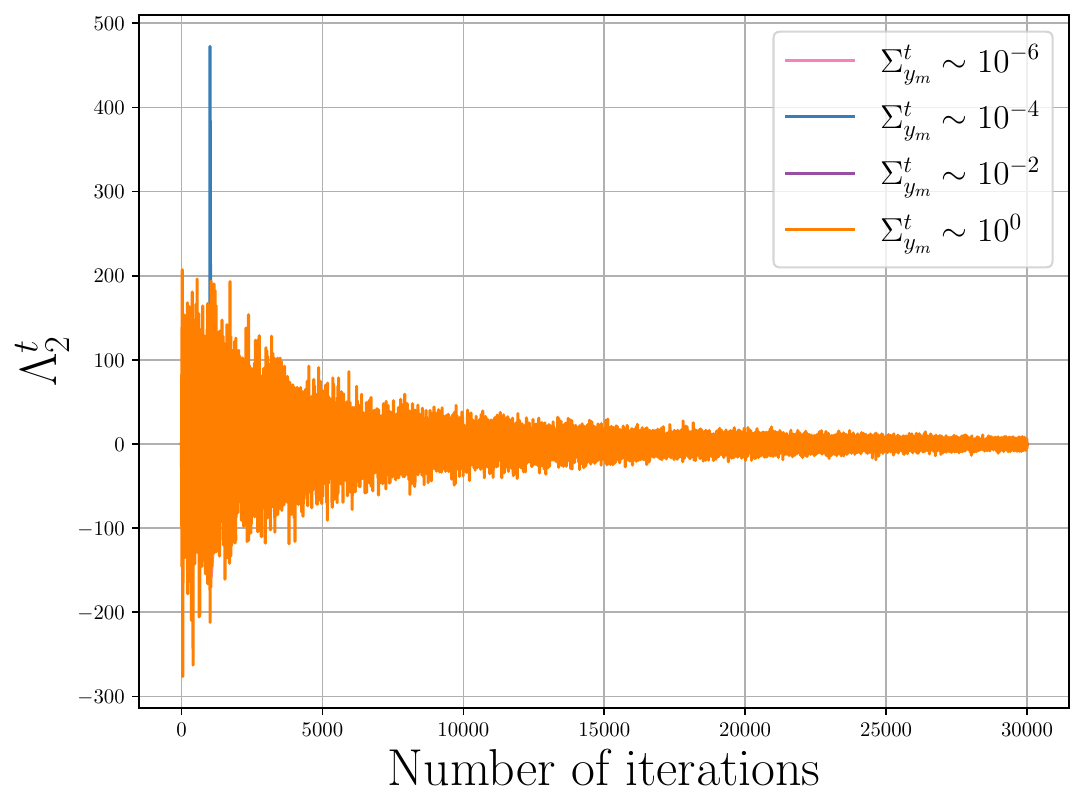}
    \caption{} \label{1b}
  \end{subfigure}%
  \begin{subfigure}[b]{0.33\columnwidth}
    \centering
    \includegraphics[width=\textwidth]{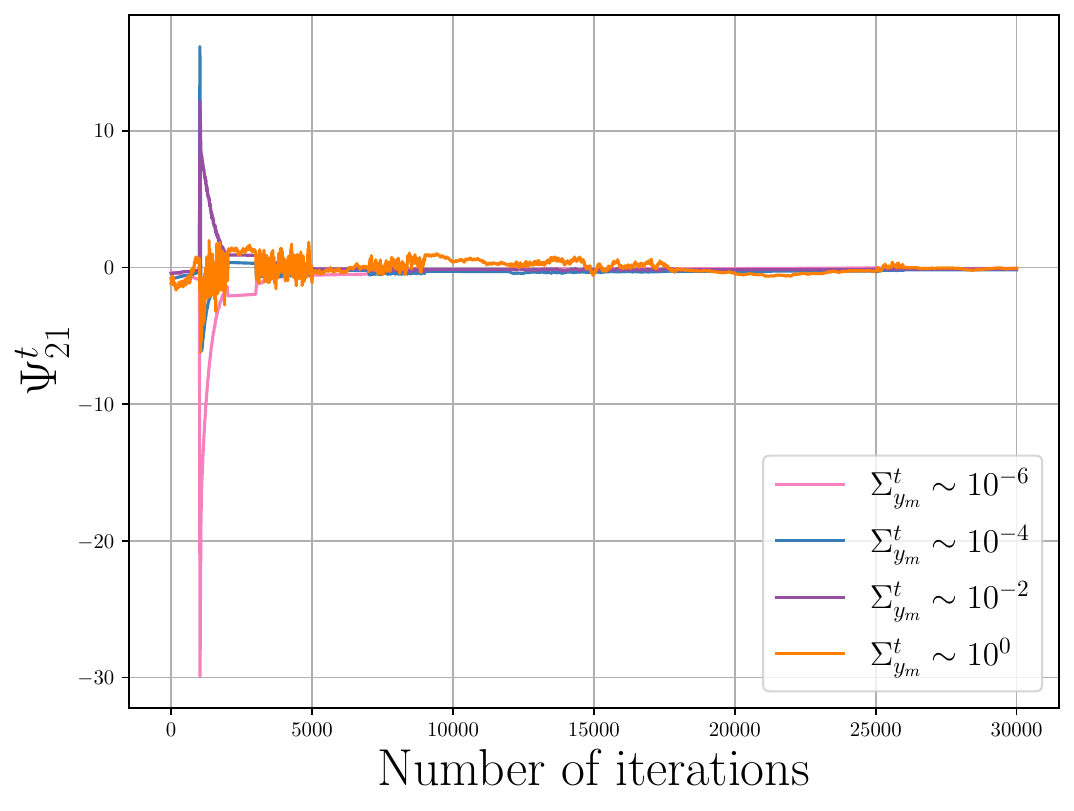}
    \caption{} \label{1c}
  \end{subfigure}

  \caption{%
    Uncertainty propagation in the cross-covariance terms during FedGC learning for different regimes of $\Sigma_{y_m}^t$
    (a) $\Gamma_{21}^t$ vs iterations, 
    (b) $\Lambda_2^t$ vs iterations, 
    (c) $\Psi_{21}^t$ vs iterations.%
  in the synthetic dataset}
  \label{fig:aleatoric_cross_cov_sigma_y}
\end{figure}

\begin{figure}[htbp]
  \centering
  \begin{subfigure}[b]{0.33\columnwidth}
    \centering
    \includegraphics[width=\textwidth]{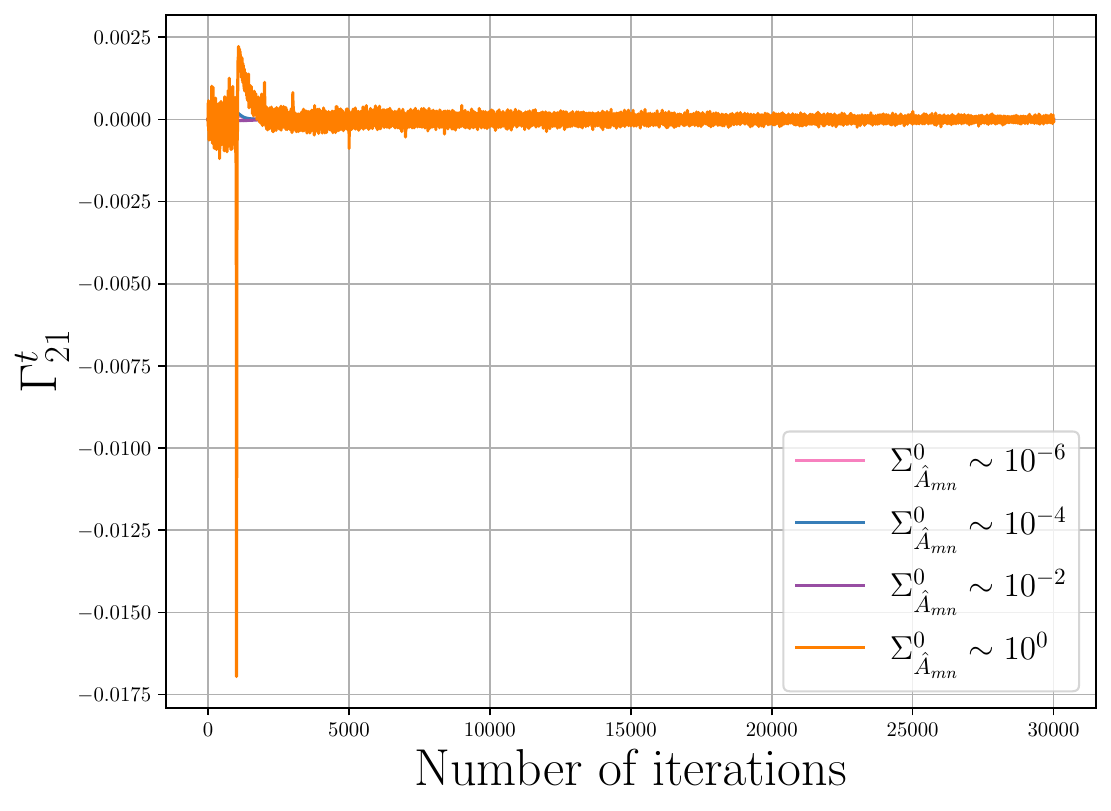}
    \caption{} \label{1a}
  \end{subfigure}%
  \begin{subfigure}[b]{0.33\columnwidth}
    \centering
    \includegraphics[width=\textwidth]{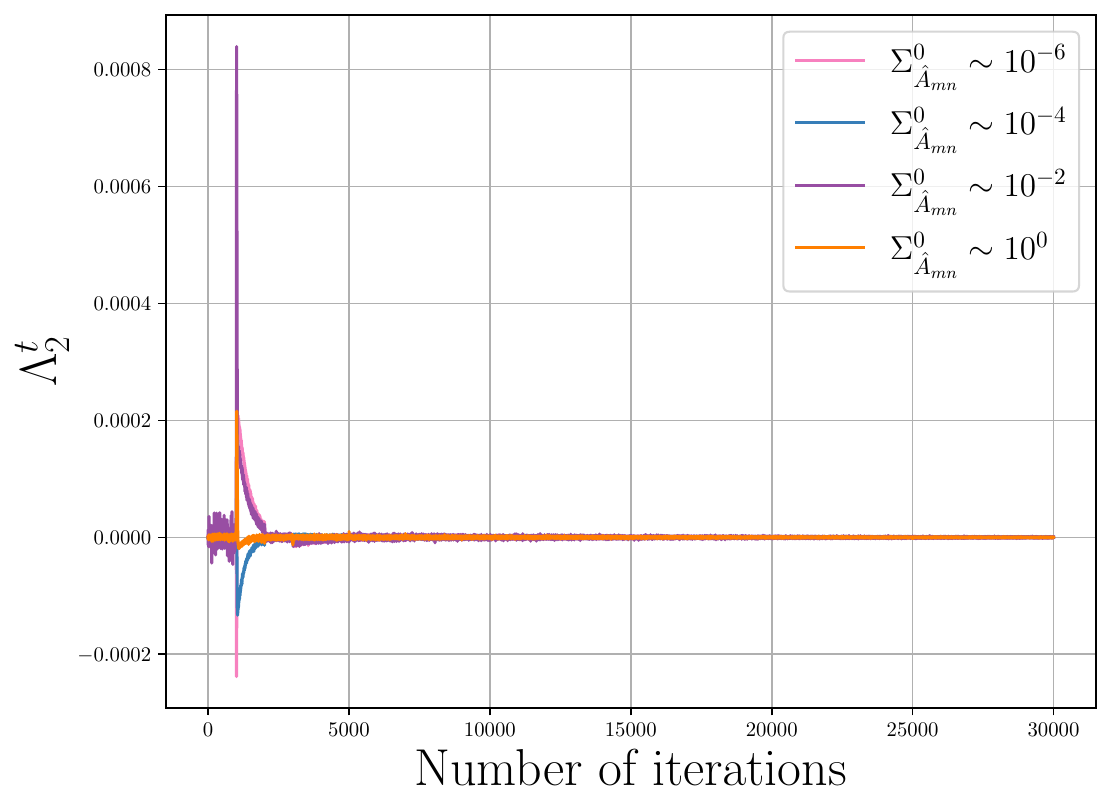}
    \caption{} \label{1b}
  \end{subfigure}%
  \begin{subfigure}[b]{0.33\columnwidth}
    \centering
    \includegraphics[width=\textwidth]{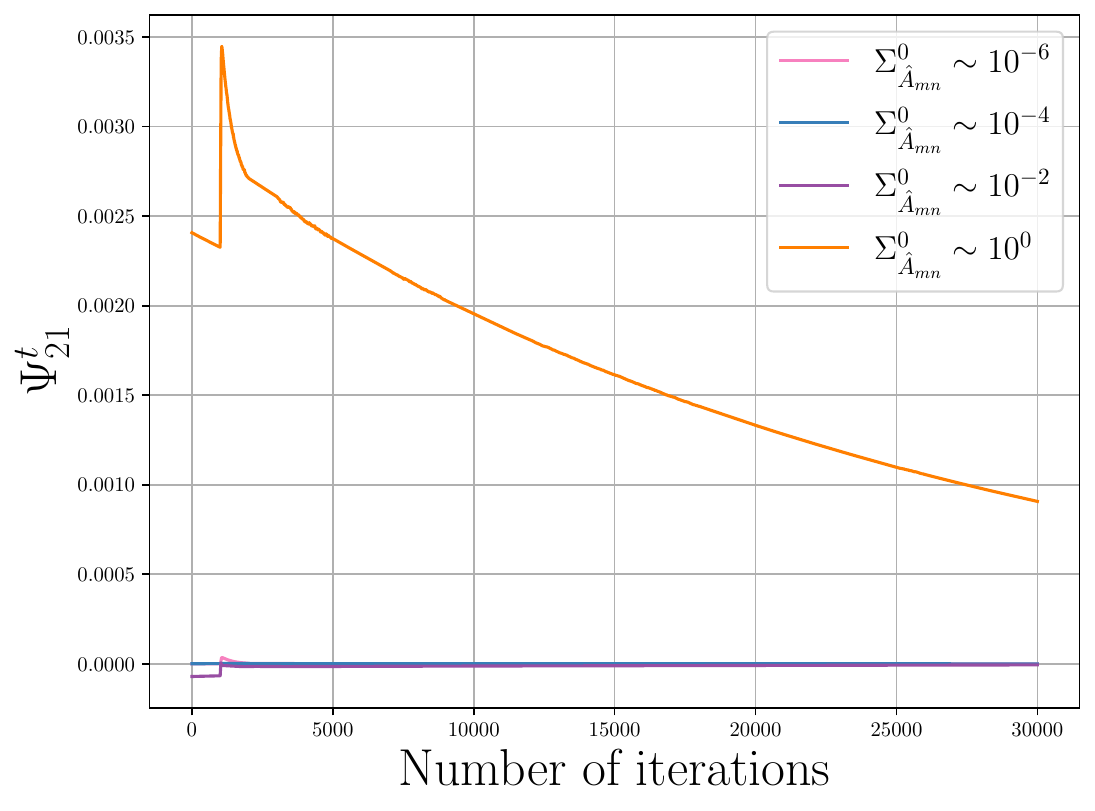}
    \caption{} \label{1c}
  \end{subfigure}

  \caption{%
    Uncertainty propagation in the cross-covariance terms during FedGC learning for different regimes of $\Sigma_{\hat{A}_{mn}}^0$
    (a) $\Gamma_{21}^t$ vs iterations, 
    (b) $\Lambda_2^t$ vs iterations, 
    (c) $\Psi_{21}^t$ vs iterations.%
  in the synthetic dataset}
  \label{fig:epistemic_cross_cov_Sigma_Amn}
\end{figure}

\begin{figure}[htbp]
  \centering
  \begin{subfigure}[b]{0.33\columnwidth}
    \centering
    \includegraphics[width=\textwidth]{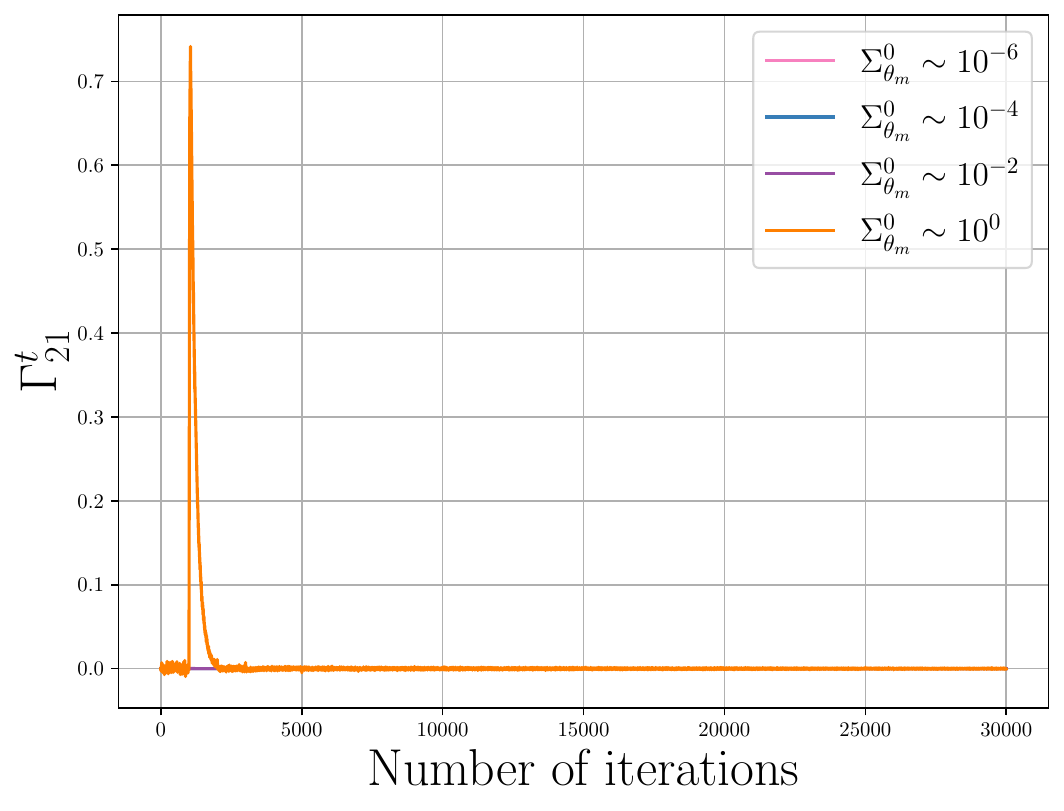}
    \caption{} \label{1a}
  \end{subfigure}%
  \begin{subfigure}[b]{0.33\columnwidth}
    \centering
    \includegraphics[width=\textwidth]{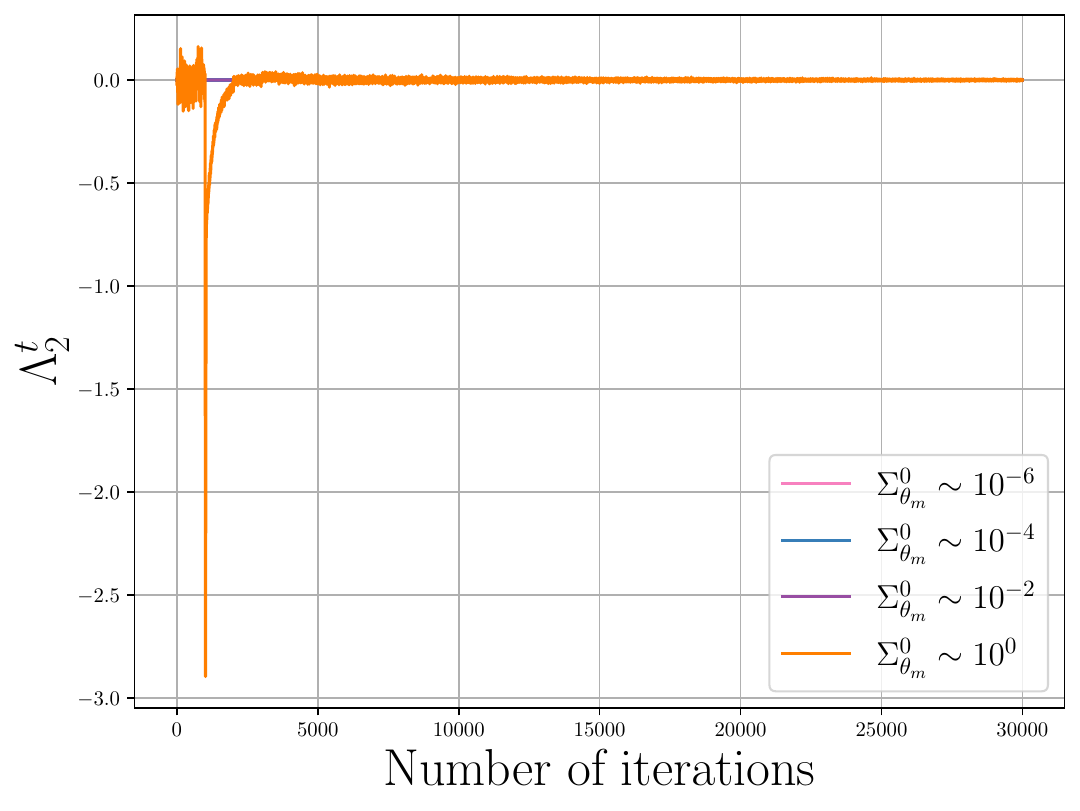}
    \caption{} \label{1b}
  \end{subfigure}%
  \begin{subfigure}[b]{0.33\columnwidth}
    \centering
    \includegraphics[width=\textwidth]{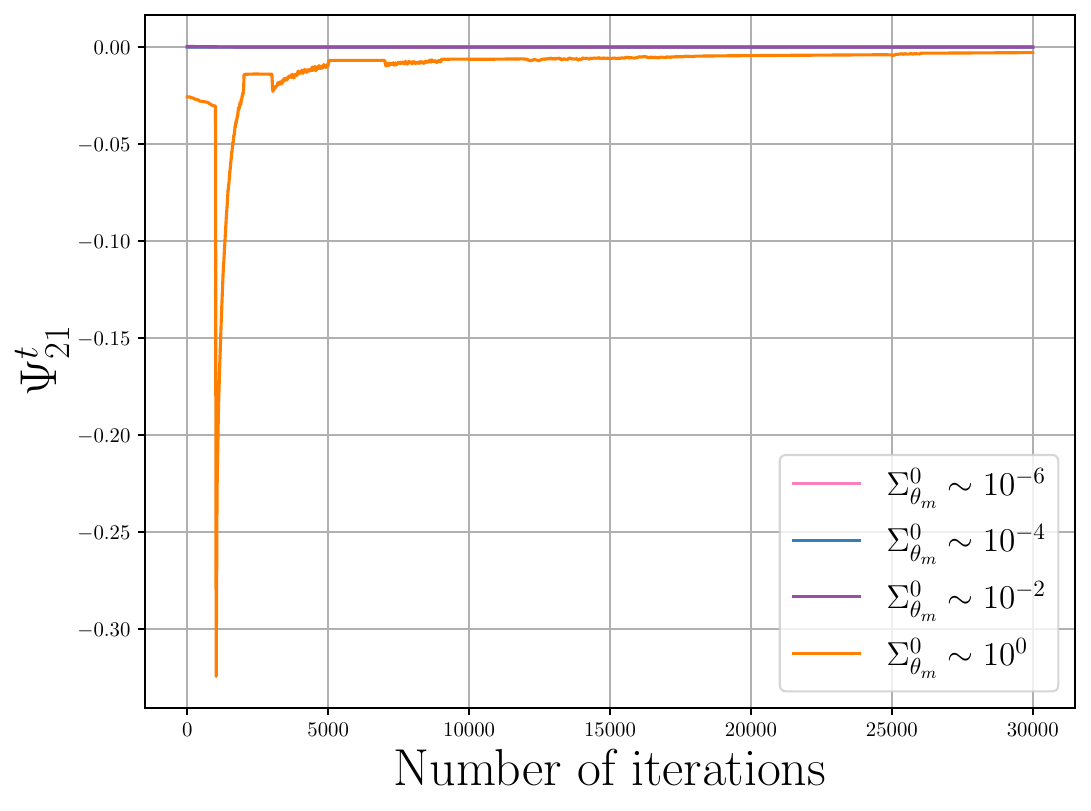}
    \caption{} \label{1c}
  \end{subfigure}

  \caption{%
    Uncertainty propagation in the cross-covariance terms during FedGC learning for different regimes of $\Sigma_{\theta_m^t}$
    (a) $\Gamma_{21}^t$ vs iterations, 
    (b) $\Lambda_2^t$ vs iterations, 
    (c) $\Psi_{21}^t$ vs iterations.%
  in the synthetic dataset}
  \label{fig:epistemic_cross_cov_Sigma_theta}
\end{figure}

\begin{figure}[htbp]
  \centering
  \begin{subfigure}[b]{0.4\columnwidth}
    \centering
    \includegraphics[width=\textwidth]{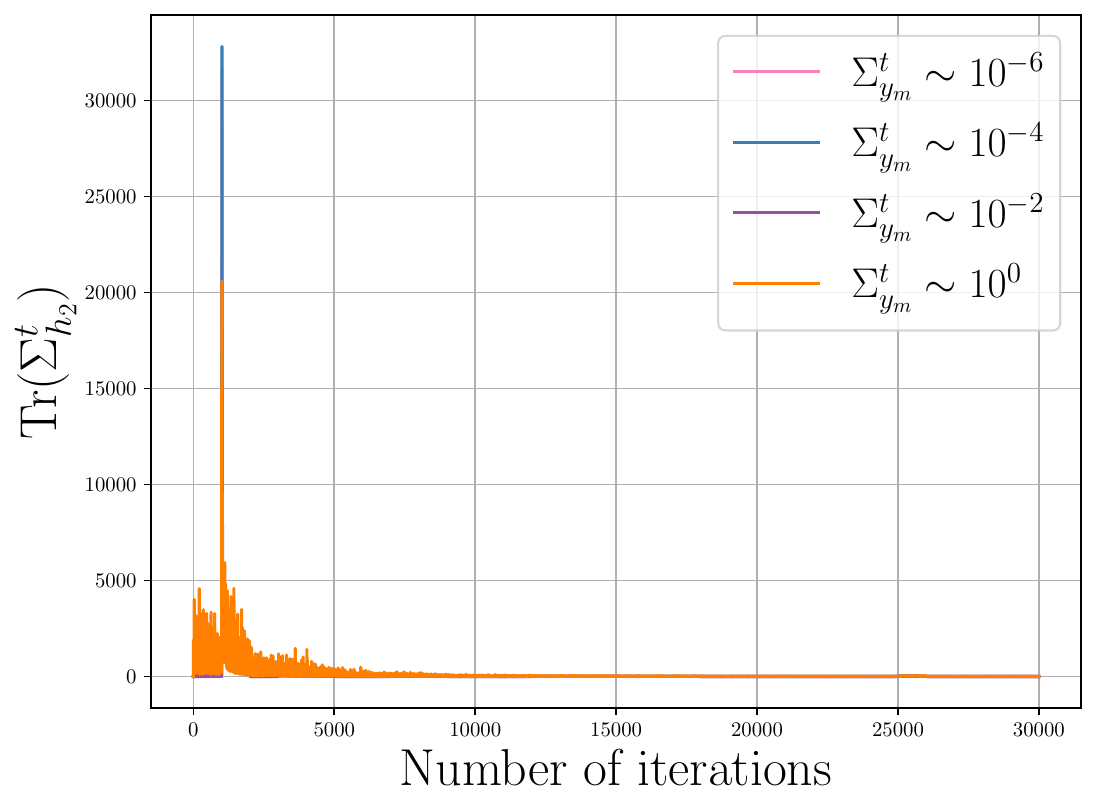}
    \caption{} \label{1a}
  \end{subfigure}%
  \begin{subfigure}[b]{0.4\columnwidth}
    \centering
    \includegraphics[width=\textwidth]{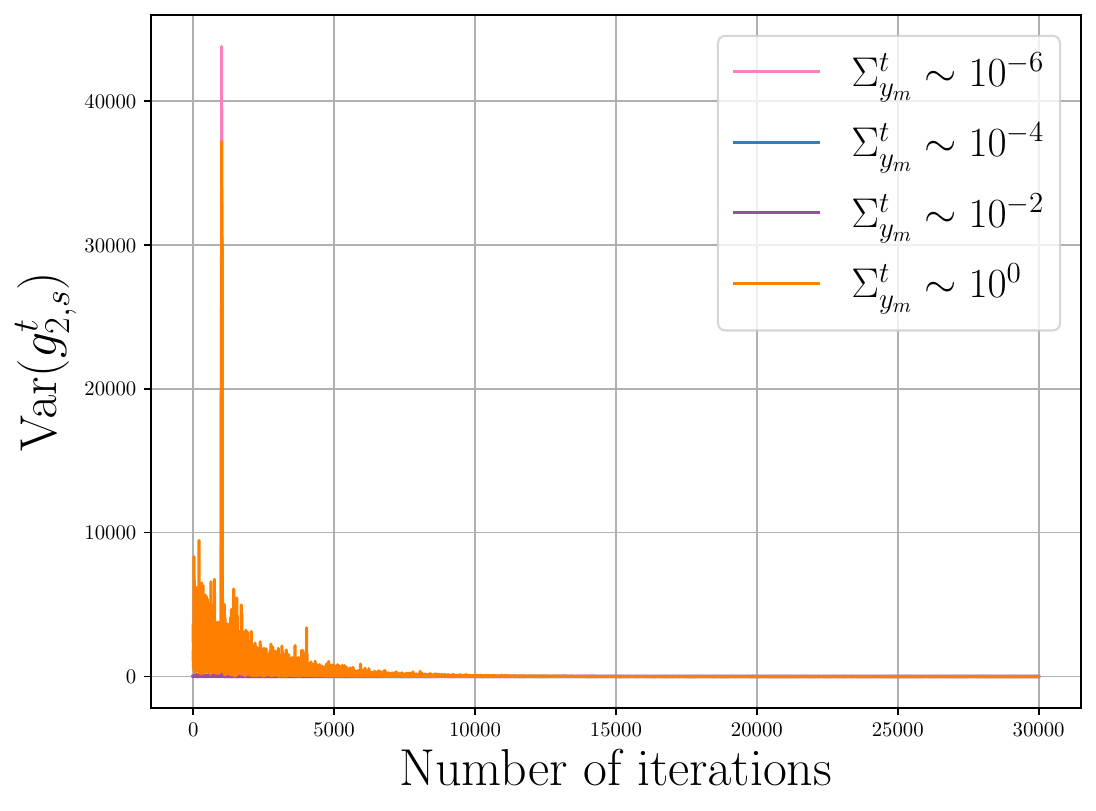}
    \caption{} \label{1b}
  \end{subfigure}%
  \caption{%
    Uncertainty propagation in the communicated terms during FedGC learning for different regimes of $\Sigma_{y_m}^t$
    (a) $\operatorname{Var}\bigl(g_{m, s}^{\,t}\bigr)$ vs iterations, 
    (b) $\Sigma_{h_m}^t$ vs iterations, 
  in the synthetic dataset}
  \label{fig:aleatoric_comm_sigma_y}
\end{figure}

\begin{figure}[htbp]
  \centering
  \begin{subfigure}[b]{0.4\columnwidth}
    \centering
    \includegraphics[width=\textwidth]{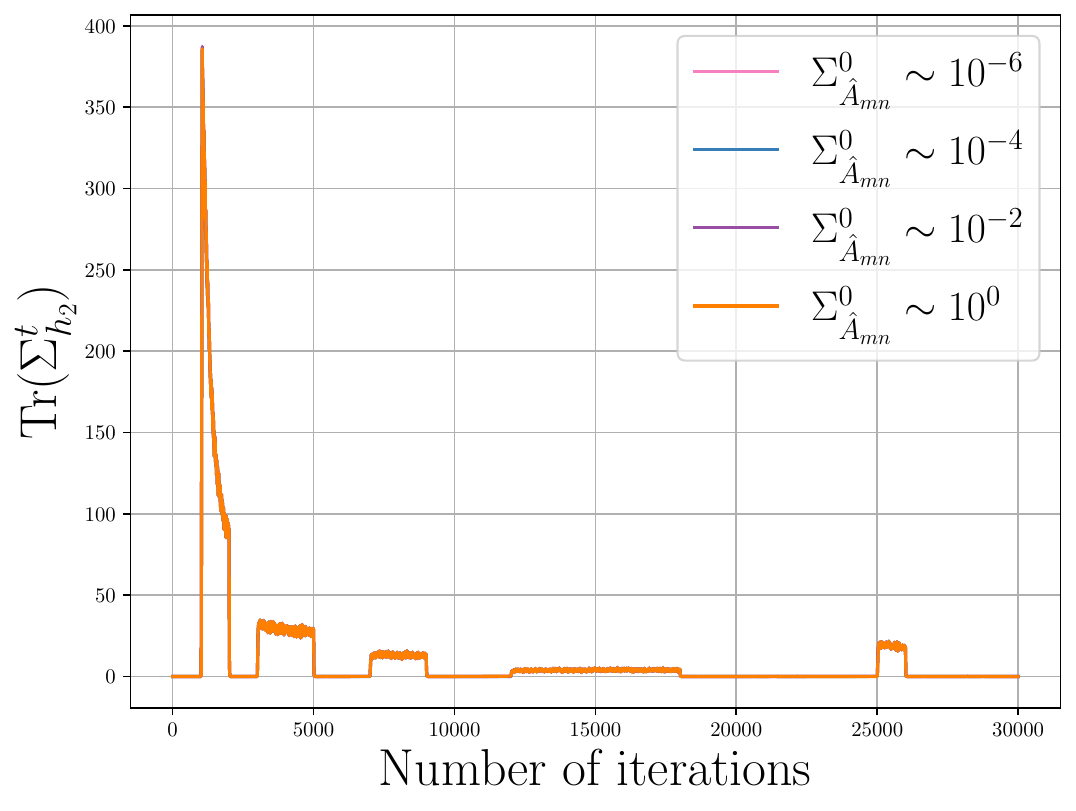}
    \caption{} \label{1a}
  \end{subfigure}%
  \begin{subfigure}[b]{0.4\columnwidth}
    \centering
    \includegraphics[width=\textwidth]{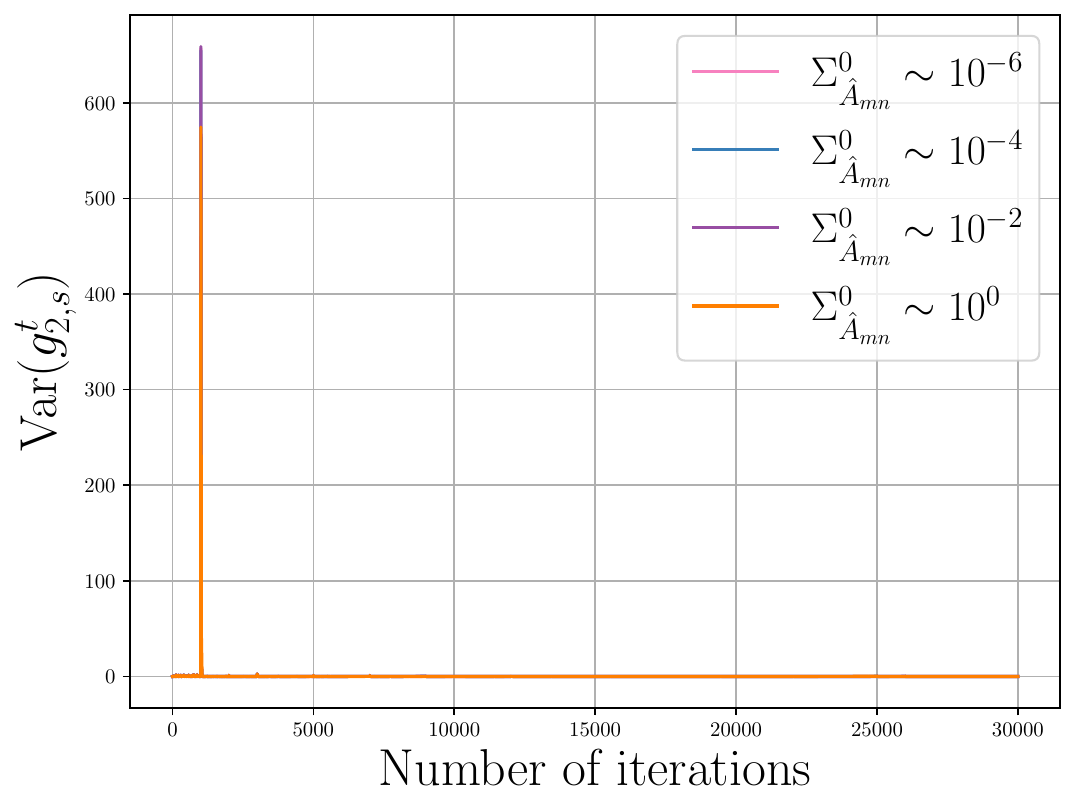}
    \caption{} \label{1b}
  \end{subfigure}%
  \caption{%
    Uncertainty propagation in the communicated terms during FedGC learning for different regimes of $\Sigma_{\hat{A}_{mn}^0}$
    (a) $\operatorname{Var}\bigl(g_{m, s}^{\,t}\bigr)$ vs iterations, 
    (b) $\Sigma_{h_m}^t$ vs iterations, 
  in the synthetic dataset}
  \label{fig:aleatoric_comm_sigma_Amn}
\end{figure}

\begin{figure}[htbp]
  \centering
  \begin{subfigure}[b]{0.4\columnwidth}
    \centering
    \includegraphics[width=\textwidth]{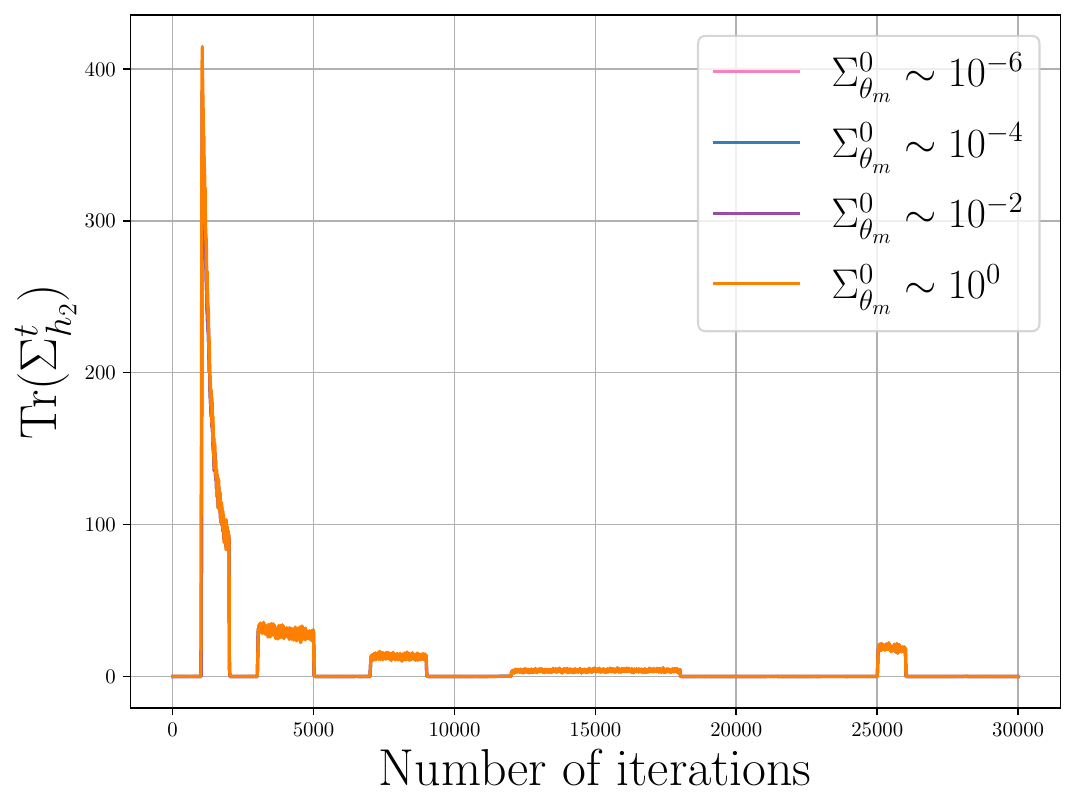}
    \caption{} \label{1a}
  \end{subfigure}%
  \begin{subfigure}[b]{0.4\columnwidth}
    \centering
    \includegraphics[width=\textwidth]{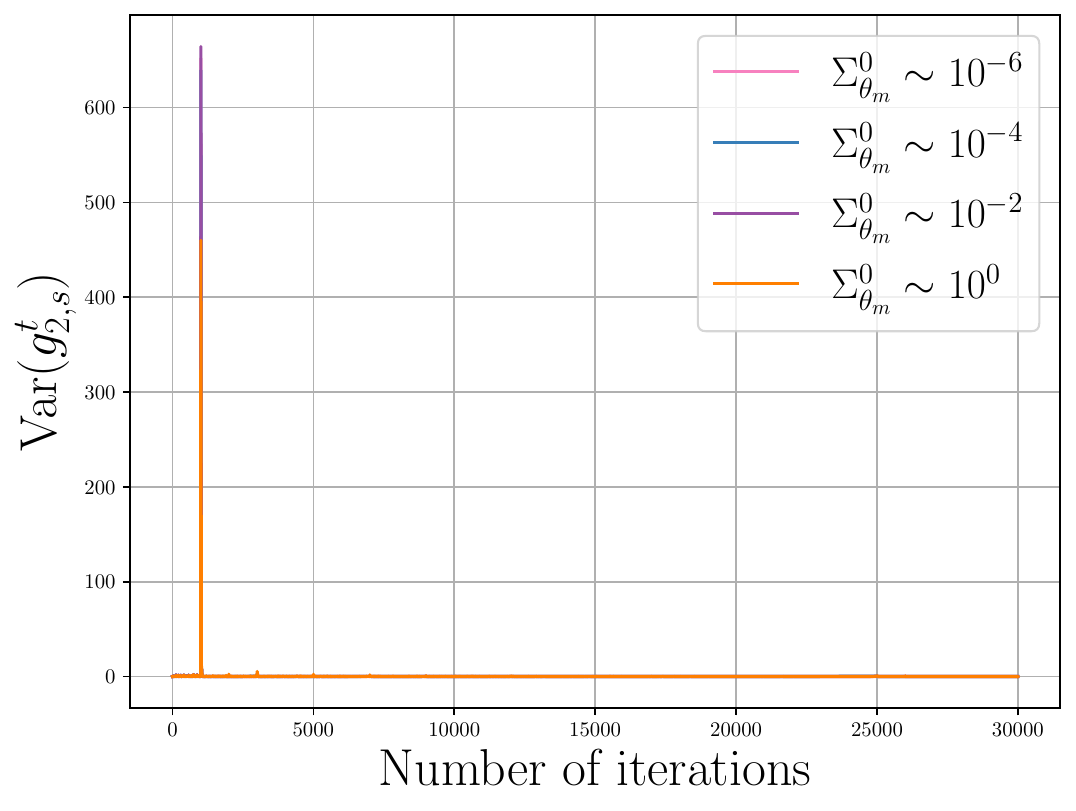}
    \caption{} \label{1b}
  \end{subfigure}%
  \caption{%
    Uncertainty propagation in the communicated terms during FedGC learning for different regimes of $\Sigma_{\theta_m^t}$
    (a) $\operatorname{Var}\bigl(g_{m, s}^{\,t}\bigr)$ vs iterations, 
    (b) $\Sigma_{h_{m}}^t$ vs iterations, 
  in the synthetic dataset}
  \label{fig:aleatoric_comm_sigma_theta}
\end{figure}

\section{Additional Details on Industrial}
\subsection{Preprocessing of Real-world Datasets}
For the real-world datasets, each client contains a multivariate observation stream $\mathbf{Y}^{(c)} \in \mathbb{R}^{T \times p_c}$. We first standardize each client independently using training statistics:$
\widetilde{\mathbf{Y}}^{(c)}
=
\left(
\mathbf{Y}^{(c)}
-
\boldsymbol{\mu}^{(c)}_Y
\right)
\oslash
\boldsymbol{\sigma}^{(c)}_Y.
$
Next, we compute a client-wise singular value decomposition (SVD),
\(
\widetilde{\mathbf{Y}}^{(c)}_{\mathrm{train}}
=
\mathbf{U}^{(c)}
\boldsymbol{\Sigma}^{(c)}
{\mathbf{V}^{(c)}}^\top
\),
and retain the two right singular vectors associated with the largest singular values:
\(
\mathbf{B}^{(c)}
=
[
\mathbf{v}^{(c)}_1,
\mathbf{v}^{(c)}_2
]
\).

Each sample is then projected into a two-dimensional local state space: $
\mathbf{Z}^{(c)}
=
\widetilde{\mathbf{Y}}^{(c)}
\mathbf{B}^{(c)}.
$
To mildly couple the latent components, we apply the fixed transformation $M = \begin{bmatrix}
1 & 1 \\
0 & 1
\end{bmatrix}$ such that, 
$\widehat{\mathbf{Z}}^{(c)}
=
\mathbf{Z}^{(c)} \mathbf{M}.$ Finally, the transformed representation is standardized again using training statistics.

\subsection{Nonlinear Experiments}\label{appendix:NonlinearExp}
\textbf{Neural state-transition functions.} Following Appendix~\ref{appendix:non_linear}, we evaluate our framework on nonlinear dynamical systems by using Extended Kalman Filters (EKFs) at the clients augmented with LSTM networks. The server-side linear transition model is replaced by an LSTM, and local state transitions are likewise modeled with MLPs. Table \ref{tab:nn_params} provides the experimental settings of these neural networks.  

\begin{table}[t]
\centering
\caption{Neural network architectures and hyperparameters used for nonlinear experiments on HAI dataset}
\label{tab:nn_params}
\begin{tabular}{l l l}
\toprule
\textbf{Component} & \textbf{Model} & \textbf{Parameters} \\
\midrule
Server model & LSTM+FC (Fully Connected) &
\begin{tabular}[c]{@{}l@{}}
Hidden dim = 128 \\
Num layers = 1 \\
Activation = linear (FC) \\
Optimizer = Adam, lr = $1\times 10^{-2}$
\end{tabular} \\

\midrule
Client augmentation model & LSTM + FC &
\begin{tabular}[c]{@{}l@{}}
Hidden dim = 32 \\
Num layers = 1 \\
Activation = linear (FC) \\
Optimizer = Adam, lr = $1\times 10^{-4}$
\end{tabular} \\

\midrule
Client transition model & MLP (for EKF $f_m$) &
\begin{tabular}[c]{@{}l@{}}
Hidden layers = [64, 64] \\
Activation = SiLU \\
Optimizer = Adam, lr = $1\times 10^{-3}$
\end{tabular} \\
\bottomrule
\end{tabular}
\end{table}

\textbf{Generating states.} To obtain low-dimensional latent states from the raw measurements, we apply singular value decomposition (SVD) to the stacked data matrix and retain the leading right singular vectors as the reduced state representation. Specifically, given the data matrix $Y$, we compute $Y = U \Sigma V^\top$ and use the top-$p$ columns of $V$ to define the latent states. In this setting, the measurement function corresponds to the projection from the latent state back to the observation space induced by the retained singular vectors.

\textbf{Results.} Unlike linear FedGC, the server parameters in the nonlinear setting do not directly encode Granger causality. Instead, we analyze the Jacobian $J$ of the server model, which acts as a time-dependent state-transition matrix and provides a notion of Granger causality. As in the linear case, we vary $\Sigma_{y_m}^t$ to study aleatoric uncertainty and perturb the hyperparameters of the nonlinear models discussed above to assess epistemic effects. 

Fig.~\ref{fig:trace_cov_offdiag_J_hai} shows an overall decreasing trend in the Jacobian uncertainty, with larger $\Sigma_{y_m}^t$ leading to slower convergence. Fig.~\ref{fig:trace_cov_offdiag_J_hai_epistemic} shows the evolution of uncertainty in the server Jacobian under different levels of noise in the server LSTM weights. Except at very high noise, all curves largely overlap, indicating minimal impact of epistemic initialization on Jacobian uncertainty. These results are consistent with the theoretical derivations in Section~\ref{sec:impact}, particularly the nonlinear analogue of Theorem \ref{thm:sigmaA-closed}.

\begin{figure}[htbp]
  \centering
  \begin{subfigure}[b]{0.33\columnwidth}
    \centering
    \includegraphics[width=\textwidth]{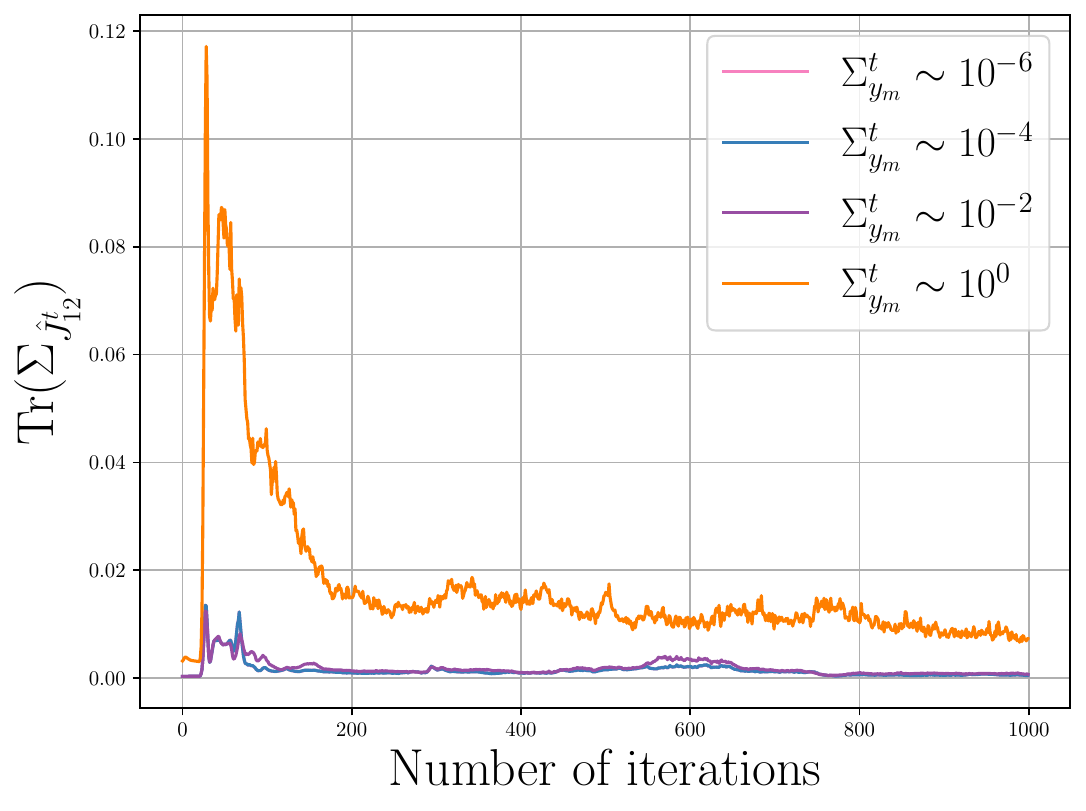}
    \caption{$\mathrm{Tr}\bigl(\Sigma_{\hat{J}_{12}^t}\bigr)$}\label{fig:traceJ12}
  \end{subfigure}%
  \begin{subfigure}[b]{0.33\columnwidth}
    \centering
    \includegraphics[width=\textwidth]{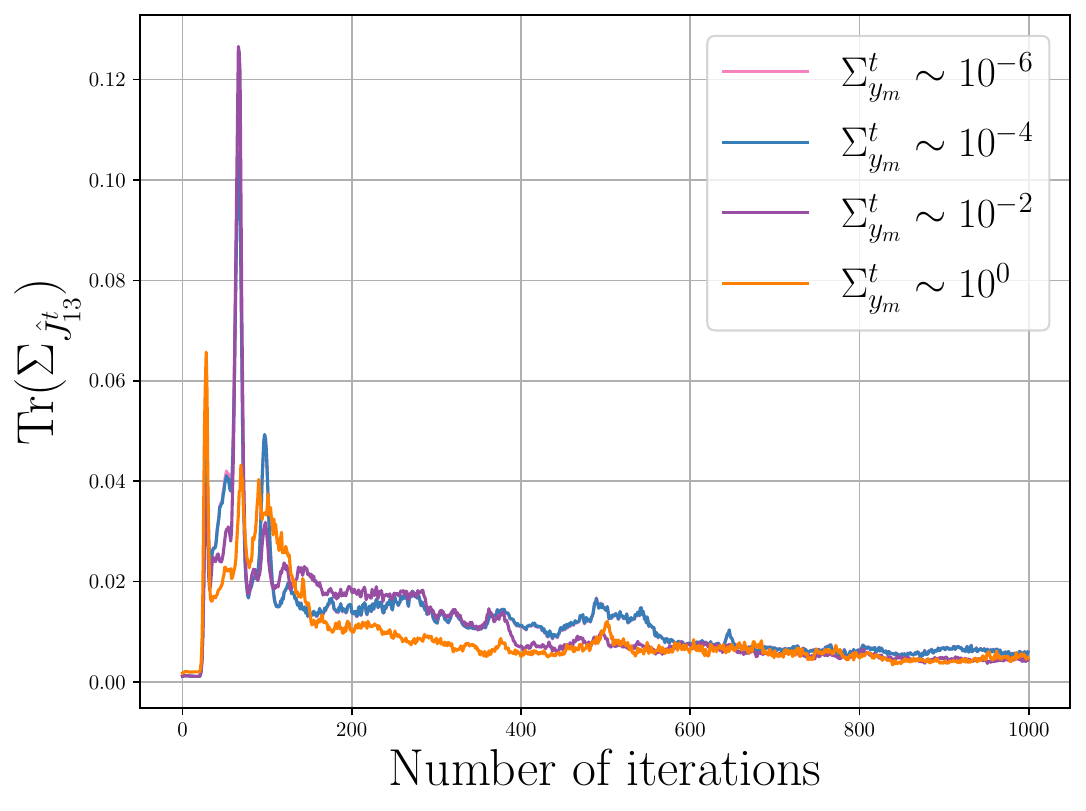}
    \caption{$\mathrm{Tr}\bigl(\Sigma_{\hat{J}_{13}^t}\bigr)$}\label{fig:traceJ13}
  \end{subfigure}%
  \begin{subfigure}[b]{0.33\columnwidth}
    \centering
    \includegraphics[width=\textwidth]{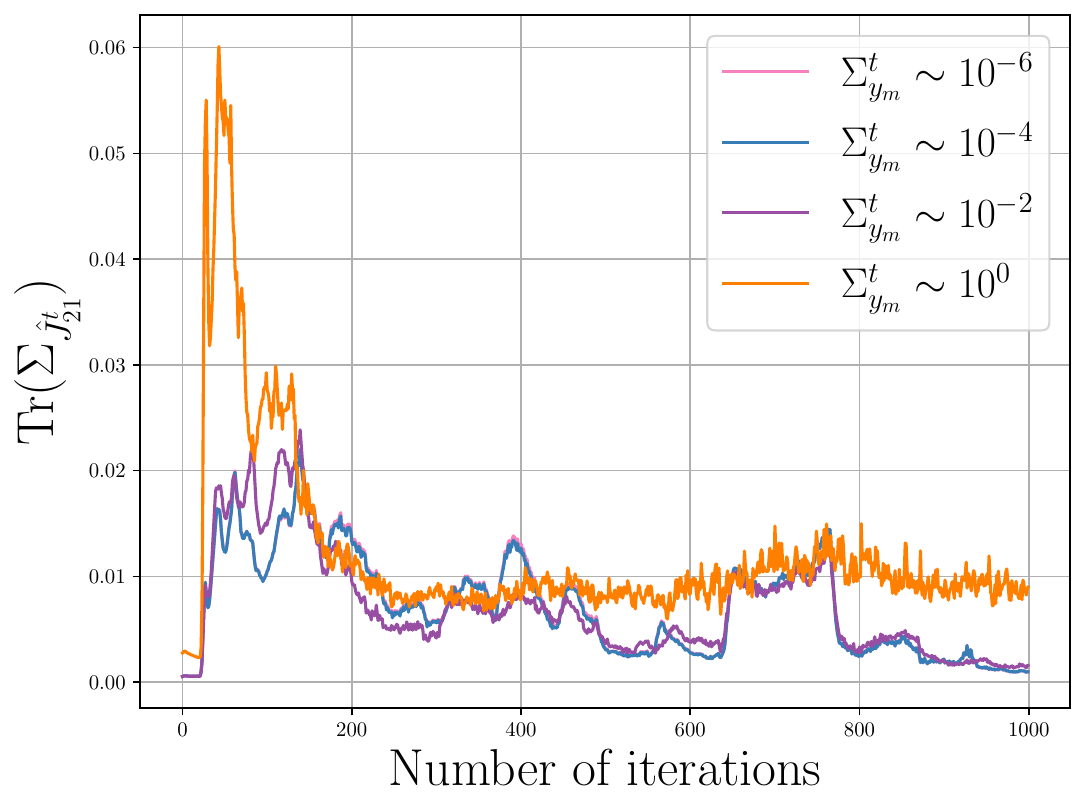}
    \caption{$\mathrm{Tr}\bigl(\Sigma_{\hat{J}_{21}^t}\bigr)$}\label{fig:traceJ21}
  \end{subfigure}

  \vskip\baselineskip

  \begin{subfigure}[b]{0.33\columnwidth}
    \centering
    \includegraphics[width=\textwidth]{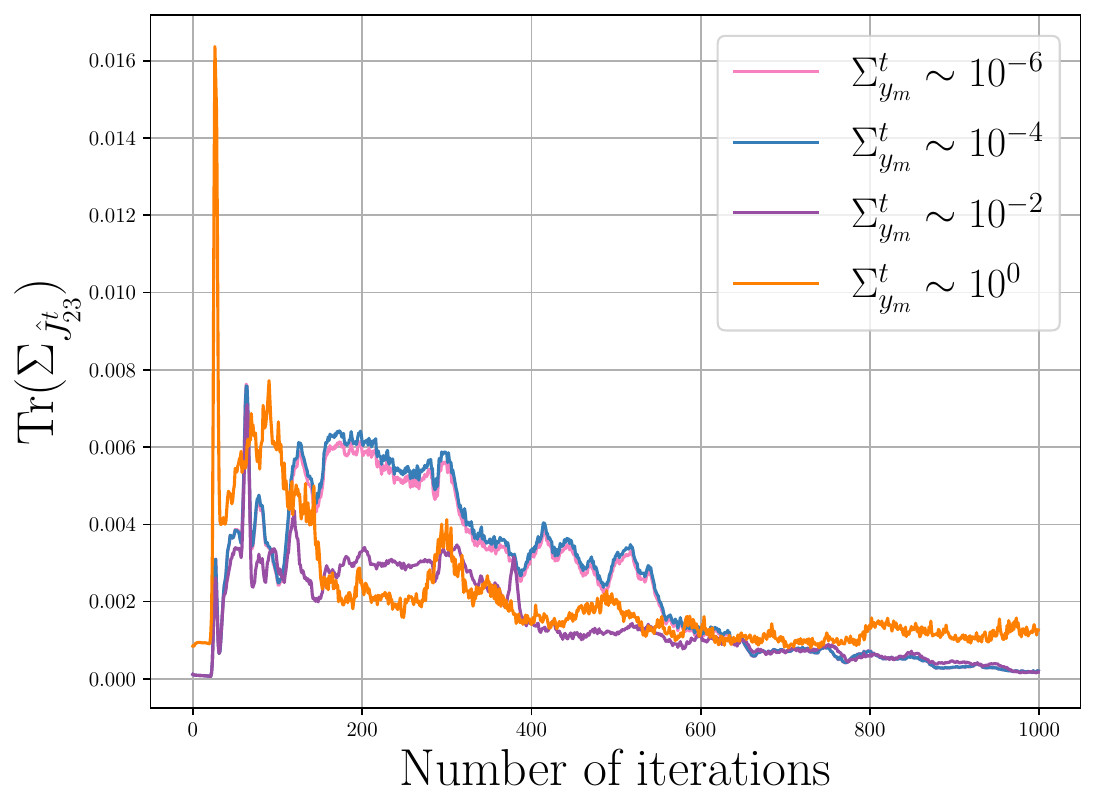}
    \caption{$\mathrm{Tr}\bigl(\Sigma_{\hat{J}_{23}^t}\bigr)$}\label{fig:traceJ23}
  \end{subfigure}%
  \begin{subfigure}[b]{0.33\columnwidth}
    \centering
    \includegraphics[width=\textwidth]{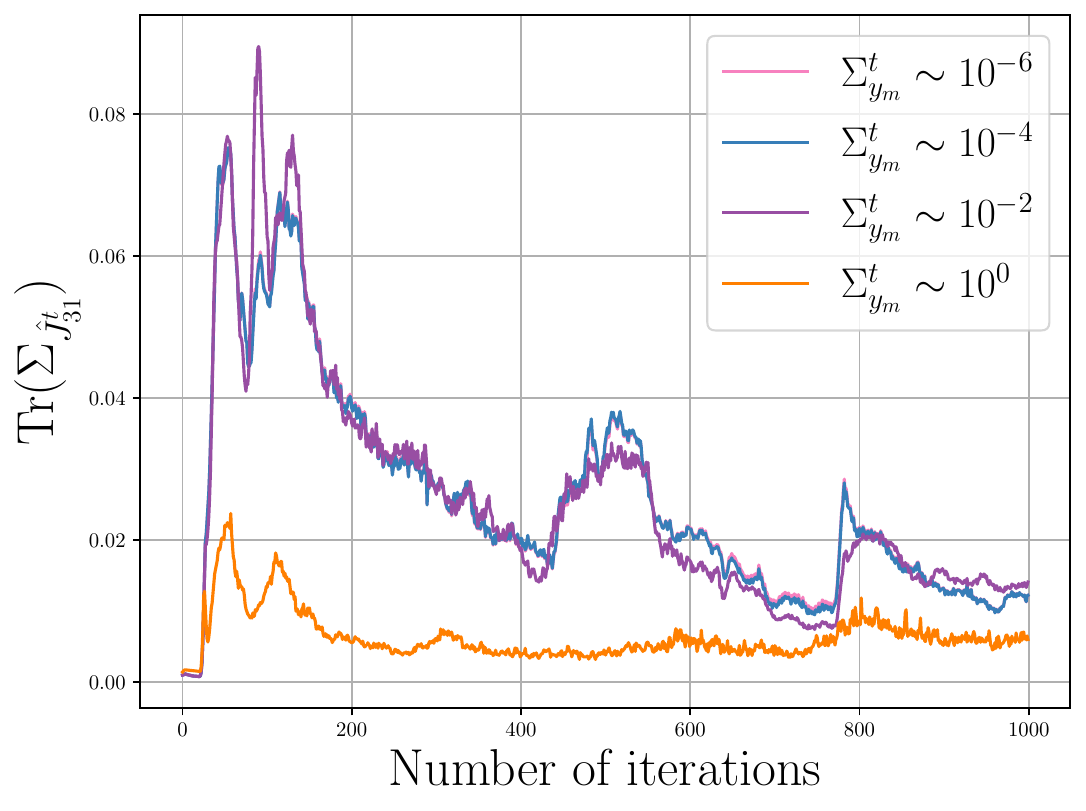}
    \caption{$\mathrm{Tr}\bigl(\Sigma_{\hat{J}_{31}^t}\bigr)$}\label{fig:traceJ31}
  \end{subfigure}%
  \begin{subfigure}[b]{0.33\columnwidth}
    \centering
    \includegraphics[width=\textwidth]{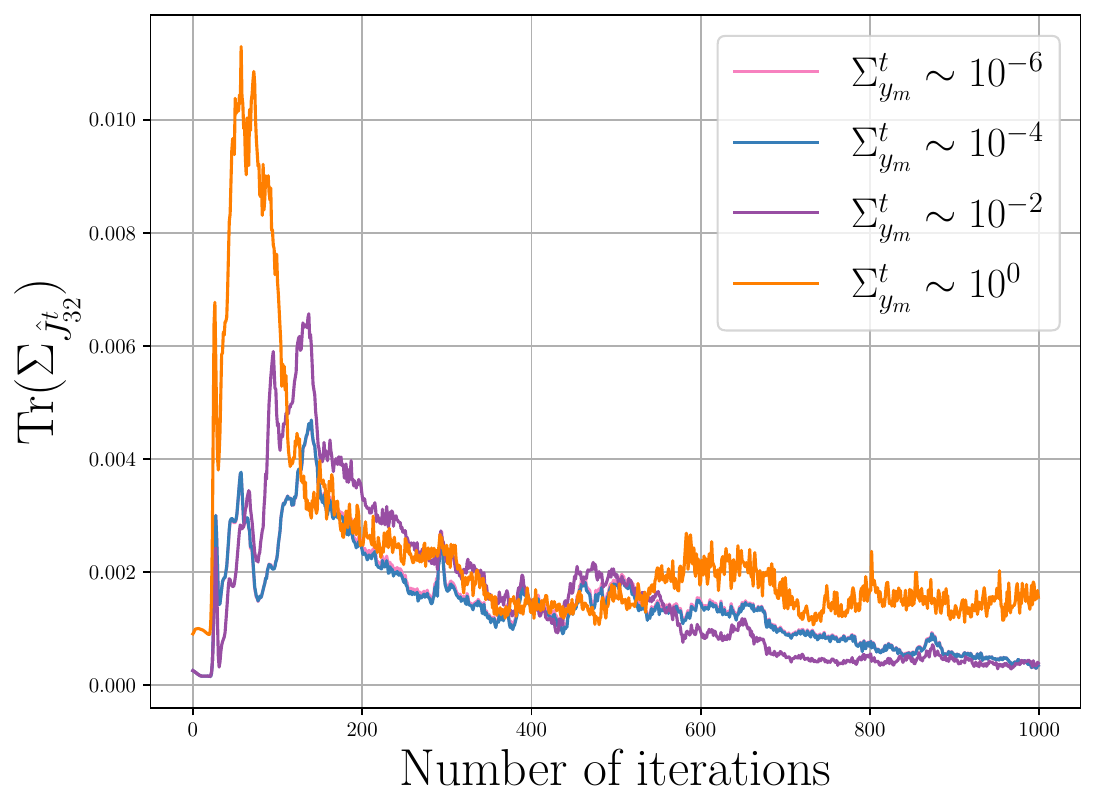}
    \caption{$\mathrm{Tr}\bigl(\Sigma_{\hat{J}_{32}^t}\bigr)$}\label{fig:traceJ32}
  \end{subfigure}

  \caption{Nonlinear Exper. on HAI: Trace of the covariance for the off-diagonal blocks of the Jacobian matrix $J$  for different $\Sigma_{y_m}^t$}
  \label{fig:trace_cov_offdiag_J_hai}
\end{figure}

\begin{figure}[htbp]
  \centering
  \begin{subfigure}[b]{0.33\columnwidth}
    \centering
    \includegraphics[width=\textwidth]{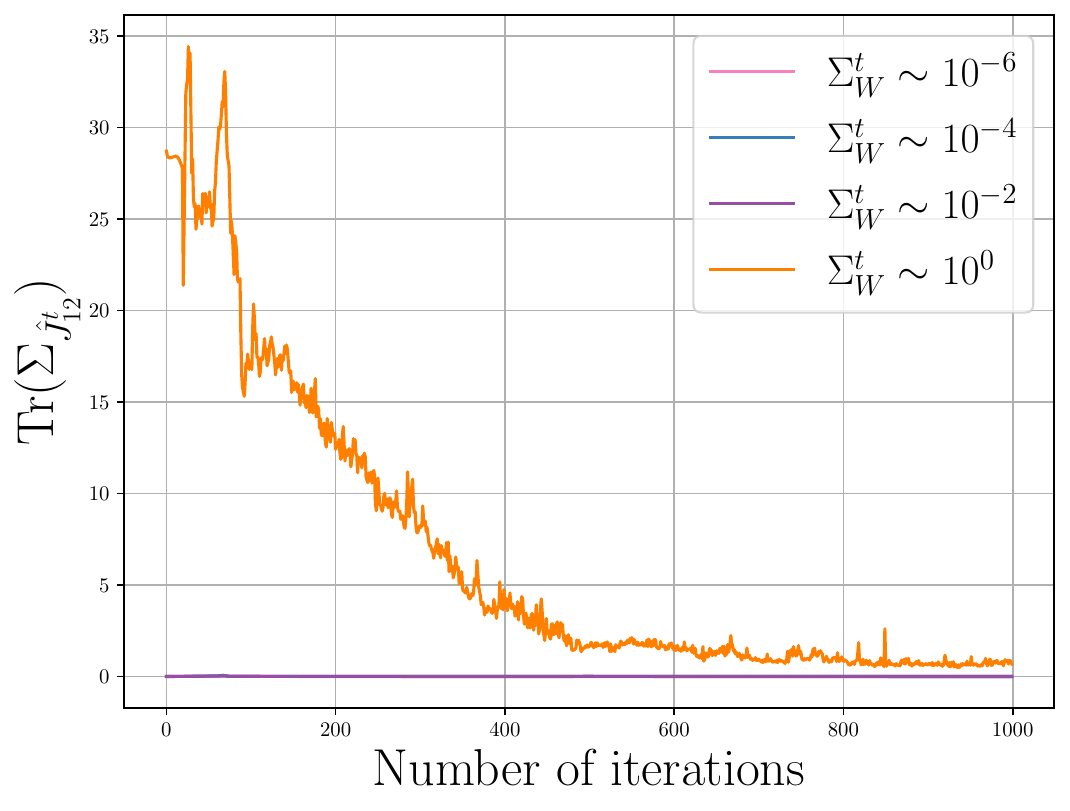}
    \caption{$\mathrm{Tr}\bigl(\Sigma_{\hat{J}_{12}^t}\bigr)$}\label{fig:traceJ12}
  \end{subfigure}%
  \begin{subfigure}[b]{0.33\columnwidth}
    \centering
    \includegraphics[width=\textwidth]{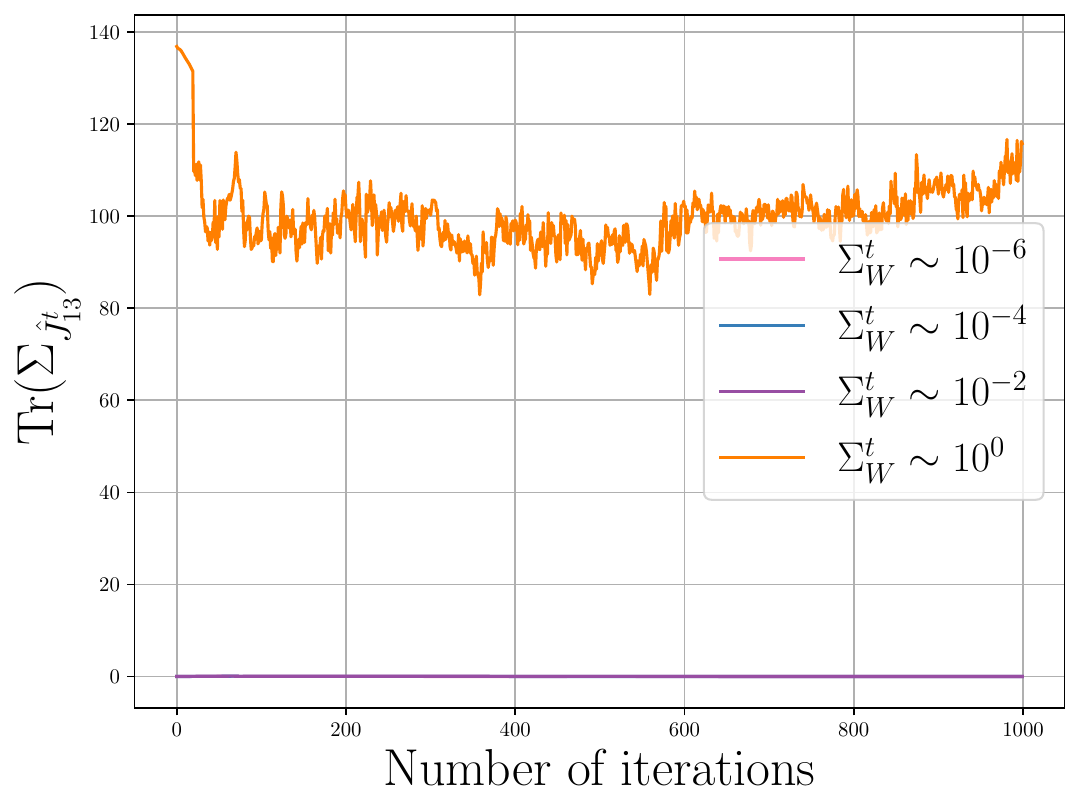}
    \caption{$\mathrm{Tr}\bigl(\Sigma_{\hat{J}_{13}^t}\bigr)$}\label{fig:traceJ13}
  \end{subfigure}%
  \begin{subfigure}[b]{0.33\columnwidth}
    \centering
    \includegraphics[width=\textwidth]{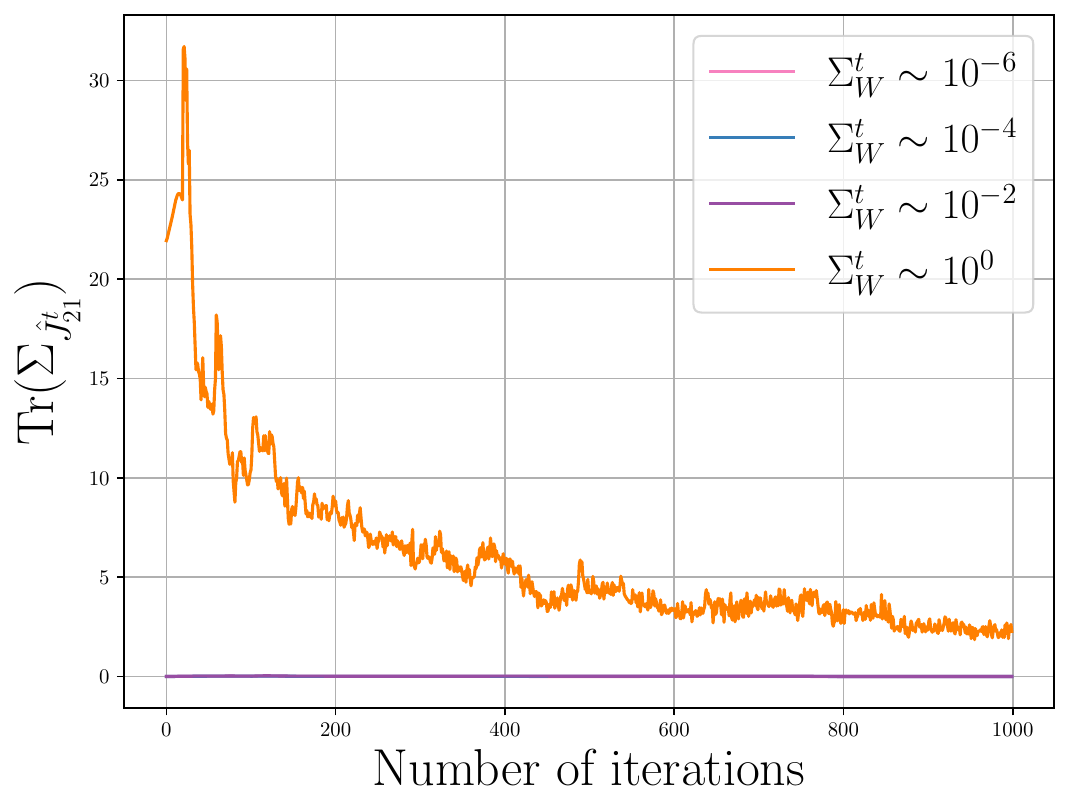}
    \caption{$\mathrm{Tr}\bigl(\Sigma_{\hat{J}_{21}^t}\bigr)$}\label{fig:traceJ21}
  \end{subfigure}

  \vskip\baselineskip

  \begin{subfigure}[b]{0.33\columnwidth}
    \centering
    \includegraphics[width=\textwidth]{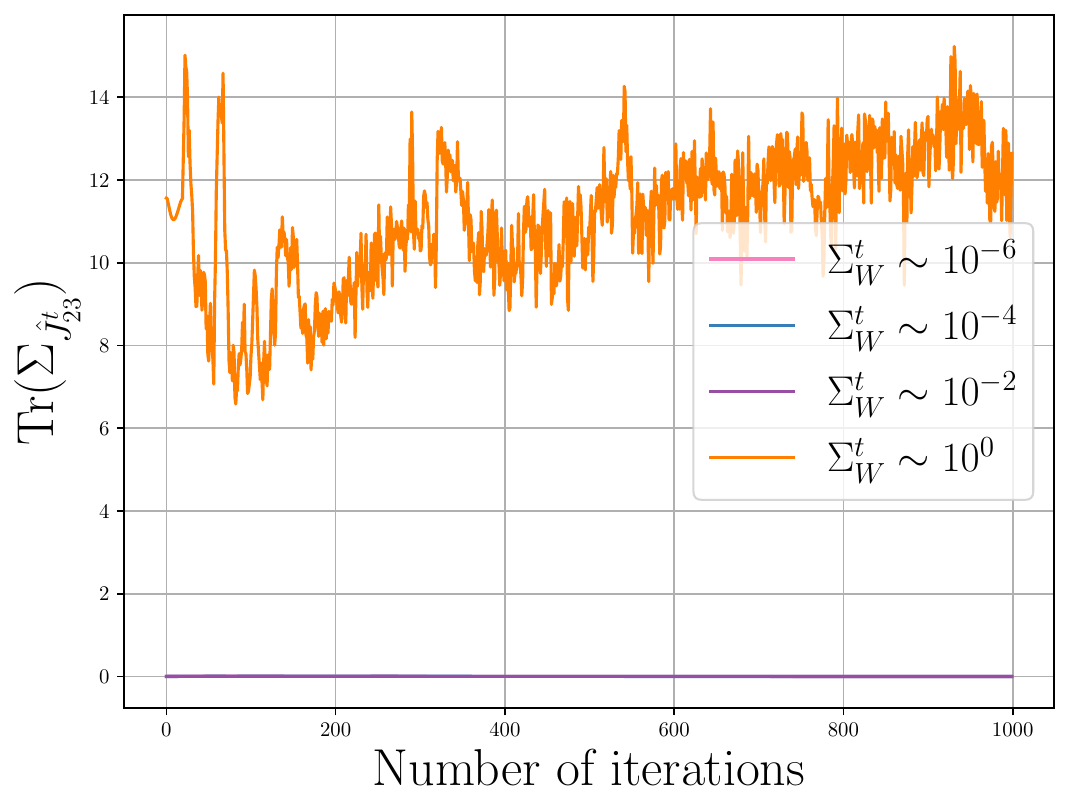}
    \caption{$\mathrm{Tr}\bigl(\Sigma_{\hat{J}_{23}^t}\bigr)$}\label{fig:traceJ23}
  \end{subfigure}%
  \begin{subfigure}[b]{0.33\columnwidth}
    \centering
    \includegraphics[width=\textwidth]{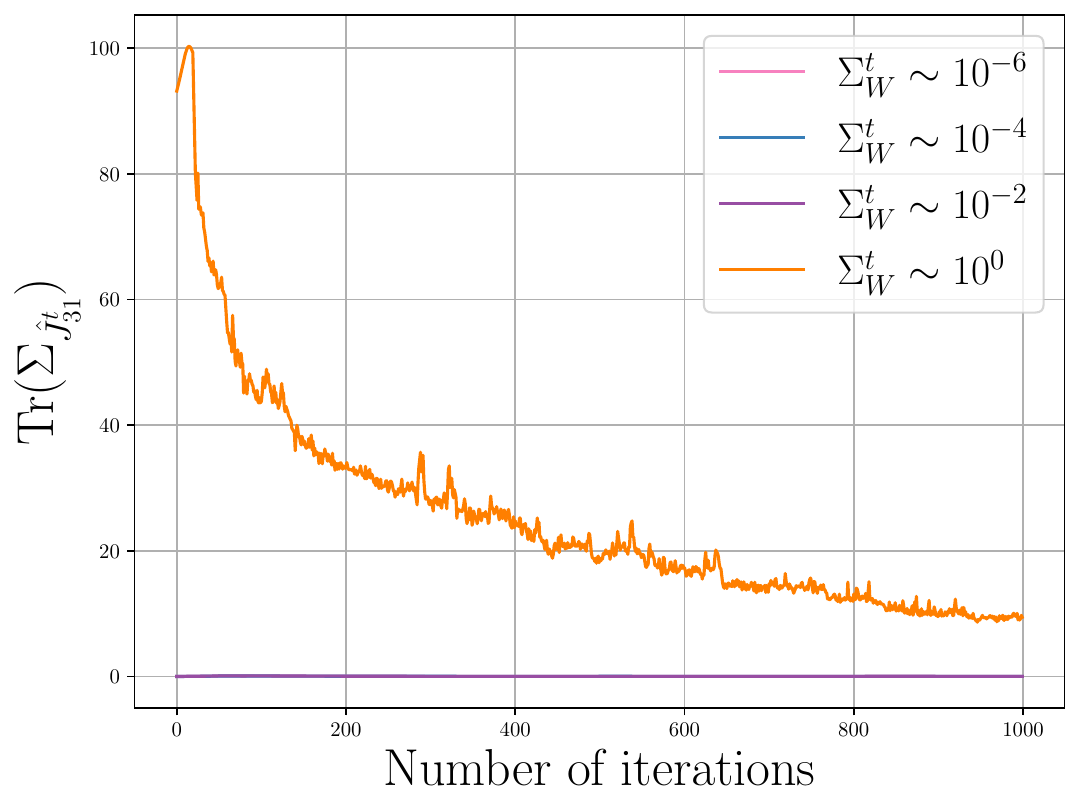}
    \caption{$\mathrm{Tr}\bigl(\Sigma_{\hat{J}_{31}^t}\bigr)$}\label{fig:traceJ31}
  \end{subfigure}%
  \begin{subfigure}[b]{0.33\columnwidth}
    \centering
    \includegraphics[width=\textwidth]{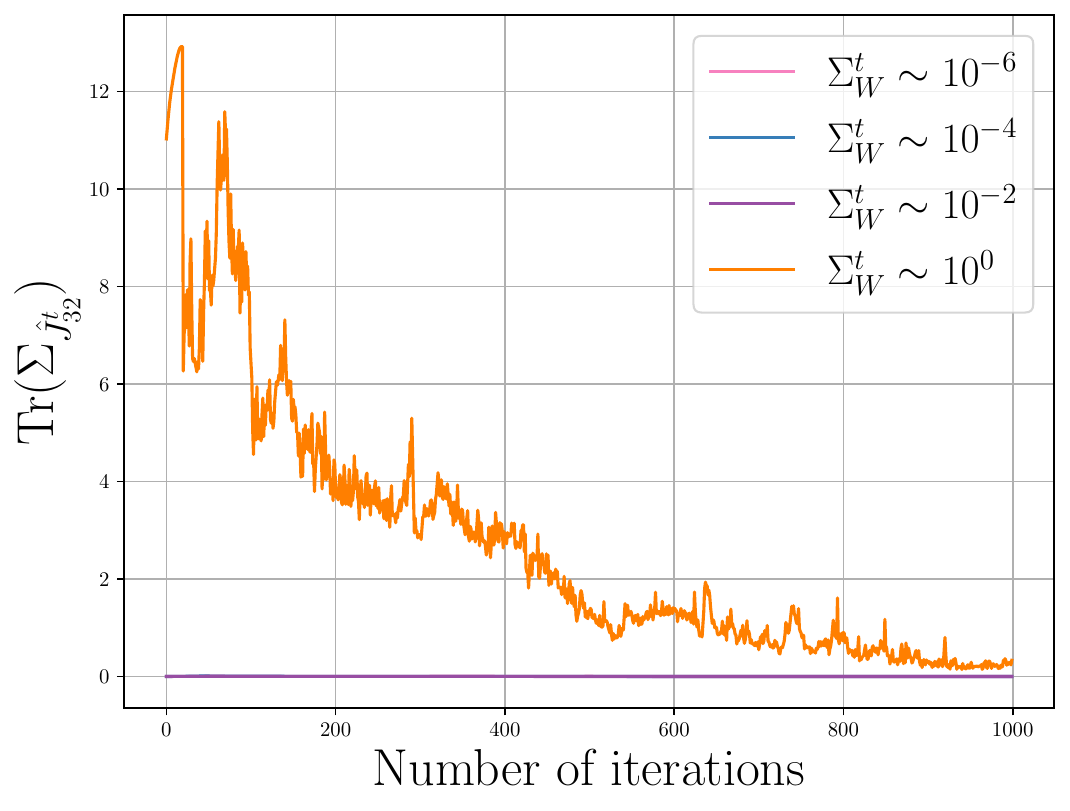}
    \caption{$\mathrm{Tr}\bigl(\Sigma_{\hat{J}_{32}^t}\bigr)$}\label{fig:traceJ32}
  \end{subfigure}

  \caption{Nonlinear Exper. on HAI: Trace of the covariance for the off-diagonal blocks of the Jacobian matrix $J$  for different $\Sigma_{W}^t$}
  \label{fig:trace_cov_offdiag_J_hai_epistemic}
\end{figure}

\begin{figure}[htbp]
  \centering
  \begin{subfigure}[b]{0.33\columnwidth}
    \centering
    \includegraphics[width=\textwidth]{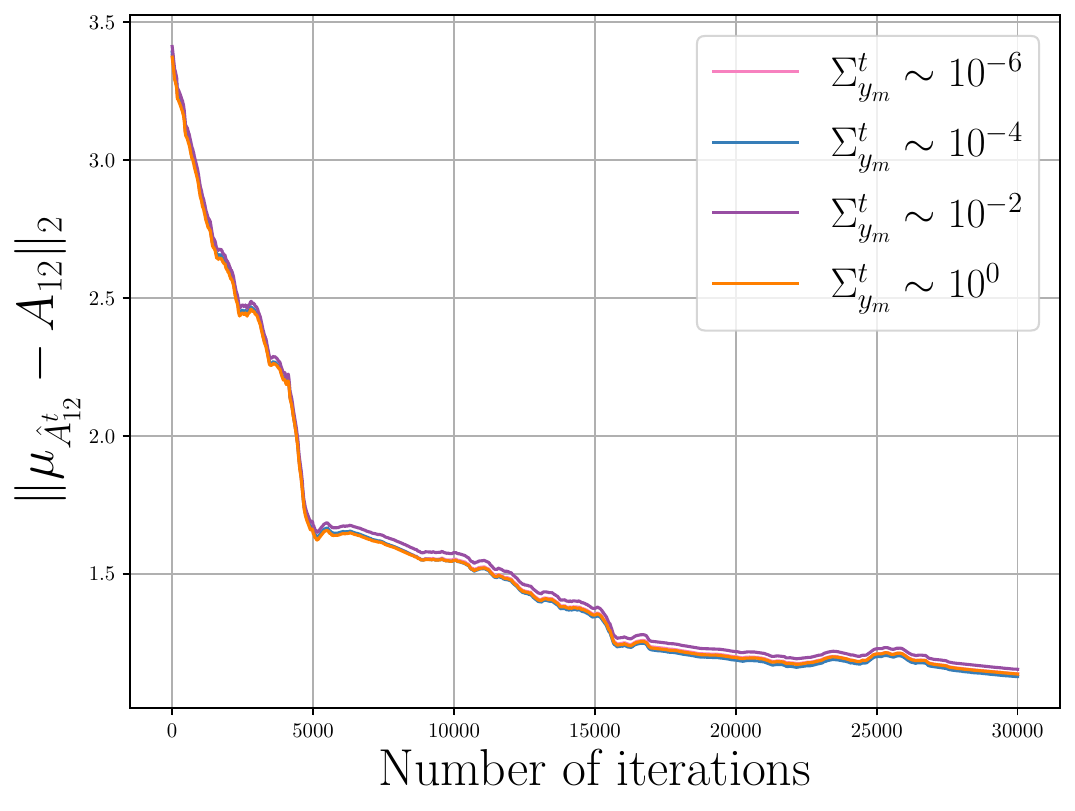}
    \caption{$\|\mu_{\hat{A}_{12}^t}-A_{12}\|_2$}\label{fig:L2Error12}
  \end{subfigure}%
  \begin{subfigure}[b]{0.33\columnwidth}
    \centering
    \includegraphics[width=\textwidth]{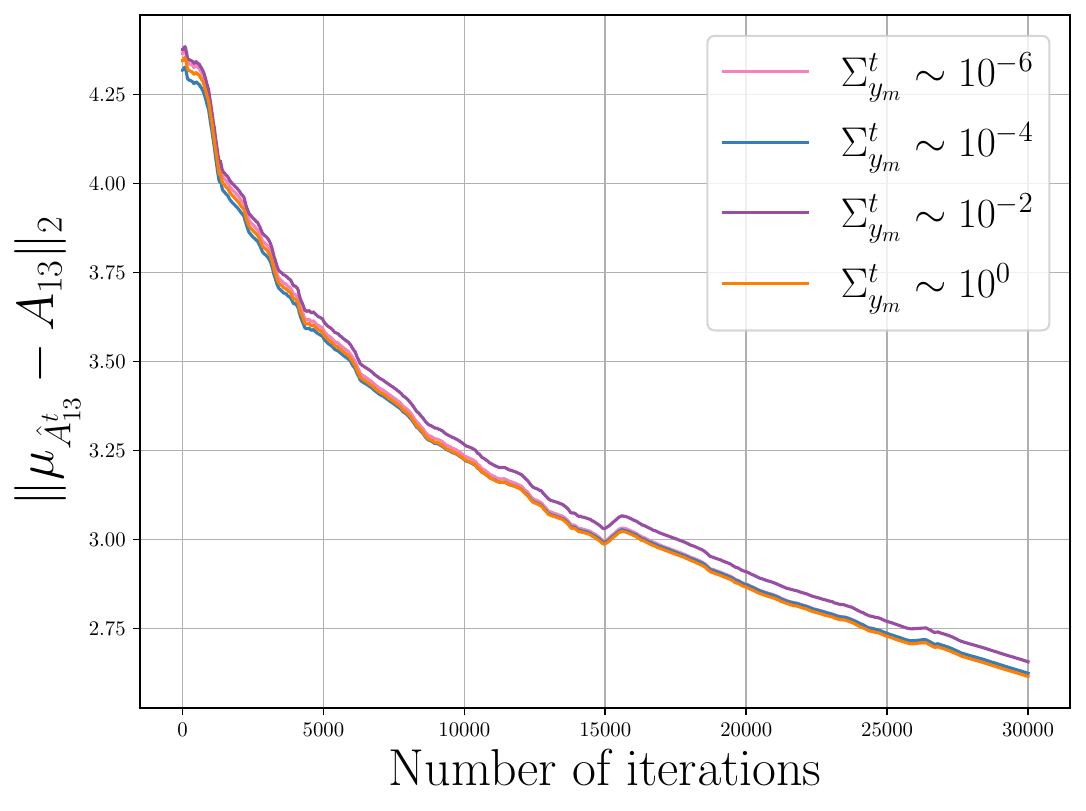}
    \caption{$\|\mu_{\hat{A}_{13}^t}-A_{13}\|_2$}\label{fig:L2Error13}
  \end{subfigure}%
  \begin{subfigure}[b]{0.33\columnwidth}
    \centering
    \includegraphics[width=\textwidth]{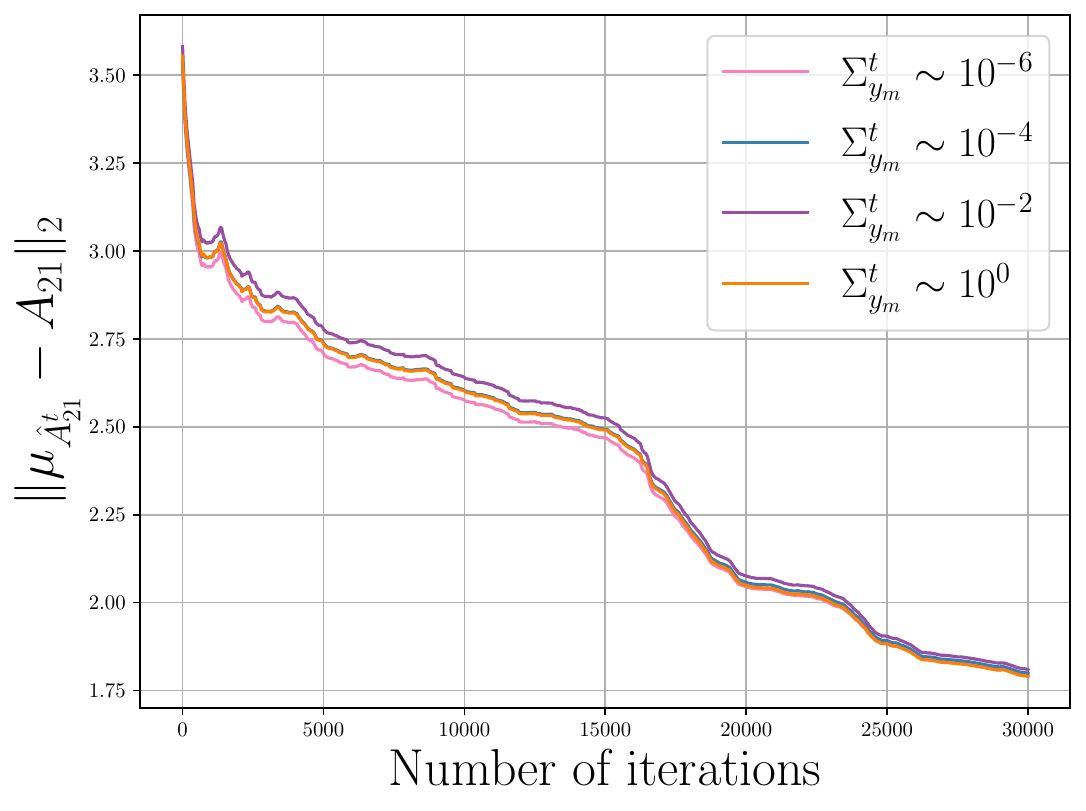}
    \caption{$\|\mu_{\hat{A}_{21}^t}-A_{21}\|_2$}\label{fig:L2Error21}
  \end{subfigure}

  \vskip\baselineskip

  \begin{subfigure}[b]{0.33\columnwidth}
    \centering
    \includegraphics[width=\textwidth]{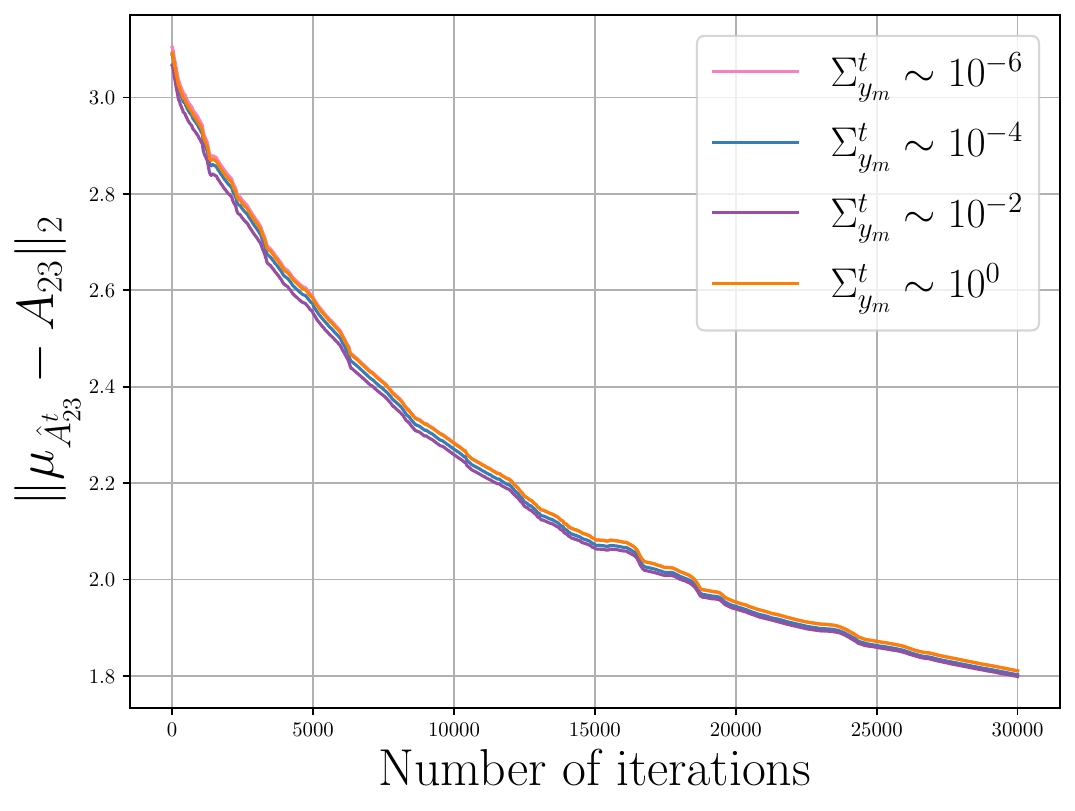}
    \caption{$\|\mu_{\hat{A}_{23}^t}-A_{23}\|_2$}\label{fig:L2Error23}
  \end{subfigure}%
  \begin{subfigure}[b]{0.33\columnwidth}
    \centering
    \includegraphics[width=\textwidth]{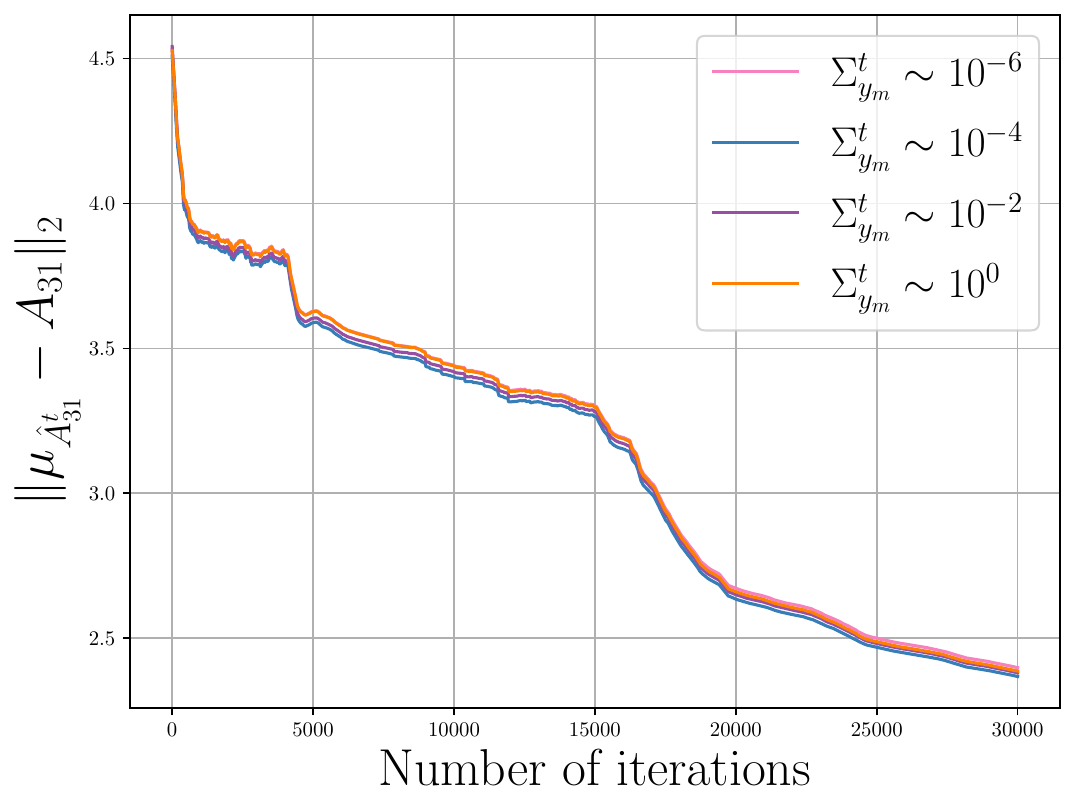}
    \caption{$\|\mu_{\hat{A}_{31}^t}-A_{31}\|_2$}\label{fig:L2Error31}
  \end{subfigure}%
  \begin{subfigure}[b]{0.33\columnwidth}
    \centering
    \includegraphics[width=\textwidth]{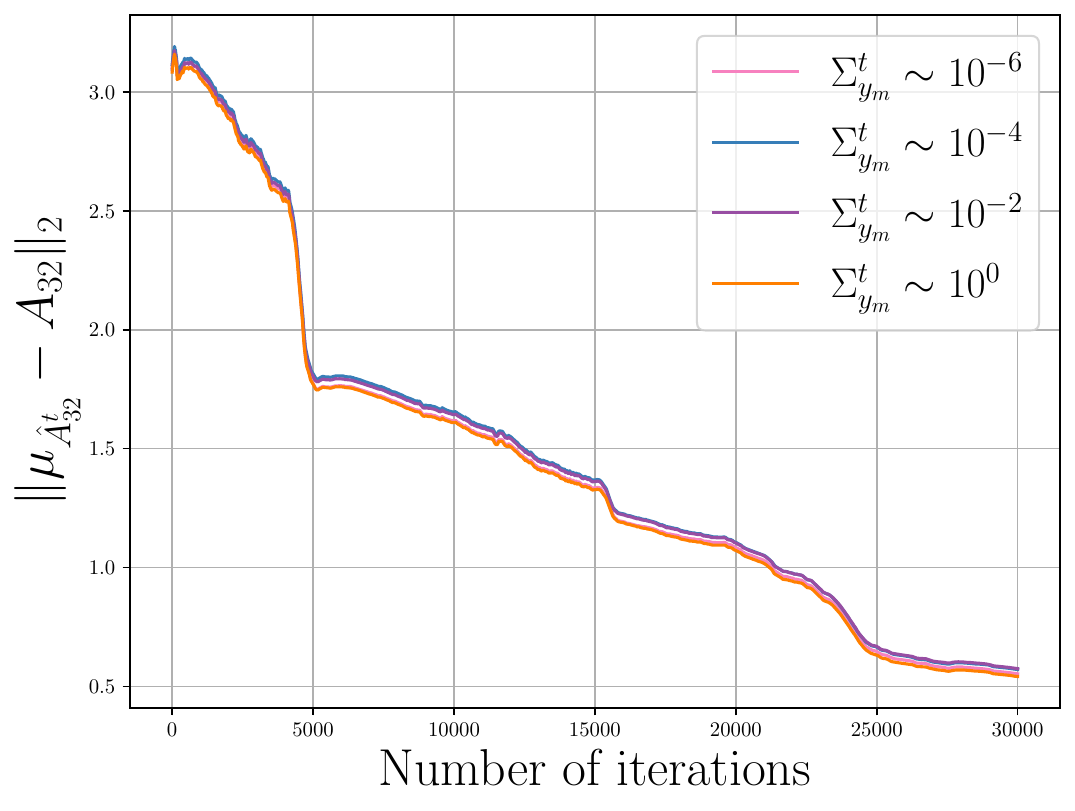}
    \caption{$\|\mu_{\hat{A}_{32}^t}-A_{32}\|_2$}\label{fig:L2Error32}
  \end{subfigure}

  \caption{Average $L_2$ norm error of each off-diagonal block of the matrix $A$ for different regimes of $\Sigma_{y_m}^t$ for HAI dataset}
  \label{fig:L2_error_offdiag_A_realworld}
\end{figure}

\begin{figure}[htbp]
  \centering
  \begin{subfigure}[b]{0.33\columnwidth}
    \centering
    \includegraphics[width=\textwidth]{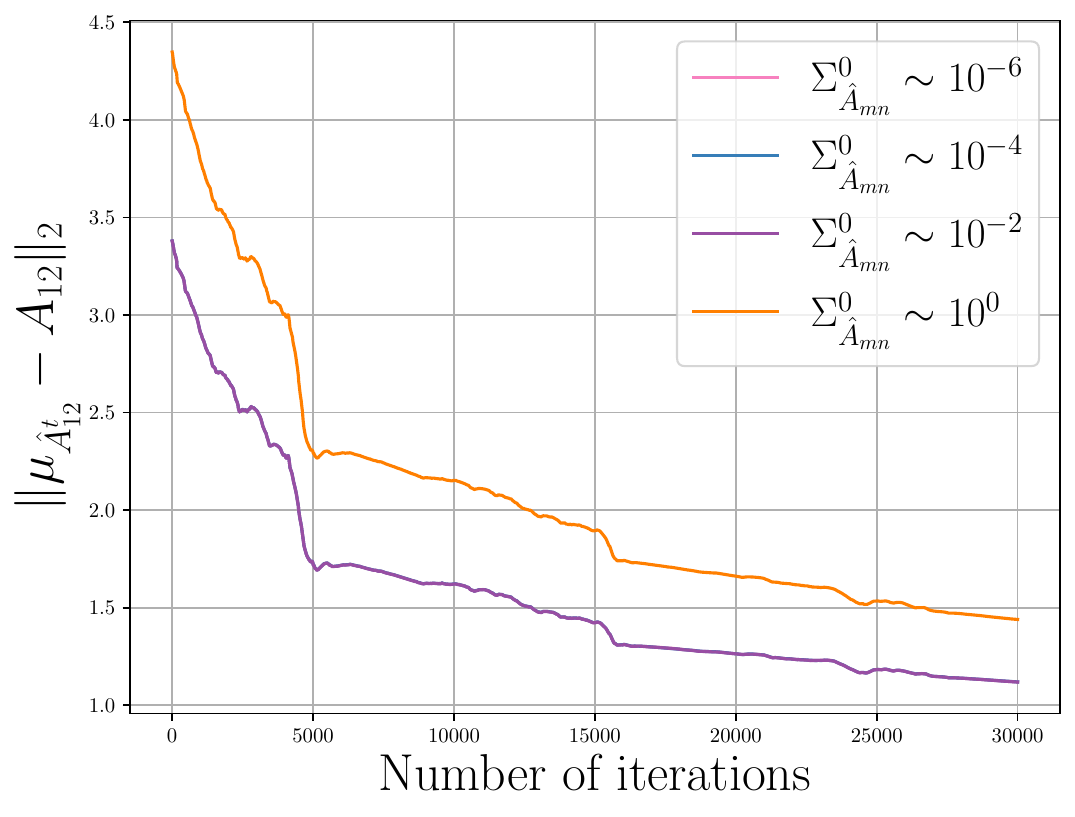}
    \caption{$\|\mu_{\hat{A}_{12}^t}-A_{12}\|_2$}\label{fig:L2Error12}
  \end{subfigure}%
  \begin{subfigure}[b]{0.33\columnwidth}
    \centering
    \includegraphics[width=\textwidth]{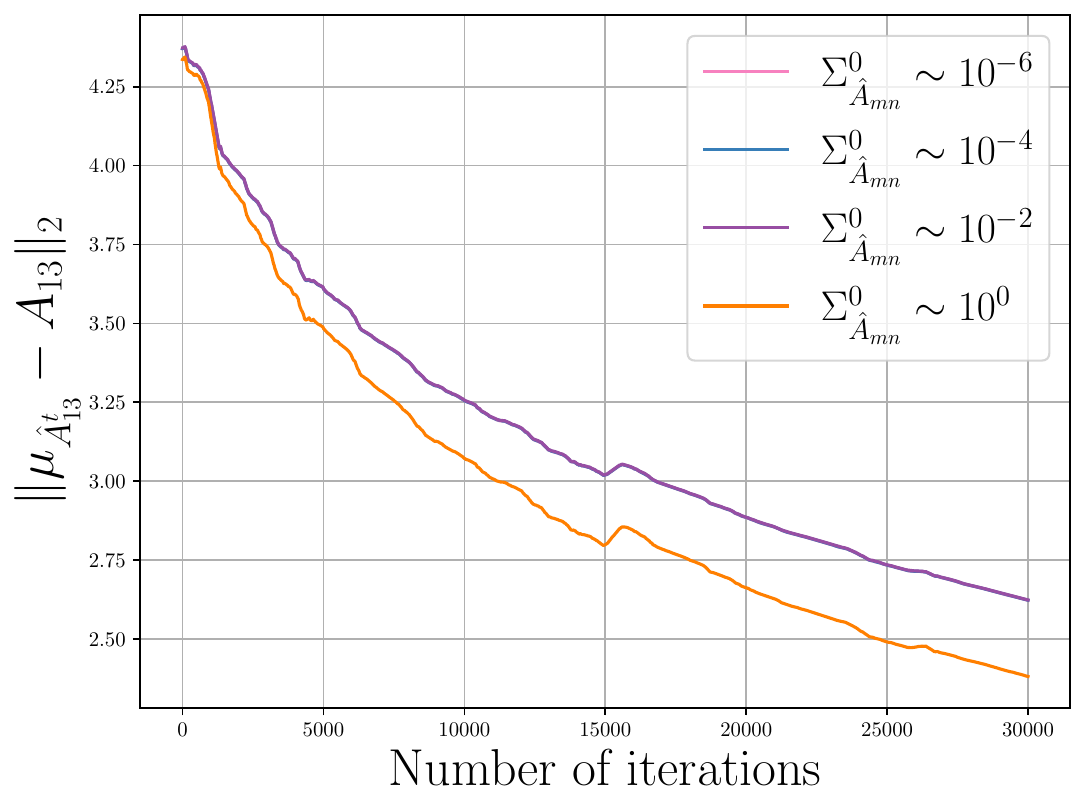}
    \caption{$\|\mu_{\hat{A}_{13}^t}-A_{13}\|_2$}\label{fig:L2Error13}
  \end{subfigure}%
  \begin{subfigure}[b]{0.33\columnwidth}
    \centering
    \includegraphics[width=\textwidth]{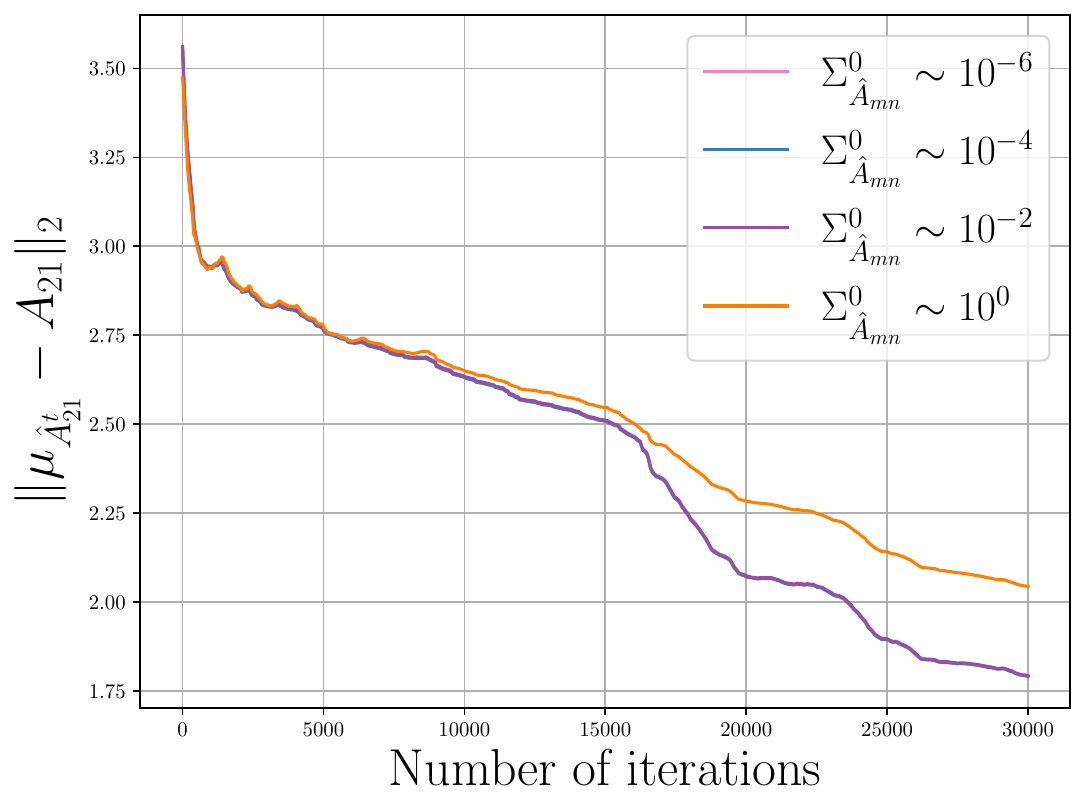}
    \caption{$\|\mu_{\hat{A}_{21}^t}-A_{21}\|_2$}\label{fig:L2Error21}
  \end{subfigure}

  \vskip\baselineskip

  \begin{subfigure}[b]{0.33\columnwidth}
    \centering
    \includegraphics[width=\textwidth]{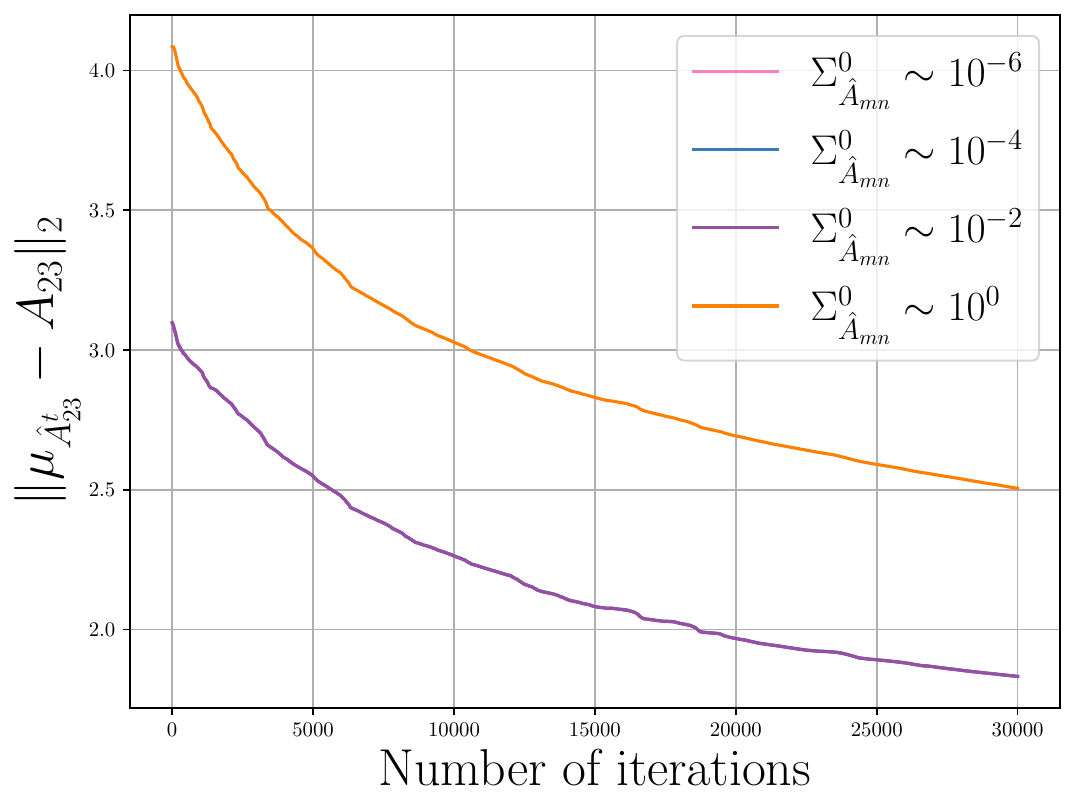}
    \caption{$\|\mu_{\hat{A}_{23}^t}-A_{23}\|_2$}\label{fig:L2Error23}
  \end{subfigure}%
  \begin{subfigure}[b]{0.33\columnwidth}
    \centering
    \includegraphics[width=\textwidth]{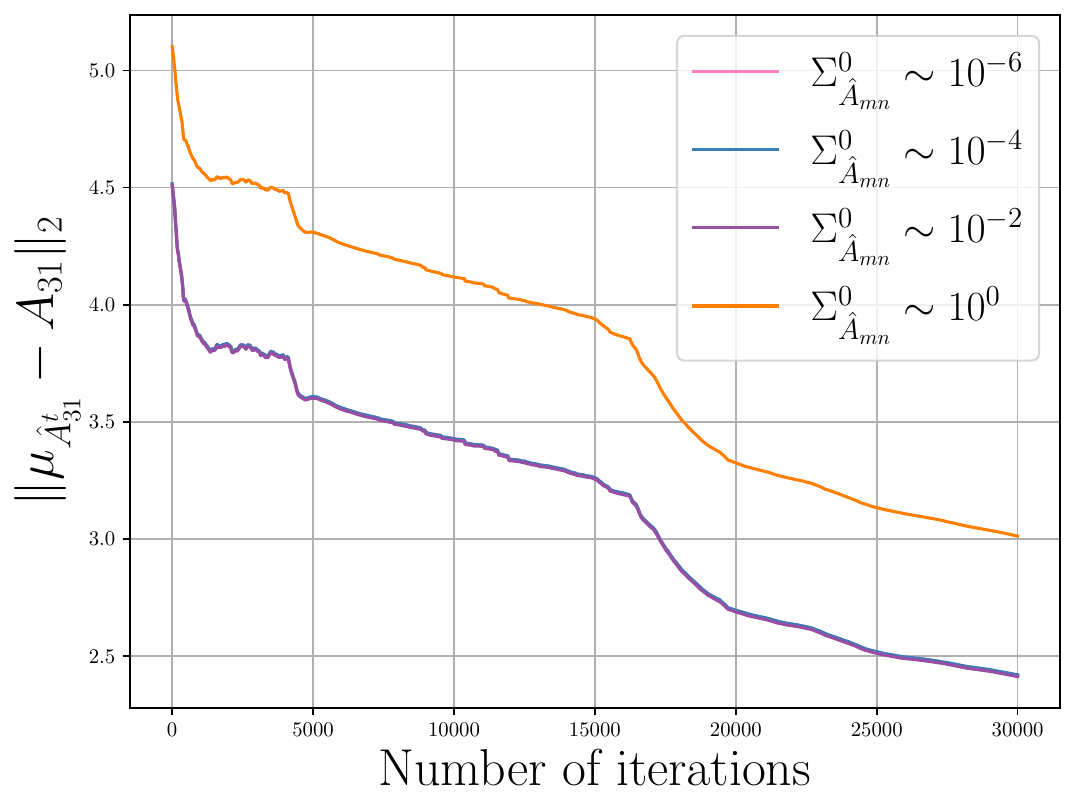}
    \caption{$\|\mu_{\hat{A}_{31}^t}-A_{31}\|_2$}\label{fig:L2Error31}
  \end{subfigure}%
  \begin{subfigure}[b]{0.33\columnwidth}
    \centering
    \includegraphics[width=\textwidth]{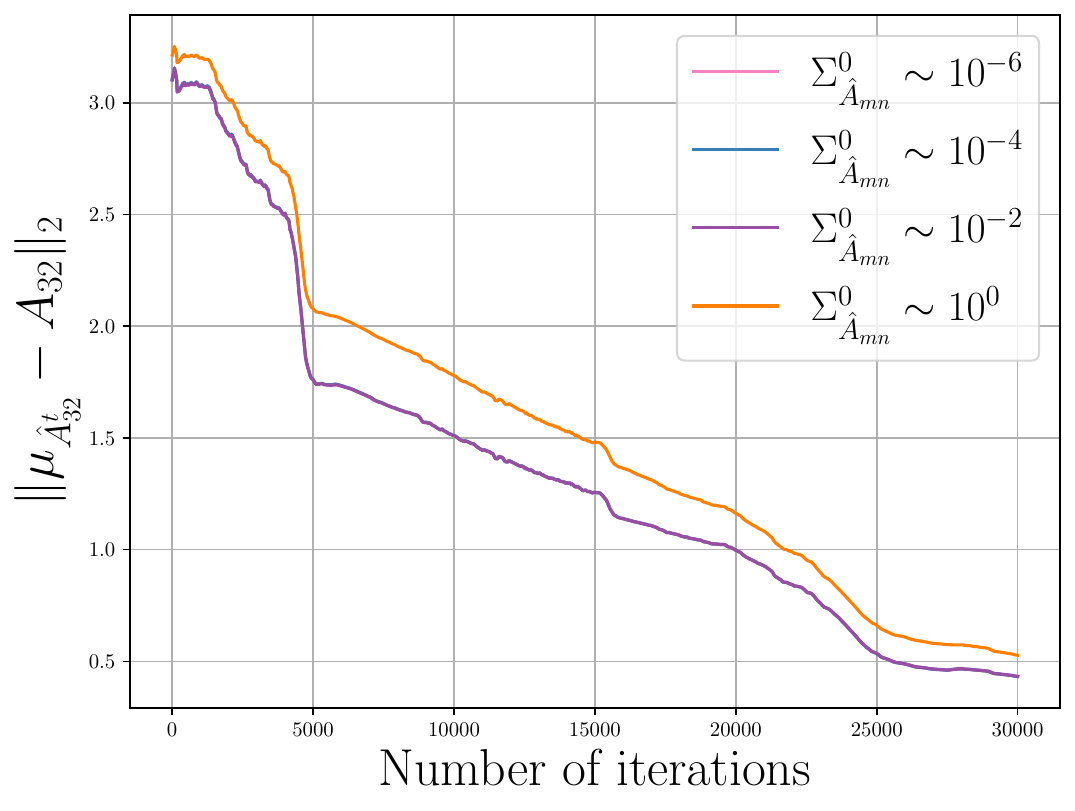}
    \caption{$\|\mu_{\hat{A}_{32}^t}-A_{32}\|_2$}\label{fig:L2Error32}
  \end{subfigure}

  \caption{Average $L_2$ norm error of each off-diagonal block of the matrix $A$ for different regimes of $\Sigma_{\hat{A}_{mn}}^0$ for HAI dataset}
  \label{fig:L2_error_offdiag_A_hai}
\end{figure}
\label{sec:realworld_appendix}
\section{Experimental Compute Resources}
All of the experiments were run locally on a 2022 MacBook Air with an Apple M2 chip, 16 GB unified memory, and 512 GB SSD storage. No cloud provider, computing cluster, or external GPU was used. The Apple M2 MacBook Air has an 8-core CPU with 4 performance cores and 4 efficiency cores, a 16-core Neural Engine, and 100 GB/s memory bandwidth; the 512 GB configuration is associated with the M2 configuration that can include the 10-core integrated GPU.

\section{Privacy Analysis}\label{appendix:privacy}
We analyze the privacy of the client-side quantities involved in the
uncertainty propagation. In particular, server-side quantities computed
solely from differentially private client releases inherit the same
privacy guarantees by the post-processing property of differential
privacy.

\textbf{Latent Privacy.}
In the client--server setting of \cite{mohanty2025federated}, each
client \(m\) communicates a compressed latent state
\(\hat h_{m,c}^t\) and an augmented latent state
\(\hat h_{m,a}^t\), both of which depend on its private data
\(y_m^t\). Suppose the corresponding mappings have bounded
\(\ell_2\)-sensitivities \(\Delta_c\) and \(\Delta_a\), respectively.
The Gaussian mechanisms
\[
\tilde h_{m,c}^t
=
\hat h_{m,c}^t+\xi_{c,m}^t,
\qquad
\xi_{c,m}^t\sim\mathcal N(0,\sigma_c^2I),
\]
and
\[
\tilde h_{m,a}^t
=
\hat h_{m,a}^t+\xi_{a,m}^t,
\qquad
\xi_{a,m}^t\sim\mathcal N(0,\sigma_a^2I),
\]
with
\[
\sigma_c
\ge
\frac{\Delta_c\sqrt{2\log(1.25/\delta)}}{\varepsilon},
\qquad
\sigma_a
\ge
\frac{\Delta_a\sqrt{2\log(1.25/\delta)}}{\varepsilon},
\]
provide \((\varepsilon,\delta)\)-DP for the corresponding
client releases at iteration \(t\). The sensitivity bound for
\(\hat h_{m,c}^t\) follows the analysis of
\cite{mohanty2025federated}; an analogous bound is required for
\(\hat h_{m,a}^t\).

\textbf{Privacy of Cross-Covariance Tracking.}
The uncertainty propagation additionally contains the cross-covariances
\[
\Gamma_{mn}^t
:=
\operatorname{Cov}(a_{mn}^t,\hat h_{m,a}^t),
\qquad
\Psi_{mn}^t
:=
\operatorname{Cov}(a_{mn}^t,v_m^t),
\]
where
\(a_{mn}^t=\operatorname{Vec}(\hat A_{mn}^t)\) and
\(v_m^t=\operatorname{Vec}(\theta_m^t)\).

These cross-covariances are analytical quantities used in the
uncertainty propagation and do not themselves require transmission of
the raw client parameter \(v_m^t\). In particular, the appearance of
\(\Psi_{mn}^t\) in the covariance recursion does not imply that
\(v_m^t\) is communicated to the server.

For \(\Gamma_{mn}^t\), the server-side covariance can instead be
formed using the privatized augmented state,
\[
\tilde\Gamma_{mn}^t
:=
\operatorname{Cov}(a_{mn}^t,\tilde h_{m,a}^t).
\]
Conditional on the server history, if the fresh privacy noise
\(\xi_{a,m}^t\) is independent of
\(a_{mn}^t\) and has zero mean, then
\[
\mathbb E[\tilde\Gamma_{mn}^t]
=
\Gamma_{mn}^t.
\]
Moreover, any server-side computation performed solely from
differentially private client releases incurs no additional privacy
loss beyond that of those releases by the post-processing property of
differential privacy.

If an implementation additionally communicates statistics involving
\(v_m^t\) in order to estimate \(\Psi_{mn}^t\) at the server, those
statistics constitute an additional client release and must be
privatized using a separately calibrated mechanism. Such communication
is not required by the uncertainty-tracking formulation considered
here.

\textbf{Implications for Uncertainty Propagation.}
Differentially private perturbations introduce additional stochastic
variation into the quantities used by the covariance recursions.
For example,
\[
\operatorname{Var}(\tilde h_{m,a}^t)
=
\operatorname{Var}(\hat h_{m,a}^t)
+
\sigma_a^2 I
\]
when the privacy noise is independent of the augmented state.
Thus, privacy noise contributes an additional uncertainty source while
preserving the form of the covariance propagation. The resulting
estimation accuracy therefore exhibits a privacy--utility trade-off
controlled by \((\varepsilon,\delta)\).
For repeated releases across iterations, the overall privacy guarantee
must additionally account for composition across training rounds.

\section{Complexity Analysis}\label{appendix:complexity}
Let each client $m$ have state dimension $p_m$ and data dimension $d_m$, with $M$ clients in total. We first provide the computation and communication complexity per client.

\textbf{(1) Computation.} In a naive implementation, updating the full parameter covariance matrix $\Sigma^t_\theta \in \mathbb{R}^{p_m d_m \times p_m d_m}$ would incur $\mathcal{O}(p_m^2 d_m^2)$ computation per round due to matrix-matrix multiplications and Kronecker product evaluations. However, our implementation can avoid this cost by exploiting the structure of low-rank quantities such as $y^t y^{t^\top}$ (which is rank-1), when $d_m \gg p_m$. As a result, matrix-vector products involving these terms can be computed without materializing the full $p_m d_m \times p_m d_m$ matrices. If we adopt a block-diagonal or factored representation of $\Sigma^t_\theta$ across layers or time steps, the per-client computational cost is further reduced to $\mathcal{O}(p_m d_m^2)$ per iteration. 

\textbf{(2) Communication.} In terms of communication, our framework introduces no additional overhead compared to vanilla FedGC. Therefore, total communication cost per client remains $\mathcal{O}(p_m)$, matching that of FedGC. No additional communication is needed for tracking uncertainty.

\textbf{All $M$ clients.} Across all $M$ clients, the total additional computation scales as $\mathcal{O}\left( \sum_{m=1}^M p_m d_m^2 \right)$ per round, while communication remains unchanged at $\mathcal{O}\left( \sum_{m=1}^M p_m \right)$. These properties make the method scalable to realistic federated settings.
\section{Extensions to the Framework}\label{appendix:non_linear}
{
\subsection{Non-Linear Models}\label{sec:NonLinear}
}

Our current theory is derived under the assumption of linear time-invariant dynamics and Gaussian noise, consistent with the client-server framework by \cite{mohanty2025federated}. Nonetheless, we emphasize that the structure of our uncertainty propagation is not inherently tied to linearity; it generalizes to any dynamical model with well-defined latent state compression and state transition. 

We consider an extension to \textit{nonlinear but stationary} dynamical systems. Nonlinear dynamics being highly model-specific, we first show that our theoretical structure carries over to two settings: (1) Extended Kalman Filters (EKF) - non-linear extension of the linear KF based formulation, and (2) Gaussian Processes (GPs) - for kernelized settings. Both EKF, and GP-based frameworks support general nonlinear transition functions while admitting tractable recursive expressions analogous to Theorem 6.8.

\textbf{(1) Extended Kalman Filter (EKF).} EKF approximates nonlinear state transitions via first-order Taylor expansions. Consider the system:
\[
h^t_c = f(h^{t-1}_c) + w^t, \quad y^t = g(h^t_c) + v^t, \quad w^t \sim \mathcal{N}(0, Q),\; v^t \sim \mathcal{N}(0, R),
\]
where \( f: \mathbb{R}^p \to \mathbb{R}^p \), \( g: \mathbb{R}^p \to \mathbb{R}^d \), with \( d \gg p \), are smooth nonlinear functions. The prediction step linearizes \( f \) around the filtered mean:
\[
\hat{h}^{t|t-1}_c = f(\hat{h}^{t-1}_c), \quad F^{t-1} := \left. \frac{\partial f}{\partial h} \right|_{h = \hat{h}^{t-1}_c}, \quad P^{t|t-1} = F^{t-1} P^{t-1} F^{t-1^\top} + Q.
\]
For the observation model, define \( J_g := \left. \frac{\partial g}{\partial h} \right|_{h = \hat{h}^{t|t-1}_c} \). The update step proceeds via:
\[
K^t = P^{t|t-1} J_g^\top (J_g P^{t|t-1} J_g^\top + R)^{-1}, \quad
\hat{h}^t_c = \hat{h}^{t|t-1}_c + K^t(y^t - g(\hat{h}^{t|t-1}_c)), \quad
P^t = (I - K^t J_g) P^{t|t-1}.
\]
We define the augmented representation used in FedGC:
\[
\hat{h}^t_a = \hat{h}^t_c + \theta^t y^t, \quad \theta^t \in \mathbb{R}^{p \times d}, \quad v^t := \mathrm{vec}(\theta^t).
\]
Following our update rule, \( v^{t+1} = v^t - \eta_1 \nabla_{v^t} \mathcal{L}^{\mathrm{local}}_t - \eta_2 \nabla_{v^t} \mathcal{L}^{\mathrm{server}}_t \), the parameter covariance update becomes:
\[
\Sigma^{t+1}_\theta = H^t \Sigma^t_\theta H^{t^\top} + G^t P^t G^{t^\top} + H^t \Lambda^t G^{t^\top} + G^t \Lambda^{t^\top} H^t + P^t \Sigma^t_A P^{t^\top},
\]
where,\\
\begin{align*}
H^t &= I_{pd} - 2\eta_1 (y^t y^{t^\top}) \otimes (J_g^\top J_g) - 2\eta_2 (y^t y^{t^\top}) \otimes (A_{mm}^\top A_{mm}), \\
G^t &= 2\eta_1 (y^t \otimes J_g^\top), \\
P^t &= \mathrm{Cov}(\hat{h}^t_c), \\
\Lambda^t &= \mathrm{Cov}(v^t, \hat{h}^t_a).
\end{align*}
\\ This recursion is structurally identical to Theorem 6.8, with \( C \) and \( P \) replaced by the Jacobian \( J_g \) and EKF posterior \( P^t \), respectively.

\textbf{(2) Gaussian Process (GP).} Consider a GP-based transition model, where \( f \sim \mathcal{GP}(0, k(\cdot,\cdot)) \) governs the latent dynamics:
\[
h^t_c = f(h^{t-1}_c) + w^t.
\]
Conditioning on past data \( \{h^i_c, h^{i-1}_c\}_{i=1}^n \), the GP posterior yields:
\[
h^t_c \sim \mathcal{N}(\mu^f_t, \Sigma^f_t + Q), \quad \mu^f_t = k^\top(h^{t-1}_c) K^{-1} \mathbf{h}, \quad \Sigma^f_t = k(h^{t-1}_c, h^{t-1}_c) - k^\top(h^{t-1}_c) K^{-1} k(h^{t-1}_c),
\]
where \( K \) is the kernel matrix over \( \{h^{i-1}_c\} \). We again define:
\[
\hat{h}^t_a = \mu^f_t + \theta^t y^t.
\]
The parameter covariance propagates as:
\[
\Sigma^{t+1}_\theta = H^t \Sigma^t_\theta H^{t^\top} + G^t \Sigma^f_t G^{t^\top} + H^t \Lambda^t G^{t^\top} + G^t \Lambda^{t^\top} H^t + P^t \Sigma^t_A P^{t^\top},
\]
with \( J_g := \partial g / \partial h \big|_{h = \mu^f_t} \) and other terms as above. This recursion mirrors the EKF case, with GP posterior variance \( \Sigma^f_t \) replacing \( P^t \), and confirms that theoretical results in the main text apply structurally to kernelized approximations.


\subsection{Relaxing Stationary Assumptions}\label{appendix:relaxing_stationary}
{
Our main theoretical results are derived under the classical assumption that each
client observes weakly stationary data with time-invariant first- and second-order
moments. This assumption enables closed-form uncertainty propagation. Here we clarify how the framework behaves when stationarity is
mildly violated and how it can be adapted to non-stationary settings.

\paragraph{Exponentially weighted moments.}
The key observation is that all uncertainty recursions remain valid when the
empirical moments are replaced by \emph{exponentially weighted} (EWMA) moments,
\[
\mu_t = (1-\lambda)\mu_{t-1} + \lambda x_t, \qquad
\Sigma_t = (1-\lambda)\Sigma_{t-1} + \lambda x_t x_t^\top,
\]
for a forgetting factor $0<\lambda<1$.
Suppose the underlying data possess slowly drifting moments
$(\mu_t^*,\Sigma_t^*)$ with bounded temporal variation
$\| \mu_t^* - \mu_{t-1}^* \| \le \delta$ and
$\| \Sigma_t^* - \Sigma_{t-1}^* \| \le \delta$.
Standard results for stochastic approximation imply
\[
\|\mu_t - \mu_t^*\| = O(\delta/\lambda), \qquad
\|\Sigma_t - \Sigma_t^*\| = O(\delta/\lambda).
\]
Thus, EWMA moments track the true time-varying moments whenever the drift is
slower than the forgetting rate. Replacing stationary moments by EWMA moments
preserves the algebraic form of our uncertainty recursions; the same propagation
equations apply with $(\mu_t,\Sigma_t)$ in place of fixed moments.

Beyond EWMA, other forms of non-stationarity can be incorporated by modifying the moment-
estimation step. Examples include sliding-window estimators, seasonally adjusted or periodic-
window estimators, trend-filtered or total-variation–regularized moment updates, and online
convex-combination estimators for abrupt regime changes. Neural-network–based models can
also accommodate complex non-stationarity, but doing so would require deriving a new set
of uncertainty–propagation equations beyond the linear framework. A systematic
treatment of these extensions is an important direction for future work.
}
\section{Proofs}\label{appendix:Proofs}
\subsection{Proposition 5.1}
\begin{proof}
Recall that
\[
v_m^t=\operatorname{Vec}(\theta_m^t),
\]
where $\theta_m^t$ is obtained from the client update
\[
\theta_m^t
=
\theta_m^{t-1}
-
\eta_1
\nabla_{\theta_m}\mathcal L_{m,c}^{t-1}
-
\eta_2
\nabla_{\theta_m}\mathcal L_{m,s}^{t-1}.
\]
For the squared local prediction loss,
\[
\mathcal L_{m,c}^{t-1}
=
\left\|
y_m^t
-
C_{mm}A_{mm}\hat h_{m,a}^{t-1}
\right\|_2^2,
\]
with
\[
\hat h_{m,a}^{t-1}
=
\hat h_{m,c}^{t-1}
+
\theta_m^{t-1}y_m^{t-1},
\]
the corresponding gradient is
\begin{align}
\nabla_{\theta_m}\mathcal L_{m,c}^{t-1}
=
-2(C_{mm}A_{mm})^\top
\Big[
y_m^t
-
C_{mm}A_{mm}
\big(
\hat h_{m,c}^{t-1}
+
\theta_m^{t-1}y_m^{t-1}
\big)
\Big]
y_m^{t-1\top}.
\label{eq:client-gradient-data-dependence}
\end{align}
Similarly, the server loss is
\[
\mathcal L_{m,s}^{t-1}
=
\left\|
A_{mm}
\big(
\hat h_{m,a}^{t-1}
-
\hat h_{m,c}^{t-1}
\big)
-
\sum_{n\neq m}
\hat A_{mn}^{t-1}\hat h_{n,c}^{t-1}
\right\|_2^2,
\]
whose gradient with respect to $\theta_m$ is
\begin{align}
\nabla_{\theta_m}\mathcal L_{m,s}^{t-1}
=
2A_{mm}^\top
\Big[
A_{mm}\theta_m^{t-1}y_m^{t-1}
-
\sum_{n\neq m}
\hat A_{mn}^{t-1}\hat h_{n,c}^{t-1}
\Big]
y_m^{t-1\top}.
\label{eq:server-gradient-data-dependence}
\end{align}

Substituting
Eqs.~\eqref{eq:client-gradient-data-dependence} and
\eqref{eq:server-gradient-data-dependence} into the client update
shows that
\begin{align}
\theta_m^t
=
\theta_m^{t-1}
&+
2\eta_1(C_{mm}A_{mm})^\top
\Big[
y_m^t
-
C_{mm}A_{mm}
\big(
\hat h_{m,c}^{t-1}
+
\theta_m^{t-1}y_m^{t-1}
\big)
\Big]
y_m^{t-1\top}
\nonumber\\
&-
2\eta_2A_{mm}^\top
\Big[
A_{mm}\theta_m^{t-1}y_m^{t-1}
-
\sum_{n\neq m}
\hat A_{mn}^{t-1}\hat h_{n,c}^{t-1}
\Big]
y_m^{t-1\top}.
\label{eq:full-client-update-data}
\end{align}

By assumption, $\{\hat h_{n,c}^{t}\}$ are computed prior to the client--server
optimization and are treated as fixed throughout the uncertainty
analysis. Hence, their presence in
Eq.~\eqref{eq:full-client-update-data} does not introduce additional
randomness. Nevertheless, the update explicitly depends on the
random client observations $y_m^{t-1}$ and $y_m^t$.

Consequently, there exists a measurable update mapping
$\mathcal U_m^{t-1}$ such that
\[
v_m^t
=
\mathcal U_m^{t-1}
\left(
v_m^{t-1},
\{\hat A_{mn}^{t-1}\}_{n\neq m},
y_m^{t-1},
y_m^t
\right),
\]
Therefore, $v_m^t$ is itself a random variable whose distribution
depends on the client data $y_m^t$. Accordingly, their cross-covariance
\[
\Omega_m^t
=
\operatorname{Cov}(v_m^t,y_m^t)
=
\mathbb E
\left[
(v_m^t-\mu_{\theta_m}^t)
(y_m^t-\mu_{y_m}^t)^\top
\right]
\]
must be retained as a potentially nonzero covariance term.

(\textit{We emphasize that the proposition does not assert that
$\Omega_m^t\neq0$ for every possible realization or parameter
configuration. Particular distributions, parameter values, or
symmetries may cause the covariance to vanish even when the random
variables are statistically dependent. Rather, the result establishes
that the federated update does not justify the independence assumption
$\Omega_m^t=0$ in general.})
\end{proof}
\subsection{Proposition 5.2}
\begin{proof}
Recall that
\[
\Lambda_m^t
=
\operatorname{Cov}
(v_m^t,\hat h_{m,a}^t),
\]
where
\[
\hat h_{m,a}^t
=
\hat h_{m,c}^t+\theta_m^t y_m^t.
\]
Since $\hat h_{m,c}^t$ is computed prior to the client--server
optimization and is treated as fixed throughout the uncertainty
analysis,
\[
\operatorname{Cov}
(v_m^t,\hat h_{m,c}^t)=0.
\]
Therefore,
\begin{equation}
\Lambda_m^t
=
\operatorname{Cov}
(v_m^t,\theta_m^t y_m^t).
\label{eq:lambda-start}
\end{equation}

For compactness, define the centered quantities
\[
\delta v_m^t
=
v_m^t-\mu_{\theta_m}^t,
\qquad
\delta y_m^t
=
y_m^t-\mu_{y_m}^t,
\]
and
\[
\delta\theta_m^t
=
\theta_m^t-\bar{\theta}_m^t,
\qquad
\bar{\theta}_m^t
=
\operatorname{unvec}_{p_m\times d_m}
(\mu_{\theta_m}^t).
\]
Thus,
\[
\theta_m^t
=
\bar{\theta}_m^t+\delta\theta_m^t,
\qquad
y_m^t
=
\mu_{y_m}^t+\delta y_m^t.
\]

Expanding their product gives
\begin{align}
\theta_m^t y_m^t
&=
(\bar{\theta}_m^t+\delta\theta_m^t)
(\mu_{y_m}^t+\delta y_m^t)
\nonumber\\
&=
\bar{\theta}_m^t\mu_{y_m}^t
+
\delta\theta_m^t\mu_{y_m}^t
+
\bar{\theta}_m^t\delta y_m^t
+
\delta\theta_m^t\delta y_m^t.
\label{eq:theta-y-expansion}
\end{align}

Taking expectations and using
$\mathbb E[\delta\theta_m^t]=0$ and
$\mathbb E[\delta y_m^t]=0$ gives
\[
\mathbb E[\theta_m^t y_m^t]
=
\bar{\theta}_m^t\mu_{y_m}^t
+
\mathbb E[
\delta\theta_m^t\delta y_m^t
].
\]
Hence,
\begin{align}
\theta_m^t y_m^t
-
\mathbb E[\theta_m^t y_m^t]
&=
\delta\theta_m^t\mu_{y_m}^t
+
\bar{\theta}_m^t\delta y_m^t
\nonumber\\
&\quad+
\delta\theta_m^t\delta y_m^t
-
\mathbb E[
\delta\theta_m^t\delta y_m^t
].
\label{eq:centered-theta-y}
\end{align}

Substituting Eq.~\eqref{eq:centered-theta-y} into
Eq.~\eqref{eq:lambda-start}, we obtain
\begin{align}
\Lambda_m^t
&=
\mathbb E\!\left[
\delta v_m^t
(\delta\theta_m^t\mu_{y_m}^t)^\top
\right]
\nonumber\\
&\quad+
\mathbb E\!\left[
\delta v_m^t
(\bar{\theta}_m^t\delta y_m^t)^\top
\right]
\nonumber\\
&\quad+
\mathbb E\!\left[
\delta v_m^t
\left(
\delta\theta_m^t\delta y_m^t
-
\mathbb E[
\delta\theta_m^t\delta y_m^t
]
\right)^\top
\right].
\label{eq:lambda-three-terms}
\end{align}

We evaluate these three terms separately.

For the first term, using the vectorization identity
\[
\delta\theta_m^t\mu_{y_m}^t
=
\left(
\mu_{y_m}^{t\top}\otimes I_{p_m}
\right)
\delta v_m^t,
\]
we have
\begin{align}
\mathbb E\!\left[
\delta v_m^t
(\delta\theta_m^t\mu_{y_m}^t)^\top
\right]
&=
\mathbb E\!\left[
\delta v_m^t\delta v_m^{t\top}
\right]
(\mu_{y_m}^t\otimes I_{p_m})
\nonumber\\
&=
\Sigma_{\theta_m}^t
(\mu_{y_m}^t\otimes I_{p_m}).
\label{eq:lambda-first}
\end{align}

For the second term,
\begin{align}
\mathbb E\!\left[
\delta v_m^t
(\bar{\theta}_m^t\delta y_m^t)^\top
\right]
&=
\mathbb E\!\left[
\delta v_m^t
\delta y_m^{t\top}
\right]
\bar{\theta}_m^{t\top}
\nonumber\\
&=
\Omega_m^t
\bar{\theta}_m^{t\top}.
\label{eq:lambda-second}
\end{align}

For the third term, since
$\mathbb E[\delta v_m^t]=0$,
\begin{align}
&\mathbb E\!\left[
\delta v_m^t
\left(
\delta\theta_m^t\delta y_m^t
-
\mathbb E[
\delta\theta_m^t\delta y_m^t
]
\right)^\top
\right]
\nonumber\\
&\qquad=
\mathbb E\!\left[
\delta v_m^t
(\delta\theta_m^t\delta y_m^t)^\top
\right]
=:
\mathcal T_m^t.
\label{eq:lambda-third}
\end{align}
This term contains the mixed third-order central moments between
the client parameter and client data.

Combining
Eqs.~\eqref{eq:lambda-first}--\eqref{eq:lambda-third} yields
\[
\Lambda_m^t
=
\Sigma_{\theta_m}^t
(\mu_{y_m}^t\otimes I_{p_m})
+
\Omega_m^t\bar{\theta}_m^{t\top}
+
\mathcal T_m^t.
\]

Finally, suppose that $(v_m^t,y_m^t)$ are jointly Gaussian.
Since $\delta\theta_m^t$ consists of entries of
$\delta v_m^t$, every entry of $\mathcal T_m^t$ is a third-order
central moment of jointly Gaussian random variables. All odd-order
central moments of a zero-mean jointly Gaussian random vector vanish.
Therefore,
\[
\mathcal T_m^t=0,
\]
and the expression reduces to
\[
\Lambda_m^t
=
\Sigma_{\theta_m}^t
(\mu_{y_m}^t\otimes I_{p_m})
+
\Omega_m^t\bar{\theta}_m^{t\top}.
\]
\end{proof}
\subsection{Lemma 5.3}


\begin{proof}
Define
\[
q_m^t:=\hat h_{m,a}^t-\hat h_{m,c}^t.
\]
The server-side loss associated with client \(m\) is
\[
L_{m,s}^t
=
\left\|
A_{mm}q_m^t
-
\sum_{r\neq m}A_{mr}^t\hat h_{r,c}^t
\right\|_2^2.
\]
Its gradient with respect to \(A_{mn}^t\) is
\[
\nabla_{A_{mn}}L_{m,s}^t
=
-2
\left(
A_{mm}q_m^t
-
\sum_{r\neq m}A_{mr}^t\hat h_{r,c}^t
\right)
\hat h_{n,c}^{t\top}.
\]
Hence,
\[
A_{mn}^{t+1}
=
A_{mn}^t
+
2\gamma A_{mm}q_m^t\hat h_{n,c}^{t\top}
-
2\gamma
\sum_{r\neq m}
A_{mr}^t\hat h_{r,c}^t\hat h_{n,c}^{t\top}.
\]
Separating the term \(r=n\),
\[
\begin{aligned}
A_{mn}^{t+1}
={}&
A_{mn}^t
\left(
I_{p_n}
-
2\gamma\hat h_{n,c}^t\hat h_{n,c}^{t\top}
\right)
+
2\gamma A_{mm}q_m^t\hat h_{n,c}^{t\top}
\\
&-
2\gamma
\sum_{r\neq m,n}
A_{mr}^t\hat h_{r,c}^t\hat h_{n,c}^{t\top}.
\end{aligned}
\]
Using
\[
\operatorname{Vec}(AXB)
=
(B^\top\otimes A)\operatorname{Vec}(X),
\]
and \(a_{mn}^t=\operatorname{Vec}(A_{mn}^t)\), we obtain
\[
a_{mn}^{t+1}
=
D_n^ta_{mn}^t
-
2\gamma
\sum_{r\neq m,n}
F_{mnr}^ta_{mr}^t
+
2\gamma B_{mn}^tq_m^t,
\]
where
\[
D_n^t
=
\left(
I_{p_n}
-
2\gamma\hat h_{n,c}^t\hat h_{n,c}^{t\top}
\right)\otimes I_{p_m},
\]
\[
B_{mn}^t
=
\hat h_{n,c}^t\otimes A_{mm},
\]
and
\[
F_{mnr}^t
=
\left(
\hat h_{n,c}^t\hat h_{r,c}^{t\top}
\right)\otimes I_{p_m}.
\]

Now write
\[
\hat h_{m,a}^{t+1}
=
\hat h_{m,a}^{t}
+
\Delta h_m^t,
\qquad
\Delta h_m^t
:=
\hat h_{m,a}^{t+1}-\hat h_{m,a}^{t}.
\]
Therefore,
\[
\begin{aligned}
\Gamma_{mn}^{t+1}
&=
\operatorname{Cov}
\left(
a_{mn}^{t+1},
\hat h_{m,a}^{t+1}
\right)
\\
&=
\operatorname{Cov}
\left(
a_{mn}^{t+1},
\hat h_{m,a}^{t}
\right)
+
\operatorname{Cov}
\left(
a_{mn}^{t+1},
\Delta h_m^t
\right).
\end{aligned}
\]
Define
\[
E_{mn}^t
:=
\operatorname{Cov}
\left(
a_{mn}^{t+1},
\Delta h_m^t
\right).
\]
Substituting the vectorized server update into the first
covariance gives
\[
\begin{aligned}
\operatorname{Cov}
\left(
a_{mn}^{t+1},
\hat h_{m,a}^{t}
\right)
={}&
D_n^t
\operatorname{Cov}
\left(
a_{mn}^t,\hat h_{m,a}^t
\right)
\\
&-
2\gamma
\sum_{r\neq m,n}
F_{mnr}^t
\operatorname{Cov}
\left(
a_{mr}^t,\hat h_{m,a}^t
\right)
\\
&+
2\gamma B_{mn}^t
\operatorname{Cov}
\left(
q_m^t,\hat h_{m,a}^t
\right).
\end{aligned}
\]
By definition,
\[
\operatorname{Cov}
\left(
a_{mn}^t,\hat h_{m,a}^t
\right)
=
\Gamma_{mn}^t,
\qquad
\operatorname{Cov}
\left(
a_{mr}^t,\hat h_{m,a}^t
\right)
=
\Gamma_{mr}^t.
\]
Furthermore,
\[
q_m^t
=
\hat h_{m,a}^t-\hat h_{m,c}^t.
\]
Under Assumption~\textbf{(A1)}, \(\hat h_{m,c}^t\) is fixed, and hence
\[
\operatorname{Cov}
\left(
q_m^t,\hat h_{m,a}^t
\right)
=
\operatorname{Var}
\left(
\hat h_{m,a}^t
\right)
=
\Sigma_{h_m}^t.
\]
Consequently,
\[
\Gamma_{mn}^{t+1}
=
D_n^t\Gamma_{mn}^t
-
2\gamma
\sum_{r\neq m,n}
F_{mnr}^t\Gamma_{mr}^t
+
2\gamma B_{mn}^t\Sigma_{h_m}^t
+
E_{mn}^t.
\]

Finally, by the covariance Cauchy--Schwarz inequality,
\[
\|E_{mn}^t\|_F
\leq
\sqrt{
\operatorname{tr}
\left(
\Sigma_{A_{mn}}^{t+1}
\right)}
\sqrt{
\mathbb E
\left[
\left\|
\Delta h_m^t
-
\mathbb E[\Delta h_m^t]
\right\|_2^2
\right]}.
\]
If
\[
\|\Delta h_m^t\|_2\leq\epsilon
\]
almost surely, then
\[
\mathbb E
\left[
\left\|
\Delta h_m^t
-
\mathbb E[\Delta h_m^t]
\right\|_2^2
\right]
\leq
\mathbb E[\|\Delta h_m^t\|_2^2]
\leq
\epsilon^2.
\]
Therefore,
\[
\|E_{mn}^t\|_F
\leq
\epsilon
\sqrt{
\operatorname{tr}
\left(
\Sigma_{A_{mn}}^{t+1}
\right)}.
\]
\end{proof}

\subsection{Lemma 5.4}
\begin{proof}
For compactness, define
\[
K_m:=C_{mm}A_{mm}.
\]
Since
\[
\hat h_{m,a}^t
=
\hat h_{m,c}^t+\theta_m^ty_m^t,
\]
we have
\[
\hat h_{m,a}^t-\hat h_{m,c}^t
=
\theta_m^ty_m^t.
\]

The client update is
\[
\theta_m^{t+1}
=
\theta_m^t
-
\eta_1\nabla_{\theta_m}L_{m,c}^t
-
\eta_2\nabla_{\theta_m}L_{m,s}^t.
\]
For the local prediction loss,
\[
L_{m,c}^t
=
\left\|
y_m^{t+1}
-
K_m
\left(
\hat h_{m,c}^t+\theta_m^ty_m^t
\right)
\right\|_2^2,
\]
the gradient is
\[
\nabla_{\theta_m}L_{m,c}^t
=
-2K_m^\top
\left[
y_m^{t+1}
-
K_m
\left(
\hat h_{m,c}^t+\theta_m^ty_m^t
\right)
\right]
y_m^{t\top}.
\]
Similarly,
\[
L_{m,s}^t
=
\left\|
A_{mm}\theta_m^ty_m^t
-
\sum_{r\neq m}
A_{mr}^t\hat h_{r,c}^t
\right\|_2^2,
\]
and therefore
\[
\nabla_{\theta_m}L_{m,s}^t
=
2A_{mm}^\top
\left[
A_{mm}\theta_m^ty_m^t
-
\sum_{r\neq m}
A_{mr}^t\hat h_{r,c}^t
\right]
y_m^{t\top}.
\]
Substituting these gradients gives
\[
\begin{aligned}
\theta_m^{t+1}
={}&
\theta_m^t
-
2
\left(
\eta_1K_m^\top K_m
+
\eta_2A_{mm}^\top A_{mm}
\right)
\theta_m^ty_m^ty_m^{t\top}
\\
&+
2\eta_1K_m^\top
\left(
y_m^{t+1}-K_m\hat h_{m,c}^t
\right)y_m^{t\top}
\\
&+
2\eta_2A_{mm}^\top
\sum_{r\neq m}
A_{mr}^t\hat h_{r,c}^t y_m^{t\top}.
\end{aligned}
\]
Using
\[
\operatorname{Vec}(AXB)
=
(B^\top\otimes A)\operatorname{Vec}(X),
\]
the client update becomes
\[
v_m^{t+1}
=
H_m^tv_m^t+b_m^t
+
\sum_{r\neq m}P_{mr}^ta_{mr}^t,
\]
where
\[
H_m^t
=
I_{p_md_m}
-
2(y_m^ty_m^{t\top})
\otimes
\left(
\eta_1K_m^\top K_m
+
\eta_2A_{mm}^\top A_{mm}
\right),
\]
\[
P_{mr}^t
=
2\eta_2
\left(
y_m^t\hat h_{r,c}^{t\top}
\otimes A_{mm}^\top
\right),
\]
and
\[
b_m^t
=
2\eta_1
\left(
y_m^t\otimes K_m^\top
\right)
\left(
y_m^{t+1}-K_m\hat h_{m,c}^t
\right).
\]

Conditional on the client data used in the current update,
\(H_m^t\), \(P_{mr}^t\), and \(b_m^t\) are deterministic.
Hence \(b_m^t\) contributes no conditional covariance.

From the server parameter update we have,
\[
a_{mn}^{t+1}
=
D_n^ta_{mn}^t
-
2\gamma
\sum_{r\neq m,n}
F_{mnr}^ta_{mr}^t
+
2\gamma B_{mn}^tq_m^t,
\]
where
\[
q_m^t
=
\hat h_{m,a}^t-\hat h_{m,c}^t.
\]

We now evaluate
\[
\Psi_{mn}^{t+1}
=
\operatorname{Cov}
\left(
a_{mn}^{t+1},v_m^{t+1}
\right).
\]
Substituting both updates yields
\[
\begin{aligned}
\Psi_{mn}^{t+1}
={}&
D_n^t
\operatorname{Cov}(a_{mn}^t,v_m^t)
H_m^{t\top}
\\
&+
D_n^t
\sum_{s\neq m}
\operatorname{Cov}(a_{mn}^t,a_{ms}^t)
P_{ms}^{t\top}
\\
&-
2\gamma
\sum_{r\neq m,n}
F_{mnr}^t
\operatorname{Cov}(a_{mr}^t,v_m^t)
H_m^{t\top}
\\
&-
2\gamma
\sum_{r\neq m,n}
\sum_{s\neq m}
F_{mnr}^t
\operatorname{Cov}(a_{mr}^t,a_{ms}^t)
P_{ms}^{t\top}
\\
&+
2\gamma B_{mn}^t
\operatorname{Cov}(q_m^t,v_m^t)
H_m^{t\top}
\\
&+
2\gamma B_{mn}^t
\sum_{s\neq m}
\operatorname{Cov}(q_m^t,a_{ms}^t)
P_{ms}^{t\top}.
\end{aligned}
\]

By Assumption~\textbf{(A2)},
\[
\operatorname{Cov}(a_{mr}^t,a_{ms}^t)=0,
\qquad r\neq s.
\]
Therefore,
\[
\sum_{s\neq m}
\operatorname{Cov}(a_{mn}^t,a_{ms}^t)
P_{ms}^{t\top}
=
\Sigma_{A_{mn}}^tP_{mn}^{t\top},
\]
and, for each \(r\neq m,n\),
\[
\sum_{s\neq m}
\operatorname{Cov}(a_{mr}^t,a_{ms}^t)
P_{ms}^{t\top}
=
\Sigma_{A_{mr}}^tP_{mr}^{t\top}.
\]

Furthermore, by definition,
\[
\operatorname{Cov}(a_{mn}^t,v_m^t)
=
\Psi_{mn}^t,
\qquad
\operatorname{Cov}(a_{mr}^t,v_m^t)
=
\Psi_{mr}^t.
\]
Since \(\hat h_{m,c}^t\) is fixed under Assumption~\textbf{(A1)},
\[
\operatorname{Cov}(q_m^t,v_m^t)
=
\operatorname{Cov}(\hat h_{m,a}^t,v_m^t)
=
\Lambda_m^{t\top},
\]
and
\[
\operatorname{Cov}(q_m^t,a_{ms}^t)
=
\operatorname{Cov}(\hat h_{m,a}^t,a_{ms}^t)
=
\Gamma_{ms}^{t\top}.
\]

Substituting these identities gives
\[
\begin{aligned}
\Psi_{mn}^{t+1}
={}&
D_n^t\Psi_{mn}^tH_m^{t\top}
-
2\gamma
\sum_{r\neq m,n}
F_{mnr}^t\Psi_{mr}^tH_m^{t\top}
\\
&+
D_n^t\Sigma_{A_{mn}}^tP_{mn}^{t\top}
-
2\gamma
\sum_{r\neq m,n}
F_{mnr}^t\Sigma_{A_{mr}}^tP_{mr}^{t\top}
\\
&+
2\gamma B_{mn}^t\Lambda_m^{t\top}H_m^{t\top}
+
2\gamma B_{mn}^t
\sum_{r\neq m}
\Gamma_{mr}^{t\top}P_{mr}^{t\top}.
\end{aligned}
\]
\end{proof}

\subsection{Lemma 5.5}
\begin{proof}
Recall that
\[
\hat h_{m,a}^t
=
\hat h_{m,c}^t+\theta_m^t y_m^t.
\]
Since $\hat h_{m,c}^t$ is fixed throughout the uncertainty
analysis,
\[
\Sigma_{h_m}^t
=
\operatorname{Var}(\theta_m^t y_m^t).
\]

Define
\[
\bar{\theta}_m^t
=
\mathbb E[\theta_m^t],
\qquad
\delta\theta_m^t
=
\theta_m^t-\bar{\theta}_m^t,
\]
and
\[
\delta y_m^t
=
y_m^t-\mu_{y_m}^t.
\]
Then
\[
\theta_m^t
=
\bar{\theta}_m^t+\delta\theta_m^t,
\qquad
y_m^t
=
\mu_{y_m}^t+\delta y_m^t.
\]

Expanding their product gives
\[
\theta_m^t y_m^t
=
\bar{\theta}_m^t\mu_{y_m}^t
+
\delta\theta_m^t\mu_{y_m}^t
+
\bar{\theta}_m^t\delta y_m^t
+
\delta\theta_m^t\delta y_m^t.
\]

A first-order covariance propagation neglects the second-order
perturbation term
$\delta\theta_m^t\delta y_m^t$. Hence,
\[
\theta_m^t y_m^t
-
\bar{\theta}_m^t\mu_{y_m}^t
\approx
\delta\theta_m^t\mu_{y_m}^t
+
\bar{\theta}_m^t\delta y_m^t.
\]

Since
\[
v_m^t=\operatorname{Vec}(\theta_m^t),
\]
we have
\[
\delta v_m^t
=
v_m^t-\mu_{\theta_m}^t
=
\operatorname{Vec}(\delta\theta_m^t).
\]
Using the vectorization identity,
\[
\delta\theta_m^t\mu_{y_m}^t
=
\left(
\mu_{y_m}^{t\top}\otimes I_{p_m}
\right)
\delta v_m^t.
\]
Define
\[
J_m^t
:=
\mu_{y_m}^{t\top}\otimes I_{p_m}.
\]
Therefore,
\[
\hat h_{m,a}^t
-
\mathbb E[\hat h_{m,a}^t]
\approx
J_m^t\delta v_m^t
+
\bar{\theta}_m^t\delta y_m^t.
\]

Taking the variance gives
\[
\Sigma_{h_m}^t
\approx
\operatorname{Var}
\left(
J_m^t\delta v_m^t
+
\bar{\theta}_m^t\delta y_m^t
\right).
\]

Expanding the variance,
\[
\begin{aligned}
\Sigma_{h_m}^t
\approx{}&
J_m^t
\operatorname{Var}(\delta v_m^t)
J_m^{t\top}
+
\bar{\theta}_m^t
\operatorname{Var}(\delta y_m^t)
\bar{\theta}_m^{t\top}
\\
&+
J_m^t
\operatorname{Cov}(\delta v_m^t,\delta y_m^t)
\bar{\theta}_m^{t\top}
\\
&+
\bar{\theta}_m^t
\operatorname{Cov}(\delta y_m^t,\delta v_m^t)
J_m^{t\top}.
\end{aligned}
\]

By definition,
\[
\operatorname{Var}(\delta v_m^t)
=
\Sigma_{\theta_m}^t,
\]
\[
\operatorname{Var}(\delta y_m^t)
=
\Sigma_{y_m}^t,
\]
and
\[
\operatorname{Cov}(\delta v_m^t,\delta y_m^t)
=
\Omega_m^t.
\]
Consequently,
\[
\operatorname{Cov}(\delta y_m^t,\delta v_m^t)
=
\Omega_m^{t\top}.
\]

Substitution yields
\[
\begin{aligned}
\Sigma_{h_m}^t
\approx{}&
J_m^t\Sigma_{\theta_m}^tJ_m^{t\top}
+
\bar{\theta}_m^t\Sigma_{y_m}^t\bar{\theta}_m^{t\top}
\\
&+
J_m^t\Omega_m^t\bar{\theta}_m^{t\top}
+
\bar{\theta}_m^t\Omega_m^{t\top}J_m^{t\top}.
\end{aligned}
\]
\textit{The first term represents uncertainty propagated from the client
parameter, the second represents uncertainty propagated from the
client measurements, and the final two terms capture the effect of
the client model--data dependence established in
Proposition~\ref{prop:theta-y-dependence}.}
\end{proof}

\subsection{Lemma 5.6}

\begin{proof}
Recall that the server-consistency loss for client $m$ is
\[
\mathcal L_{m,s}^t
=
\left\|
A_{mm}
\left(
\hat h_{m,a}^t-\hat h_{m,c}^t
\right)
-
\sum_{n\neq m}
A_{mn}^t\hat h_{n,c}^t
\right\|_2^2.
\]

Define the server residual
\[
r_m^t
:=
A_{mm}
\left(
\hat h_{m,a}^t-\hat h_{m,c}^t
\right)
-
\sum_{n\neq m}
A_{mn}^t\hat h_{n,c}^t.
\]

The gradient transmitted from the server to client $m$ with respect
to the augmented state is
\[
g_m^t
=
\nabla_{\hat h_{m,a}^t}\mathcal L_{m,s}^t
=
2A_{mm}^\top r_m^t.
\]

Since $A_{mm}$ is deterministic,
\[
\Sigma_{g_m}^t
:=
\operatorname{Var}(g_m^t)
=
4A_{mm}^\top
\operatorname{Var}(r_m^t)
A_{mm}.
\]

It therefore remains to derive
$\operatorname{Var}(r_m^t)$.

Since $\hat h_{m,c}^t$ is fixed,
\[
\operatorname{Var}
\left(
\hat h_{m,a}^t-\hat h_{m,c}^t
\right)
=
\operatorname{Var}(\hat h_{m,a}^t)
=
\Sigma_{h_m}^t.
\]

Hence,
\[
\operatorname{Var}
\left(
A_{mm}
(\hat h_{m,a}^t-\hat h_{m,c}^t)
\right)
=
A_{mm}\Sigma_{h_m}^tA_{mm}^\top.
\]

Next, define
\[
a_{mn}^t
=
\operatorname{Vec}(A_{mn}^t).
\]
Using
\[
A_{mn}^t\hat h_{n,c}^t
=
\left(
\hat h_{n,c}^{t\top}\otimes I_{p_m}
\right)a_{mn}^t,
\]
define
\[
L_{mn}^t
:=
\hat h_{n,c}^{t\top}\otimes I_{p_m}.
\]
Thus,
\[
A_{mn}^t\hat h_{n,c}^t
=
L_{mn}^t a_{mn}^t.
\]

Because $\hat h_{n,c}^t$ is fixed,
$L_{mn}^t$ is deterministic, and therefore
\[
\operatorname{Var}
\left(
A_{mn}^t\hat h_{n,c}^t
\right)
=
L_{mn}^t
\Sigma_{A_{mn}}^t
L_{mn}^{t\top}.
\]

Under Assumption~\textbf{(A2)}, distinct server blocks are uncorrelated.
Thus, for $n\neq r$,
\[
\operatorname{Cov}
\left(
a_{mn}^t,a_{mr}^t
\right)=0.
\]
Consequently,
\[
\operatorname{Var}
\left(
\sum_{n\neq m}
A_{mn}^t\hat h_{n,c}^t
\right)
=
\sum_{n\neq m}
L_{mn}^t
\Sigma_{A_{mn}}^t
L_{mn}^{t\top}.
\]

We must also account for the dependence between the augmented
client state and each server parameter. Recall that
\[
\Gamma_{mn}^t
=
\operatorname{Cov}
\left(
a_{mn}^t,\hat h_{m,a}^t
\right).
\]

Since $\hat h_{m,c}^t$ is fixed,
\[
\operatorname{Cov}
\left(
\hat h_{m,a}^t-\hat h_{m,c}^t,
a_{mn}^t
\right)
=
\Gamma_{mn}^{t\top}.
\]

Therefore,
\[
\begin{aligned}
&\operatorname{Cov}
\left(
A_{mm}
(\hat h_{m,a}^t-\hat h_{m,c}^t),
L_{mn}^t a_{mn}^t
\right)
\\
&\qquad=
A_{mm}
\Gamma_{mn}^{t\top}
L_{mn}^{t\top},
\end{aligned}
\]
and, by transposition,
\[
\begin{aligned}
&\operatorname{Cov}
\left(
L_{mn}^t a_{mn}^t,
A_{mm}
(\hat h_{m,a}^t-\hat h_{m,c}^t)
\right)
\\
&\qquad=
L_{mn}^t
\Gamma_{mn}^t
A_{mm}^\top.
\end{aligned}
\]

Since the two components of $r_m^t$ enter with opposite signs,
the cross-covariance terms enter negatively. Hence,
\[
\begin{aligned}
\operatorname{Var}(r_m^t)
={}&
A_{mm}\Sigma_{h_m}^tA_{mm}^\top
+
\sum_{n\neq m}
L_{mn}^t\Sigma_{A_{mn}}^tL_{mn}^{t\top}
\\
&-
\sum_{n\neq m}
\left(
A_{mm}\Gamma_{mn}^{t\top}L_{mn}^{t\top}
+
L_{mn}^t\Gamma_{mn}^tA_{mm}^\top
\right).
\end{aligned}
\]

Defining this quantity as $R_m^t$ gives
\[
\Sigma_{g_m}^t
=
4A_{mm}^\top R_m^tA_{mm}.
\]

Therefore,
\[
\Sigma_{g_m}^t
=
4A_{mm}^\top
\left[
A_{mm}\Sigma_{h_m}^tA_{mm}^\top
+
\sum_{n\neq m}
L_{mn}^t\Sigma_{A_{mn}}^tL_{mn}^{t\top}
-
\sum_{n\neq m}
\left(
A_{mm}\Gamma_{mn}^{t\top}L_{mn}^{t\top}
+
L_{mn}^t\Gamma_{mn}^tA_{mm}^\top
\right)
\right]
A_{mm}.
\]
\end{proof}





\subsection{Theorem 5.7}
\begin{proof}
Define
\[
q_m^t
:=
\hat h_{m,a}^t-\hat h_{m,c}^t.
\]
The server-side loss for client $m$ is
\[
\mathcal L_{m,s}^t
=
\left\|
A_{mm}q_m^t
-
\sum_{r\neq m}
A_{mr}^t\hat h_{r,c}^t
\right\|_2^2.
\]

The gradient with respect to $A_{mn}^t$ is
\[
\nabla_{A_{mn}}\mathcal L_{m,s}^t
=
-2
\left(
A_{mm}q_m^t
-
\sum_{r\neq m}
A_{mr}^t\hat h_{r,c}^t
\right)
\hat h_{n,c}^{t\top}.
\]
Therefore, the server update is
\[
A_{mn}^{t+1}
=
A_{mn}^t
+
2\gamma A_{mm}q_m^t\hat h_{n,c}^{t\top}
-
2\gamma
\sum_{r\neq m}
A_{mr}^t
\hat h_{r,c}^t
\hat h_{n,c}^{t\top}.
\]

Separating the term $r=n$ gives
\[
\begin{aligned}
A_{mn}^{t+1}
={}&
A_{mn}^t
\left(
I_{p_n}
-
2\gamma
\hat h_{n,c}^t\hat h_{n,c}^{t\top}
\right)
+
2\gamma A_{mm}q_m^t\hat h_{n,c}^{t\top}
\\
&-
2\gamma
\sum_{r\neq m,n}
A_{mr}^t
\hat h_{r,c}^t
\hat h_{n,c}^{t\top}.
\end{aligned}
\]

Using the identity
\[
\operatorname{Vec}(AXB)
=
(B^\top\otimes A)\operatorname{Vec}(X),
\]
and defining
\[
a_{mn}^t:=\operatorname{Vec}(A_{mn}^t),
\]
\[
D_n^t
:=
\left(
I_{p_n}
-
2\gamma
\hat h_{n,c}^t\hat h_{n,c}^{t\top}
\right)
\otimes I_{p_m},
\]
\[
B_{mn}^t
:=
\hat h_{n,c}^t\otimes A_{mm},
\]
and
\[
F_{mnr}^t
:=
\left(
\hat h_{n,c}^t\hat h_{r,c}^{t\top}
\right)
\otimes I_{p_m},
\]
the vectorized update becomes
\[
a_{mn}^{t+1}
=
D_n^t a_{mn}^t
+
2\gamma B_{mn}^t q_m^t
-
2\gamma
\sum_{r\neq m,n}
F_{mnr}^t a_{mr}^t.
\]

Under Assumption~\textbf{(A1)}, the states $h_{m, c}^t $ are fixed throughout
the uncertainty analysis. Hence, $D_n^t$, $B_{mn}^t$, and
$F_{mnr}^t$ are deterministic.

Taking the variance of the update yields
\[
\begin{aligned}
\Sigma_{A_{mn}}^{t+1}
=
\operatorname{Var}
\left(
D_n^t a_{mn}^t
+
2\gamma B_{mn}^t q_m^t
-
2\gamma
\sum_{r\neq m,n}
F_{mnr}^t a_{mr}^t
\right).
\end{aligned}
\]

The variance of the first term is
\[
\operatorname{Var}(D_n^t a_{mn}^t)
=
D_n^t\Sigma_{A_{mn}}^tD_n^{t\top}.
\]

Since $\hat h_{m,c}^t$ is fixed,
\[
\operatorname{Var}(q_m^t)
=
\operatorname{Var}
\left(
\hat h_{m,a}^t-\hat h_{m,c}^t
\right)
=
\Sigma_{h_m}^t.
\]
Thus,
\[
\operatorname{Var}
\left(
2\gamma B_{mn}^tq_m^t
\right)
=
4\gamma^2
B_{mn}^t
\Sigma_{h_m}^t
B_{mn}^{t\top}.
\]

Under Assumption~\textbf{(A2)}, distinct server-parameter blocks are
uncorrelated, such that for $r\neq s$,
\[
\operatorname{Cov}(a_{mr}^t,a_{ms}^t)=0.
\]
Therefore,
\[
\begin{aligned}
&\operatorname{Var}
\left(
-2\gamma
\sum_{r\neq m,n}
F_{mnr}^t a_{mr}^t
\right)
\\
&\qquad=
4\gamma^2
\sum_{r\neq m,n}
F_{mnr}^t
\Sigma_{A_{mr}}^t
F_{mnr}^{t\top}.
\end{aligned}
\]

Next, consider the cross-covariance between $a_{mn}^t$ and
$q_m^t$. Since
\[
q_m^t
=
\hat h_{m,a}^t-\hat h_{m,c}^t
\]
and $\hat h_{m,c}^t$ is fixed,
\[
\operatorname{Cov}(a_{mn}^t,q_m^t)
=
\operatorname{Cov}(a_{mn}^t,\hat h_{m,a}^t)
=
\Gamma_{mn}^t.
\]
Hence, the corresponding cross terms are
\[
2\gamma
D_n^t\Gamma_{mn}^tB_{mn}^{t\top}
+
2\gamma
B_{mn}^t\Gamma_{mn}^{t\top}D_n^{t\top}.
\]

Moreover, Assumption~\textbf{(A2)} gives, for $r\neq n$,
\[
\operatorname{Cov}(a_{mn}^t,a_{mr}^t)=0,
\]
so the cross-covariance between
$D_n^ta_{mn}^t$ and
$-2\gamma\sum_{r\neq m,n}F_{mnr}^ta_{mr}^t$
vanishes.

Finally, the remaining cross-covariance is between $q_m^t$
and the other server blocks. By definition,
\[
\operatorname{Cov}(a_{mr}^t,q_m^t)
=
\operatorname{Cov}(a_{mr}^t,\hat h_{m,a}^t)
=
\Gamma_{mr}^t,
\]
and consequently
\[
\operatorname{Cov}(q_m^t,a_{mr}^t)
=
\Gamma_{mr}^{t\top}.
\]

Since the server-block terms enter the update with coefficient
$-2\gamma$, their cross-covariance with
$2\gamma B_{mn}^tq_m^t$ contributes
\[
-4\gamma^2
\sum_{r\neq m,n}
B_{mn}^t
\Gamma_{mr}^{t\top}
F_{mnr}^{t\top}
\]
and its transpose
\[
-4\gamma^2
\sum_{r\neq m,n}
F_{mnr}^t
\Gamma_{mr}^t
B_{mn}^{t\top}.
\]

Collecting all variance and cross-covariance terms gives
\[
\begin{aligned}
\Sigma_{A_{mn}}^{t+1}
={}&
D_n^t\Sigma_{A_{mn}}^tD_n^{t\top}
+
4\gamma^2
B_{mn}^t\Sigma_{h_m}^tB_{mn}^{t\top}
\\
&+
4\gamma^2
\sum_{r\neq m,n}
F_{mnr}^t
\Sigma_{A_{mr}}^t
F_{mnr}^{t\top}
\\
&+
2\gamma
D_n^t\Gamma_{mn}^tB_{mn}^{t\top}
+
2\gamma
B_{mn}^t\Gamma_{mn}^{t\top}D_n^{t\top}
\\
&-
4\gamma^2
\sum_{r\neq m,n}
\left(
B_{mn}^t
\Gamma_{mr}^{t\top}
F_{mnr}^{t\top}
+
F_{mnr}^t
\Gamma_{mr}^t
B_{mn}^{t\top}
\right).
\end{aligned}
\]
\end{proof}

\subsection{Theorem 5.8}
\begin{proof}
From the client update derived in Lemma~\ref{lem:psi-recursion},
\[
v_m^{t+1}
=
H_m^tv_m^t
+
b_m^t
+
\sum_{n\neq m}P_{mn}^ta_{mn}^t,
\]
where
\[
H_m^t
=
I_{p_md_m}
-
2(y_m^ty_m^{t\top})
\otimes
\left[
\eta_1(C_{mm}A_{mm})^\top(C_{mm}A_{mm})
+
\eta_2A_{mm}^\top A_{mm}
\right],
\]
\[
P_{mn}^t
=
2\eta_2
\left(
y_m^t\hat h_{n,c}^{t\top}
\otimes A_{mm}^\top
\right),
\]
and
\[
b_m^t
=
2\eta_1
\left(
y_m^t\otimes(C_{mm}A_{mm})^\top
\right)
\left[
y_m^{t+1}
-
C_{mm}A_{mm}\hat h_{m,c}^t
\right].
\]

Conditional on the current client data, $H_m^t$, $P_{mn}^t$,
and $b_m^t$ are deterministic. Therefore, $b_m^t$ contributes
no conditional variance, and
\[
\Sigma_{\theta_m}^{t+1}
=
\operatorname{Var}
\left(
H_m^tv_m^t
+
\sum_{n\neq m}P_{mn}^ta_{mn}^t
\right).
\]

Expanding the variance gives
\[
\begin{aligned}
\Sigma_{\theta_m}^{t+1}
={}&
H_m^t
\operatorname{Var}(v_m^t)
H_m^{t\top}
\\
&+
\sum_{n\neq m}
P_{mn}^t
\operatorname{Var}(a_{mn}^t)
P_{mn}^{t\top}
\\
&+
\sum_{\substack{n,r\neq m\\n\neq r}}
P_{mn}^t
\operatorname{Cov}(a_{mn}^t,a_{mr}^t)
P_{mr}^{t\top}
\\
&+
\sum_{n\neq m}
H_m^t
\operatorname{Cov}(v_m^t,a_{mn}^t)
P_{mn}^{t\top}
\\
&+
\sum_{n\neq m}
P_{mn}^t
\operatorname{Cov}(a_{mn}^t,v_m^t)
H_m^{t\top}.
\end{aligned}
\]

By definition,
\[
\operatorname{Var}(v_m^t)
=
\Sigma_{\theta_m}^t,
\]
\[
\operatorname{Var}(a_{mn}^t)
=
\Sigma_{A_{mn}}^t,
\]
and
\[
\operatorname{Cov}(a_{mn}^t,v_m^t)
=
\Psi_{mn}^t.
\]
Hence,
\[
\operatorname{Cov}(v_m^t,a_{mn}^t)
=
\Psi_{mn}^{t\top}.
\]

Under Assumption~\textbf{(A2)}, distinct server-parameter blocks are
uncorrelated. Thus, for $n\neq r$,
\[
\operatorname{Cov}(a_{mn}^t,a_{mr}^t)=0.
\]
Consequently, all cross-block terms in the double summation vanish.

Substituting these quantities into the variance expansion yields
\[
\begin{aligned}
\Sigma_{\theta_m}^{t+1}
={}&
H_m^t\Sigma_{\theta_m}^tH_m^{t\top}
+
\sum_{n\neq m}
P_{mn}^t
\Sigma_{A_{mn}}^t
P_{mn}^{t\top}
\\
&+
\sum_{n\neq m}
\left(
H_m^t
\Psi_{mn}^{t\top}
P_{mn}^{t\top}
+
P_{mn}^t
\Psi_{mn}^t
H_m^{t\top}
\right).
\end{aligned}
\]
\end{proof}
\subsection{Proposition 6.1}

\begin{proof}
Define
\[
K_m:=C_{mm}A_{mm},
\qquad
S_m^t:=\eta_1K_m^\top K_m+\eta_2A_{mm}^\top A_{mm}.
\]

From Section~\ref{sec:propagation},
\[
D_n^t
=
\left(
I_{p_n}
-
2\gamma
\hat h_{n,c}^t\hat h_{n,c}^{t\top}
\right)
\otimes I_{p_m}.
\]
Since
\[
\hat h_{n,c}^t\rightarrow\hat h_{n,c},
\]
continuity of the outer and Kronecker products gives
\[
D_n^t
\rightarrow
\left(
I_{p_n}
-
2\gamma
\hat h_{n,c}\hat h_{n,c}^{\top}
\right)
\otimes I_{p_m}
=:D_n.
\]

Similarly,
\[
B_{mn}^t
=
\hat h_{n,c}^t\otimes A_{mm},
\]
and hence
\[
B_{mn}^t
\rightarrow
\hat h_{n,c}\otimes A_{mm}
=:B_{mn}.
\]

For the cross-block server gain,
\[
F_{mnr}^t
=
\left(
\hat h_{n,c}^t\hat h_{r,c}^{t\top}
\right)
\otimes I_{p_m}.
\]
Therefore,
\[
F_{mnr}^t
\rightarrow
\left(
\hat h_{n,c}\hat h_{r,c}^{\top}
\right)
\otimes I_{p_m}
=:F_{mnr}.
\]

Next, recall that the client gain is
\[
H_m^t
=
I_{p_md_m}
-
2(y_m^ty_m^{t\top})\otimes S_m^t,
\]
where $S_m^t=S_m$ when the system matrices and learning rates
are fixed. Its expectation is therefore
\[
\bar H_m^t
:=
\mathbb E[H_m^t]
=
I_{p_md_m}
-
2\mathbb E[y_m^ty_m^{t\top}]
\otimes S_m.
\]

Using
\[
\mathbb E[y_m^ty_m^{t\top}]
=
\Sigma_{y_m}^t
+
\mu_{y_m}^t\mu_{y_m}^{t\top},
\]
we obtain
\[
\bar H_m^t
=
I_{p_md_m}
-
2
\left(
\Sigma_{y_m}^t
+
\mu_{y_m}^t\mu_{y_m}^{t\top}
\right)
\otimes S_m.
\]

Since
\[
\mu_{y_m}^t\rightarrow\mu_{y_m},
\qquad
\Sigma_{y_m}^t\rightarrow\Sigma_{y_m},
\]
it follows that
\[
\bar H_m^t
\rightarrow
I_{p_md_m}
-
2
\left(
\Sigma_{y_m}
+
\mu_{y_m}\mu_{y_m}^{\top}
\right)
\otimes S_m
=:H_m.
\]

Finally,
\[
P_{mn}^t
=
2\eta_2
\left(
y_m^t\hat h_{n,c}^{t\top}
\otimes A_{mm}^\top
\right).
\]
Since $\hat h_{n,c}^t$ is deterministic under
Assumption~(A1),
\[
\bar P_{mn}^t
:=
\mathbb E[P_{mn}^t]
=
2\eta_2
\left(
\mu_{y_m}^t\hat h_{n,c}^{t\top}
\otimes A_{mm}^\top
\right).
\]
Using
\[
\mu_{y_m}^t\rightarrow\mu_{y_m},
\qquad
\hat h_{n,c}^t\rightarrow\hat h_{n,c},
\]
gives
\[
\bar P_{mn}^t
\rightarrow
2\eta_2
\left(
\mu_{y_m}\hat h_{n,c}^{\top}
\otimes A_{mm}^\top
\right)
=:P_{mn}.
\]

Thus,
\[
(D_n^t,B_{mn}^t,F_{mnr}^t,\bar H_m^t,\bar P_{mn}^t)
\rightarrow
(D_n,B_{mn},F_{mnr},H_m,P_{mn}),
\]
\end{proof}

\subsection{Proposition 6.2}
\begin{proof}
From Lemma~\ref{lem:gamma-recursion},
\[
\Gamma_{mn}^{t+1}
=
D_n^t\Gamma_{mn}^t
-2\gamma\sum_{r\neq m,n}F_{mnr}^t\Gamma_{mr}^t
+2\gamma B_{mn}^t\Sigma_{h_m}^t
+E_{mn}^t.
\]
Under Assumption~\textbf{(A4)}, the deterministic local states
converge, and hence
\(D_n^t\to D_n\), \(B_{mn}^t\to B_{mn}\), and
\(F_{mnr}^t\to F_{mnr}\).
Together with \(E_{mn}^t\to0\) and stability of the corresponding
coupled linear operator, the recursion converges to its unique fixed
point,
\[
\Gamma_{mn}^{\infty}
=
D_n\Gamma_{mn}^{\infty}
-2\gamma\sum_{r\neq m,n}F_{mnr}\Gamma_{mr}^{\infty}
+2\gamma B_{mn}\Sigma_{h_m}^{\infty}.
\]

For \(\Psi_{mn}^t\), Lemma~\ref{lem:psi-recursion} gives the frozen-gain
cross-covariance recursion. Under the mean-gain steady-state
approximation, the data-dependent gains are replaced by their mean
gains, which satisfy
\(\bar H_m^t\to H_m\) and \(\bar P_{mr}^t\to P_{mr}\) by
Proposition~\ref{prop:gains-limit}.
Stability of the resulting coupled linear operator then implies
convergence to its unique fixed point,
\[
\begin{aligned}
\Psi_{mn}^{\infty}
={}&
D_n\Psi_{mn}^{\infty}H_m^\top
-2\gamma\sum_{r\neq m,n}F_{mnr}\Psi_{mr}^{\infty}H_m^\top\\
&+D_n\Sigma_{A_{mn}}^{\infty}P_{mn}^\top
-2\gamma\sum_{r\neq m,n}
F_{mnr}\Sigma_{A_{mr}}^{\infty}P_{mr}^\top\\
&+2\gamma B_{mn}\Lambda_m^{\infty\top}H_m^\top
+2\gamma B_{mn}\sum_{r\neq m}
\Gamma_{mr}^{\infty\top}P_{mr}^\top .
\end{aligned}
\]
\end{proof}

\subsection{Corollary 6.3}
\begin{proof}
From Lemma~\ref{lem:var-h-m},
\[
\Sigma_{h_m}^t
\approx
J_m^t\Sigma_{\theta_m}^tJ_m^{t\top}
+\bar{\theta}_m^t\Sigma_{y_m}^t\bar{\theta}_m^{t\top}
+J_m^t\Omega_m^t\bar{\theta}_m^{t\top}
+\bar{\theta}_m^t\Omega_m^{t\top}J_m^{t\top},
\]
where
\[
J_m^t=\mu_{y_m}^{t\top}\otimes I_{p_m}.
\]
By Assumption~\textbf{(A4)},
\[
\mu_{y_m}^t\to\mu_{y_m},
\qquad
\Sigma_{y_m}^t\to\Sigma_{y_m},
\]
and therefore
\[
J_m^t\to J_m
=
\mu_{y_m}^\top\otimes I_{p_m}.
\]
Together with the assumed convergence
\[
\Sigma_{\theta_m}^t\to\Sigma_{\theta_m}^{\infty},
\qquad
\bar{\theta}_m^t\to\bar{\theta}_m^{\infty},
\qquad
\Omega_m^t\to\Omega_m^{\infty},
\]
taking the limit of the first-order covariance approximation yields
\[
\Sigma_{h_m}^{\infty}
\approx
J_m\Sigma_{\theta_m}^{\infty}J_m^\top
+\bar{\theta}_m^{\infty}\Sigma_{y_m}\bar{\theta}_m^{\infty\top}
+J_m\Omega_m^{\infty}\bar{\theta}_m^{\infty\top}
+\bar{\theta}_m^{\infty}\Omega_m^{\infty\top}J_m^\top.
\]
\end{proof}
\subsection{Proposition 6.4}
\begin{proof}
The covariance recursions in Section\ref{sec:propagation}
are affine in the covariance and cross-covariance quantities contained
in \(\mathbf{s}_m^t\), after applying the Gaussian closure of
Proposition~\ref{prop:lambda-closed-form} and the first-order
approximation of Lemma~\ref{lem:var-h-m}..
Vectorizing these recursions and stacking the resulting equations
therefore gives
\[
\mathbf{s}_m^{t+1}
=
\mathcal{K}_m^t\mathbf{s}_m^t+\mathbf{c}_m^t,
\]
where \(\mathcal{K}_m^t\) collects all terms linear in
\(\mathbf{s}_m^t\), while \(\mathbf{c}_m^t\) contains the remaining
data-dependent forcing terms.

Under Assumption~\textbf{(A4)}, the deterministic local-state
quantities converge. By Proposition~\ref{prop:gains-limit}, the mean
gains converge, and by assumption
\(\bar{\theta}_m^t\to\bar{\theta}_m^\infty\) and
\(\Omega_m^t\to\Omega_m^\infty\).
Consequently, all coefficients entering the reduced covariance
recursion converge, so that
\[
\mathcal{K}_m^t\to\mathcal{K}_m,
\qquad
\mathbf{c}_m^t\to\mathbf{c}_m.
\]
The limiting covariance dynamics are therefore characterized by
\[
\mathbf{s}_m^{t+1}
=
\mathcal{K}_m\mathbf{s}_m^t+\mathbf{c}_m.
\]
\end{proof}

\subsection{Theorem 6.5}
\begin{proof}
By Proposition~\ref{prop:closed-uncertainty}, the reduced covariance
state satisfies
\[
\mathbf{s}_m^{t+1}
=
\mathcal K_m(t)\mathbf{s}_m^t+\mathbf c_m(t),
\]
where
\(
\mathcal K_m(t)\to\mathcal K_m
\)
and
\(
\mathbf c_m(t)\to\mathbf c_m
\).
Since there exist \(T<\infty\) and \(0<q<1\) such that
\(\|\mathcal K_m(t)\|\le q\) for all \(t\ge T\), it follows that
\(\|\mathcal K_m\|\le q<1\). Hence,
\(I-\mathcal K_m\) is invertible and the limiting affine system has
the unique fixed point
\[
\mathbf{s}_m^\infty
=
(I-\mathcal K_m)^{-1}\mathbf c_m.
\]

Define
\[
\mathbf e_m^t
=
\mathbf{s}_m^t-\mathbf{s}_m^\infty.
\]
Using
\(
\mathbf{s}_m^\infty
=
\mathcal K_m\mathbf{s}_m^\infty+\mathbf c_m
\),
we obtain
\[
\mathbf e_m^{t+1}
=
\mathcal K_m(t)\mathbf e_m^t+\mathbf r_m^t,
\]
where
\[
\mathbf r_m^t
=
\bigl(\mathcal K_m(t)-\mathcal K_m\bigr)
\mathbf{s}_m^\infty
+
\bigl(\mathbf c_m(t)-\mathbf c_m\bigr).
\]
Since
\(\mathcal K_m(t)\to\mathcal K_m\) and
\(\mathbf c_m(t)\to\mathbf c_m\),
we have \(\mathbf r_m^t\to0\). Therefore, for \(t\ge T\),
\[
\|\mathbf e_m^{t+1}\|
\le
q\|\mathbf e_m^t\|+\|\mathbf r_m^t\|.
\]
Iterating gives
\[
\|\mathbf e_m^t\|
\le
q^{t-T}\|\mathbf e_m^T\|
+
\sum_{k=T}^{t-1}
q^{t-1-k}\|\mathbf r_m^k\|.
\]
The first term converges to zero since \(q<1\), and the second
converges to zero because \(\mathbf r_m^k\to0\) and the geometric
weights are summable. Thus,
\(
\mathbf{s}_m^t\to\mathbf{s}_m^\infty
\).

Moreover, the homogeneous contribution from time \(T\) is
\[
\mathcal K_m(t-1)\mathcal K_m(t-2)\cdots
\mathcal K_m(T)\mathbf{s}_m^T,
\]
and
\[
\left\|
\mathcal K_m(t-1)\cdots\mathcal K_m(T)\mathbf{s}_m^T
\right\|
\le
q^{t-T}\|\mathbf{s}_m^T\|
\to0.
\]
Since \(\Sigma_{A_{mn}}^t\) is a component of
\(\mathbf{s}_m^t\), its explicitly propagated initial-covariance
contribution therefore vanishes asymptotically.
\end{proof}
\subsection{Theorem 6.6}
\begin{proof}
By Theorem~\ref{thm:sigmaA-closed},
\[
\mathbf{s}_m^t
\to
\mathbf{s}_m^\infty
=
(I-\mathcal K_m)^{-1}\mathbf c_m,
\]
and this fixed point is unique. Since
\(\operatorname{Vec}(\Sigma_{\theta_m}^t)\) is the client-covariance
component of \(\mathbf{s}_m^t\), it follows that
\[
\Sigma_{\theta_m}^t
\to
\Sigma_{\theta_m}^\infty,
\]
where \(\Sigma_{\theta_m}^\infty\) is the unique corresponding
component of \(\mathbf{s}_m^\infty\).

Furthermore, Theorem~\ref{thm:sigmaA-closed} shows that the
homogeneous contribution of the initial covariance state vanishes
asymptotically. Therefore, for fixed limiting coefficients
\(\mathcal K_m\) and \(\mathbf c_m\),
\(\Sigma_{\theta_m}^\infty\) does not retain an explicit contribution
from \(\Sigma_{\theta_m}^0\).
\end{proof}




\section{Limitations}\label{appendix:limitations}
Our theoretical analysis is derived under a linear time-invariant state-space model with Gaussian noise. In particular, the closed-form covariance recursions and the asymptotic decay of the explicitly propagated initial-covariance contribution rely on this setting and the associated covariance approximations. Although Appendix~\ref{appendix:non_linear} discusses extensions to certain nonlinear models using EKF- and GP-based approximations, these extensions preserve only the structural form of the recursions. We do not establish convergence guarantees, fixed-point uniqueness, or analogous steady-state uncertainty guarantees in the nonlinear regime. A rigorous nonlinear treatment is left to future work.

Assumption~(A2) requires distinct off-diagonal server parameters within the same block row to be uncorrelated. This assumption eliminates cross-block covariance terms and enables the block-wise decomposition used throughout the analysis. In practical systems, however, shared latent confounders or correlated process noise across subsystems may induce correlations between these parameters. In such settings, the corresponding cross-block covariances must be propagated explicitly, potentially increasing the number of covariance terms within each block row from $\mathcal{O}(M)$ to $\mathcal{O}(M^2)$. Characterizing uncertainty propagation under correlated subsystem dynamics remains an open problem.

The asymptotic results in Theorems~\ref{thm:sigmaA-closed} and~\ref{thm:sigmatheta-closed} require convergence of the covariance-system coefficients and stable coupled uncertainty dynamics. Specifically, they establish that the explicitly propagated contribution of the initial covariance vanishes asymptotically. At finite training horizons, this contribution may not have fully decayed. Consequently, the steady-state covariance $\Sigma_{A_{mn}}^\infty$ may not perfectly represent the finite-time covariance $\Sigma_{A_{mn}}^t$ used for post-training inference, potentially affecting the Type-I and Type-II error rates of the hypothesis test in Section~\ref{sec:utility}. We do not derive finite-sample bounds quantifying the gap between $\Sigma_{A_{mn}}^t$ and $\Sigma_{A_{mn}}^\infty$, nor corresponding finite-time corrections to the hypothesis test.

Finally, root-cause localization provides only an indirect evaluation of the inferred causal structure on HAI and TEP, where ground-truth causal graphs are unavailable. AC@1 may be influenced by the uneven distribution of root events across clients, while AC@2 is less discriminative because the datasets contain only three and five clients, respectively. Consequently, the RCA results should be interpreted as evidence of downstream utility rather than direct validation of the recovered causal structure.

\subsection*{AI use statement}


In this work, we used generative AI tools for refining the writing and checking the correctness of theoretical results.
We have not used generative AI tools for experiments.
We have reviewed all AI-assisted work. All LLM-generated statements were checked by the authors; any LLM-assisted theoretical results were reviewed line by line by the authors. 
We take responsibility for the final content of this work, including text, claims, or artifacts produced with the aid of generative AI.

\end{document}